\pdfoutput=1
\documentclass[12pt,english]{article}

\newif\ifblinded
\blindedfalse

\newcommand{\paperdoi}{}

\usepackage[osf]{mathpazo}

\usepackage[T1]{fontenc}
\usepackage[utf8]{inputenc}
\usepackage{geometry}
\usepackage{verbatim}
\usepackage{amsmath}
\usepackage{setspace}

\usepackage{array}
\usepackage{float}
\usepackage{units}
\usepackage{multirow}
\usepackage{amsmath}
\usepackage{amssymb}
\usepackage{graphicx}

\usepackage[authoryear]{natbib}
\makeatletter

\usepackage{placeins}
\usepackage[T1]{fontenc}

\usepackage{booktabs}
\usepackage{ragged2e}
\newcolumntype{R}[1]{>{\RaggedRight\arraybackslash}p{#1}}
\usepackage{array}

\usepackage{caption}
\renewcommand{\thetable}{\arabic{table}}
\usepackage{lmodern}
\usepackage{placeins}\usepackage{relsize}%
\usepackage{xcolor}
\usepackage{sgame}
\usepackage{pdfpages}
\usepackage{float}
\usepackage{eurosym}%
\usepackage{longtable}
\usepackage[font=scriptsize]{caption}
\usepackage{amsthm}
\theoremstyle{plain}

\newtheorem{lemma}{Lemma}

\newtheorem{proposition}{Proposition}

\usepackage{accents}

\usepackage{latexsym,url,graphicx,dsfont, mathtools, amsmath,amssymb,setspace,amsfonts}
\usepackage{secdot}
\usepackage{float}
\usepackage{makeidx}
\usepackage{multirow}
\usepackage{graphicx}
\usepackage{natbib}

\usepackage{caption}
\usepackage{subcaption}

\usepackage{etoolbox}
\apptocmd{\sloppy}{\hbadness 10000\relax}{}{}

\usepackage{natbib}
 \bibpunct[, ]{(}{)}{,}{a}{}{,}%
 \def\BIBand{and}%

\definecolor{hopkins-blue}{RGB}{0,45,114}
\definecolor{columbia-blue}{RGB}{185, 217, 235}
\definecolor{chicago-maroon}{RGB}{128,0,0}
\definecolor{northwestern-purple}{RGB}{82,0,99}
\definecolor{cornell-red}{RGB}{179,27,27}
\definecolor{lawngreen}{RGB}{0,250,154}
\definecolor{gray}{RGB}{192,192,192}

\usepackage{nicefrac}
\usepackage{xr-hyper}
\usepackage[colorlinks,citecolor=chicago-maroon,urlcolor=chicago-maroon,linkcolor=chicago-maroon]{hyperref}
\usepackage[nameinlink]{cleveref}
\hypersetup{pdftitle={AI Model Design and Physician Liability},
  pdfauthor={Shujie Luan, Shubhranshu Singh, Tinglong Dai},
  pdfsubject={Forthcoming in Management Science},
  pdfkeywords={artificial intelligence; health AI; human-AI interaction; liability; product design}}
\crefname{assumption}{Assumption}{Assumptions}
\crefname{lemma}{Lemma}{Lemmas}
\crefname{theorem}{Theorem}{Theorems}
\crefname{corollary}{Corollary}{Corollaries}
\crefname{proposition}{Proposition}{Propositions}
\crefname{claim}{Claim}{Claims}
\crefname{subclaim}{Subclaim}{Subclaims}
\crefname{procedure}{Procedure}{Procedures}
\crefname{algorithm}{Algorithm}{Algorithms}
\crefname{example}{Example}{Examples}
\crefname{figure}{Figure}{Figures}
\crefname{section}{Section}{Sections}
\crefname{appendix}{Appendix}{Appendices}
\crefname{table}{Table}{Tables}
\crefname{assumption}{Assumption}{Assumptions}

\usepackage{xspace}

\allowdisplaybreaks

\normalfont

\usepackage{setspace}

\usepackage{tikz}
\usepackage{caption}
\usepackage{microtype}

\makeatother

\usepackage{babel}

\definecolor{cyan}{gray}{0}
\definecolor{red}{gray}{0}
\definecolor{green}{gray}{0}
\definecolor{blue}{gray}{0}

\begin{document}

\thispagestyle{empty}
\begin{spacing}{1.0}          %

\ifblinded
  \vspace*{1.0in}
\else
  \vspace*{0.35in}
\fi

\begin{center}

{\huge AI Model Design and Physician Liability}

\ifblinded
 \vspace{0.45in}
\else
  \vspace{0.45in}

  {\large
  Shujie Luan \qquad
  Shubhranshu Singh \qquad
  Tinglong Dai}%
  \renewcommand{\thefootnote}{\fnsymbol{footnote}}%
  \footnote{\fontsize{9}{11}\selectfont Shujie Luan is affiliated with  UWA Business School, University of Western Australia, Perth, Australia; Shubhranshu Singh and Tinglong Dai are affiliated with Carey Business School, Johns Hopkins University, Baltimore, Maryland 21202. The coauthors contributed equally to this work. Correspondence may be addressed to \texttt{shujie.luan@uwa.edu.au} (SL); \texttt{shubhranshu.singh@jhu.edu} (SS); \texttt{dai@jhu.edu} (TD). The authors are grateful to Department Editor Raphael Thomadsen, the associate editor, and three anonymous reviewers for their constructive guidance, and to Jane Gu and Haiyang Yang for helpful comments. They also thank participants at the 2024 Conference on Health IT and Analytics, 2024 INFORMS Marketing Science Conference, 2024 MSI Scholars Conference, 2024 MSOM Conference, 2024 POMS Annual Meeting, 2024 POMS International Conference (Istanbul), 2025 POMS India International Conference, 2025 INFORMS Annual Meeting, and 2026 Product and Service Innovation Conference, as well as seminar participants at Bilkent University, Columbia University, George Mason University, Georgia Tech Scheller College of Business, Marketing Theory Group, McGill University Desautels Faculty of Management, ShanghaiTech University, Sun Yat-sen University, University College London Global Business School for Health, University of East Anglia, University of Southern California Gould School of Law, and University of Western Australia.}

  \vspace{0.22in}

  {\small Forthcoming in \textit{Management Science}\par
  Author accepted manuscript, October 2026\par
  \ifx\paperdoi\empty\else
    \edef\doiurl{https://doi.org/\paperdoi}\expandafter\url\expandafter{\doiurl}\par
  \fi}

  \vspace{0.23in}
\fi

\end{center}

\begin{center}
{\scshape Abstract}
\end{center}

\vspace{-0.08in}

\begin{center}
\begin{minipage}{0.93\textwidth}
\small
\noindent
Clinical algorithms can perform differently across patient groups, a concern that grows as these algorithms increasingly take the form of artificial intelligence (AI) models. In the United States, a 2024 federal rule bars healthcare providers that receive federal funds from discriminating through their use of clinical algorithms, which exposes them to liability when such use disadvantages a protected patient group. We examine how such liability considerations reshape (i)~an AI firm's model design decisions that drive group-specific accuracy and (ii)~a physician's decisions to use AI in healthcare delivery. The firm designs an AI model for two patient groups, and improving accuracy for the disadvantaged group is more costly. The physician (who remains the accountable decision maker) then decides whether to consult AI, weighing the reduction in clinical uncertainty against expected liability exposure when AI errors disproportionately affect the disadvantaged group. We find that this liability can induce disparate \textit{use} of AI: the physician may reduce AI use overall and, over an intermediate range of liability, rely on AI less for disadvantaged patients. The effect is non-monotone. As liability increases, the physician's use of AI for disadvantaged patients first declines, then rises as the firm reallocates investment toward reducing disparity or switches to an equal-accuracy design. Mandating equal algorithmic accuracy across patient groups can then inadvertently harm both groups, because a uniform accuracy requirement distorts the firm's investment incentives and the physician's equilibrium AI-use decisions.

\vspace{0.1in}

\end{minipage}
\end{center}

\end{spacing}

\newpage   %

\setcounter{page}{1}          %
\renewcommand{\thefootnote}{\arabic{footnote}}%
\setcounter{footnote}{0}

\normalfont

\smallskip

\section{Introduction}\label{sec:intro}

\label{page:mkt-intro}Most research on artificial intelligence (AI) in expert decision-making treats the two ends of the AI value chain separately. One strand studies how firms design AI models under technical and regulatory constraints \citep[e.g.,][]{fu2022fair,iyer2024competitive,lai2026regulating}; another studies how human decision makers adopt, trust, or strategically act on algorithmic advice \citep[e.g.,][]{dai2023artificial,dietvorst2018overcoming,mclaughlin2022algorithmic}. In practice, design and use are jointly determined: the rules that govern how experts use AI feed back into how firms build it. This interdependence is especially consequential when regulation targets the point of use but the performance disparities that create liability exposure are shaped upstream during design. A feedback loop arises in which (i) a firm designs product quality for different customer segments and (ii) an expert intermediary, who faces policy constraints, decides whether the product is used.

Clinical decision support is a setting in which this feedback loop is first-order. It relies on clinical algorithms, which began as transparent flowcharts \citep{green1978clinical,Margolis1983} and increasingly take the form of AI models that perform well but often operate as black boxes \citep{Gottlieb2024}. The U.S. Food and Drug Administration (FDA) lists 1,614 AI-enabled medical devices authorized through June 2026 \citep{FDA2024}, many of which entered the market with limited or no publicly reported clinical evaluation \citep{lee2025early,lee2025development}. These AI tools can improve screening, diagnosis, and treatment decisions \citep{leong2023autonomous,rajpurkar2022ai,Topol2019}, yet their diffusion into routine care remains uneven \citep{abramoff2024scaling,Wu2023}. A prominent concern is \emph{algorithmic disparity} (systematic differences in predictive performance across protected groups), which raises equity concerns \citep{mehrabi2021survey,Ziad2019Dissecting,obermeyer2021algorithmic} and can expose clinicians and hospitals to legal risk \citep{mello2024antidiscrimination}. In response to these concerns, the U.S. Department of Health and Human Services (HHS) has extended Section~1557 nondiscrimination requirements to patient care decision support tools under \S~92.210, a category that covers AI models as well as flowcharts, equations, and calculators \citep{cms2024nondiscrimination}.\footnote{As of October 2026, \S~92.210 remains in force \citep{ecfr2026section92210}. Federal enforcement priorities have shifted, however: Executive Order 14281 directs agencies to deprioritize enforcement of statutes and regulations to the extent they include disparate-impact liability \citep{EO14281}. Because our analysis treats the level of expected liability as a parameter, its results apply under stronger or weaker enforcement.} The rule bars covered entities (recipients of federal financial assistance, which under the 2024 rule include providers paid through Medicare Part B) from using these tools in ways that discriminate on any of six grounds: race, color, national origin, sex, age, and disability. It also requires them to make ``reasonable efforts'' to determine which of the tools they use take any of these traits as inputs and to reduce the risk that their use leads to discrimination \citep[\S~92.210(b)--(c)]{cms2024nondiscrimination}. In proposing the rule, HHS stated that covered entities, although not liable for algorithms they did not develop, may be held liable for their decisions made in reliance on such algorithms \citep{cms2022nondiscrimination,goodman2023clinical,shachar2023prevention}. The rule thereby places the burden of algorithmic accountability on the providers who use these tools.\footnote{In the remainder of the paper, we use ``physician'' and ``provider'' interchangeably for the clinical decision maker; \cref{sec:model} explains how we assign the obligations of covered entities to the physician.}

Motivated by this policy shift, we study how deployment-facing nondiscrimination liability changes incentives in both the design and use of clinical AI when algorithms perform unevenly across patient groups. To isolate this form of liability, we abstract from broader liability channels, most notably malpractice \citep{price2019potential}, and focus on the margin we model: clinicians' reliance on decision-support tools in patient care. Our framework links upstream AI model design to downstream clinical delivery in a two-stage structure. First, an AI firm chooses its model's accuracy for two patient groups: an advantaged group for whom performance gains are relatively inexpensive, and a disadvantaged group for whom gains are costlier because of data sparsity and measurement frictions \citep{chen2021why,mehrabi2021survey}. Second, the physician decides whether to use the tool in a given case and, conditional on use, whether to follow its recommendation. In our model, liability depends on the firm's design: when the algorithm is disparate, following an incorrect recommendation for a disadvantaged patient creates a specific liability exposure; when the algorithm is equal-accuracy, this disparity-contingent liability channel is removed. The physician weighs the reduction in clinical uncertainty that AI offers against the cost of each use and the legal exposure that reliance can create \citep{patil2026protecting}. This setup allows us to study whether liability imposed at the point of use reaches upstream and induces disparity-reducing investment, or instead weakens adoption and shifts investment in ways that can leave disadvantaged patients worse off.

Our analysis yields three main results. First, liability tied to adverse outcomes from a disparate AI tool can induce disparate \emph{use} of AI. Even holding the algorithm's accuracy fixed, the physician consults AI for a narrower set of disadvantaged patients because expected legal exposure acts like an additional shadow cost of reliance for that group. A policy designed to shield disadvantaged patients from unequal algorithmic performance can thus end up limiting their access to AI.

Second, the effect of liability on AI use for disadvantaged patients is non-monotone. For small increases in liability, the direct deterrence effect dominates, and use declines. As liability rises further, however, the firm responds endogenously by reallocating investment toward disadvantaged-group accuracy and, beyond a threshold, by switching to an equal-accuracy design. These upstream responses reduce expected physician exposure and can restore use, or even increase it relative to intermediate-liability levels. This feedback loop means the same policy that suppresses downstream use at low levels can strengthen upstream incentives to reduce disparity at higher levels.

Third, we show mandating equal measured accuracy across groups does not necessarily improve welfare and can, under plausible conditions, reduce welfare for \textit{both} groups. A one-size-fits-all accuracy requirement distorts the firm's resource allocation under asymmetric improvement costs: in many cases, it lowers advantaged-group accuracy substantially while raising disadvantaged-group accuracy only modestly. When combined with reimbursement structures that reward AI use, this shift can also alter clinical deployment in ways that increase inappropriate use. Equalizing algorithmic performance and ensuring appropriate clinical reliance are distinct policy objectives. Liability standards aimed at the former may require complementary instruments, such as reimbursement design, utilization guidelines, or monitoring rules, to achieve the latter.

None of this argues against regulating algorithmic disparity, which can materially affect care and outcomes \citep{abramoff2022foundational}. For example, \citet{Ziad2019Dissecting} show that racial disparity in a widely used clinical risk algorithm led to fewer Black patients being identified for additional care despite greater illness severity. Although addressing disparity early in development is often recommended over retrofitting ex post \citep{Gichoya2023}, our analysis shows that regulatory design still matters, because the same legal objective can generate very different equilibrium outcomes depending on how incentives are structured. Our aim is to clarify the incentive trade-offs that must be addressed if liability standards are to work as intended.

More broadly, this paper treats medical AI as an input to expert decision-making. Even when AI is highly accurate, its deployment hinges on transparency, accountability, and the incentives of the human expert in the loop \citep{sagona2025trust}. Regulations such as \S~92.210 constrain reliance on disparate tools, implying that the welfare consequences of AI model design are filtered through downstream adoption decisions. Accordingly, our framework allows physicians to decide whether to use AI and, when they do, how much to rely on it; we study how these deployment decisions feed back into firms' design incentives, and how liability and reimbursement shape that feedback.

\section{Literature}\label{sec:literature}

Our work contributes to the expert-service literature originating from \citet{DarbyKarni1973}, who argue credence-goods providers such as physicians may overprovide expert services. A key theme in this literature is that liability shapes clinical behavior. Empirical evidence shows that malpractice liability influences physician decision-making \citep{currie2008first}, and experiments show that liability affects expert behavior in credence-goods markets \citep{Dulleck2011}. Theoretical work shows that liability can induce appropriate treatment choices \citep{Chen2022LiZhang} and that it disciplines short-lived providers yet can weaken the reputational incentives of long-lived ones \citep{FONG2018}.\footnote{The related literature on liability in other domains, such as product safety \citep{guan2024product, iyer2018voluntary}, offers valuable insights but does not directly apply to the expert-service context.} The introduction of AI complicates this picture. \citet{dai2023artificial}, for instance, explore physician behavior under emerging liability frameworks but focus on malpractice liability and on liability from disregarding an AI recommendation, without addressing the potential disparity of AI algorithms across patient groups. We model a healthcare regulation that governs AI and other clinical decision-support tools and analyze its impact on AI firms, physicians, and patients.

Our work also relates to the literature on the role of AI in medical decision-making. Using AI to support clinical decisions connects to the literature on information acquisition. Prior work has examined how people interact with algorithms, including algorithm preference and algorithm aversion \citep[e.g.,][]{dietvorst2018overcoming, iyer2024competitive,leung2018man,Mohammadi2024%
}. The literature has also attempted to identify the causes of overuse of and over-adherence to AI algorithms. \cite{mclaughlin2022algorithmic} attribute over-adherence to recommendation-dependent preferences; for instance, a decision maker may view the algorithmic recommendation as a default action. Turning to underuse, \cite{dai2020conspicuous} develop a signaling model to show highly skilled physicians may underutilize diagnostic tests to signal their skills. \cite{balakrishnan2022improving} show humans over-adhere to an algorithm's predictions when their own private information, which the algorithm cannot access, is valuable, and under-adhere otherwise. We differ from this stream by examining the influence of potential liability. By connecting both the upstream and downstream of medical AI, we show a physician may underuse AI at intermediate liability, because of a dominant liability concern, and overuse AI at large liability, because liability strengthens the AI firm's incentive to invest in algorithmic accuracy. %

Our work also contributes to the literature on disparate algorithm performance. Extensive research documents systematic differences in health and healthcare by race, gender, age, and other characteristics \citep[e.g.,][]{Heckler1985Report,nelson2002unequal}. As predictive algorithms diffuse across high-stakes domains, a central concern is that these tools can reproduce or even widen such gaps when training data and prediction targets are misaligned with clinical needs \citep[e.g.,][]{benjamin2016innovating,gianfrancesco2018potential}. For example, \citet{tipton2023impact} review evidence that standard ``fairness fixes'' (such as omitting race) can improve parity along one dimension while worsening broader health outcomes. Related work develops technical approaches to reducing disparate performance (by changing objectives \citep{samorani2022overbooked}, redefining targets \citep{Ziad2019Dissecting}, or incorporating group identity \citep{gillis2021fairness}) and, in parallel, economic models of algorithm design under data investment, disclosure, and fairness concerns \citep{fu2022fair,li2023beating}. We add a complementary mechanism: deployment-facing regulation can narrow measured performance gaps yet induce unequal use in equilibrium.

\label{page:mkt-lit}This paper also contributes to the economics and marketing literature on product and pricing decisions under ethical or policy constraints \citep[e.g.,][]{diao2023consumer,israeli2018online,iyer2024competitive,Ke2023Privacy,lai2026regulating}. As in studies of how marketing instruments (for example, incentives, pricing, and trust signals) influence the adoption of new technologies \citep[e.g.,][]{lambrecht2024apparent,luo2019frontiers,zimmermann2024adoption}, we study how fairness constraints reshape physician--AI interaction and adoption. In that work, a firm sets product quality for a buyer who decides whether to purchase. In our paper, an expert intermediary, not the buyer, governs use, and we endogenize the quality of the expert's information technology, which the credence-service literature takes as given. Equalizing measured quality then need not equalize realized access.   
Relatedly, our paper engages with the literature on the implications of algorithmic fairness. \cite{corbett2017algorithmic} reveal a tension between improving public safety and satisfying prevailing notions of algorithmic fairness when deciding whether to release pretrial defendants back into the community. \cite{shimao2022strategic} show fair algorithms can lead to different equilibrium behaviors among different groups of prediction subjects. %
\cite{corbett2018measure} find equal opportunity and demographic parity may harm the protected group because of heterogeneity across groups. \cite{liu2018delayed} show an overly aggressive fairness criterion may cause harm to the protected group in the long term, because giving too many loans to people in a protected group who cannot pay them back can hurt the group's credit scores on average. \cite{fu2022fair} show fair algorithms that require impact parity can make everyone worse off, including the protected group, because of the firm's strategic behavior of underinvesting in learning. Our mechanism is distinct: rather than relying on heterogeneity in patient characteristics across groups, we show mandating equal accuracy can harm disadvantaged patients even when the two groups differ only in the cost of improving algorithmic performance. The welfare loss arises because the physician's dual objective, representing patients while also pursuing AI reimbursement, distorts usage decisions under a one-size-fits-all accuracy constraint.

\section{Model}\label{sec:model}

Consider a physician who selects a treatment for a patient from two possible options, $T_1$ and $T_2$. One option is \textit{appropriate} for the patient, and the other is \textit{inappropriate}. The benefit of the appropriate treatment for any patient is $b$, where $b>0$, whereas the benefit of the inappropriate treatment is $0$.

To capture the health disparities often observed across demographic groups \citep{Ziad2019Dissecting}, we consider two patient types, \(t \in \{x, y\}\). Type-\(x\) patients represent an advantaged group, such as White patients, whereas type-\(y\) patients represent a disadvantaged group, such as Black patients. Following the protected characteristics outlined in Section~1557 of the Affordable Care Act \citep{cms2024nondiscrimination}, we assume a patient's type, \(t\), is observable. The type can be defined by any protected characteristic, such as race, sex, age, or disability. For simplicity, we normalize the mass of type-\(x\) patients to one and that of type-\(y\) patients to $\beta$, where $0<\beta \le 1$, to capture that disadvantaged patients are often a minority group.

The physician holds a prior belief that treatment $T_1$ is appropriate
for a patient with probability
$\alpha$, whereas treatment $T_2$ is appropriate with probability 
\label{page:alpha-def}$1-\alpha$, where \(\alpha \sim U[0,1]\).\footnote{The assumption of type-independent priors helps isolate the incremental role of algorithmic disparity and the specific anti-discrimination liability channel we study in this paper. Holding the physician's priors the same for both patient types ensures any cross-group differences in outcomes in our analysis arise from the AI's differential accuracy and the physician's endogenous AI-use decision under liability.} We interpret $\alpha$ as the physician's case-specific assessment after her usual initial examination of the patient and before she consults the AI; it is not a population base rate, and its distribution captures across-case heterogeneity in baseline diagnostic uncertainty rather than any group-specific bias. We study the case of type-dependent priors in \cref{sec:type-specific-priors}. We use $a_t$ to denote the appropriate treatment for a type-$t$ patient. The physician has access to an AI-powered clinical algorithm (hereafter ``AI''). The AI generates a treatment signal, which the physician uses to update her belief about the appropriate treatment. We denote the AI signal for a type-$t$ patient by $s_t$ and define AI accuracy for type $t$ as $\rho_t$, where $t=x,y$. Specifically, $P_{1|1}:=\mathcal P(s_t=T_1|a_t=T_1)=\rho_t$ and $P_{2|1}:=\mathcal P(s_t=T_2|a_t=T_1)=1-\rho_t$. Similarly, $P_{2|2}:=\mathcal P(s_t=T_2|a_t=T_2)=\rho_t$ and $P_{1|2}:=\mathcal P(s_t=T_1|a_t=T_2)=1-\rho_t$. AI recommendations are informative but imperfect; that is, $1/2<\rho_t<1$. 
The physician knows the algorithm's accuracy for each patient type, $\rho_t$, which she can learn from validation results or from experience with the tool; this benchmark abstracts from uncertainty about model quality. We call an algorithm \emph{equal-accuracy} if $\rho_x=\rho_y$, and in that case we write the common accuracy as $\rho$ (dropping the subscript $t$). An algorithm with $\rho_x>\rho_y$ delivers lower accuracy for type-$y$ patients and is therefore disparate for type-$y$ patients.

With the AI signal, the physician updates her belief about the appropriate treatment using Bayes' rule. The physician's posterior belief that $T_1$ is appropriate is
\begin{align*}
    Q_{1|1}&:=\mathcal{P}(a_t=T_1|s_t=T_1)=\frac{\alpha \rho_t}{\alpha \rho_t+(1-\alpha) (1-\rho_t)},\\
    Q_{1|2}&:=\mathcal{P}(a_t=T_1|s_t=T_2)=\frac{\alpha (1-\rho_t)}{\alpha (1-\rho_t)+(1-\alpha) \rho_t}.
\end{align*} 
The physician believes $T_2$ to be appropriate with the complementary probability, that is, $Q_{2|1}:=\mathcal{P}(a_t=T_2|s_t=T_1)=1-Q_{1|1}$ and $Q_{2|2}:=\mathcal{P}(a_t=T_2|s_t=T_2)=1-Q_{1|2}$.

\label{page:legal-mapping}We model the deployment-facing liability associated with the Section~1557 provision (\S~92.210) as follows. If the physician follows the signal of a disparate clinical algorithm for a type-$y$ patient and the resulting treatment is inappropriate, she faces a liability cost $\ell$. This mapping is our modeling choice. The rule's trigger is discrimination against a protected group, which we represent as errors from tool use that fall disproportionately on type-$y$ patients. The accuracy gap itself is a modeling primitive, and a difference in accuracy across groups need not be unlawful \citep{goodman2023clinical}. The rule's duties to identify and mitigate apply to tools that use protected traits as inputs \citep{cms2024nondiscrimination}. How far its general prohibition reaches accuracy gaps in tools without such inputs is unsettled, although in proposing the rule HHS listed prognostic scoring systems whose performance differs by race or ethnicity, even when race is not an input, among the uses of concern \citep{goodman2023clinical}. This channel is specific to type-$y$ patients in our setting; we abstract from other sources of legal exposure (for example, malpractice) to isolate the incentive effects of the provision. Reflecting that the provision governs covered entities' use of decision-support tools and reaches developers only when they are themselves covered entities \citep{cms2024nondiscrimination,mello2024antidiscrimination}, we assign liability to the physician rather than to the AI firm. Placing it on the individual physician is a further assumption. Physicians paid through Medicare Part B are themselves covered entities under the 2024 rule \citep{cms2024nondiscrimination}, and we assume that an organization passes its exposure on to its clinicians through employment, credentialing, and practice policies. When the algorithm has equal accuracy across patient types ($\rho_x=\rho_y$), this disparity-triggered liability channel does not apply. This is a feature of our model: the rule does not single out equal accuracy, and equal accuracy removes the channel here because it removes the accuracy gap that generates the exposure.

When the physician deploys the AI tool in a patient's case, the patient incurs a cost $c>0$. We model $c$ as a reduced-form burden that can include out-of-pocket expenses and non-monetary disutility, such as psychological discomfort, perceived risk, or distrust when an algorithm is involved in diagnosis or treatment \citep{longoni2019resistance}. The cost $c$ does not give patients control over AI use: although opt-out is feasible in some settings, clinical practice and accountability assign the deployment decision to the physician, and we therefore assume the physician retains full discretion over whether to use the tool.

We also allow the physician to receive a per-use benefit $r>0$ when she deploys the tool. This term captures reimbursement for AI-assisted services and broader private returns from using AI \citep{parikh2022paying}. We treat $r$ as a reduced-form benefit. Some AI-assisted services are billed under AI-specific Current Procedural Terminology (CPT) codes, and hospitals can receive New Technology Add-On Payments (NTAP) for eligible inpatient technologies, although neither guarantees a per-use payment to the physician. Depending on the setting, $r$ may also include reputational effects of AI use, which may be positive or negative. A growing literature treats AI competence as a required clinical skill \citep{liaw2022competencies,schubert2025ai,schuitmaker2025physicians}, whereas, in a randomized experiment, practicing clinicians rated a physician who relied on generative AI as less clinically skilled \citep{yang2025peer}.\footnote{Several U.S.\ providers now charge patients directly for AI-assisted reads; for example, RadNet's imaging centers offer an optional AI review of a screening mammogram for an out-of-pocket fee, which a substantial share of patients elect \citep{Andrews2024Mammography}. The fee is an instance of the patient-side cost $c>0$ and, where the provider bills for the read, of the per-use revenue $r>0$. On the patient side, $c>0$ is also consistent with documented resistance to medical AI \citep{longoni2019resistance}.}\textsuperscript{,}\footnote{\label{footnote:r-negative}In the special case in which $r<0$, possibly due to a strong negative reputational effect of using AI, our main results on the physician's reduced AI use for disadvantaged patients and the non-monotone effect of liability on the physician's AI use for disadvantaged patients continue to hold. However, in this case, disadvantaged patients are better off, whereas advantaged patients are worse off as a result of the equal-accuracy mandate. Because the physician has no incentive to overuse AI in the $r<0$ case, the disadvantaged patients cannot be worse off.}

When the physician decides whether and how to use AI, she cares about both the patient's health outcome (that is, treatment benefit net of AI-use cost) and her own nonclinical objectives (that is, revenue and liability). We assume the physician assigns a weight $\theta\ge 0$ to her nonclinical objectives. A physician with $\theta=0$ is altruistic and cares only about the patient's health outcome, whereas a physician with $\theta>0$ is \textit{impurely} altruistic and also cares about her nonclinical objectives. Unless otherwise noted, we assume $\theta>0$; the altruistic case $\theta=0$ is a benchmark for comparison.

Next, we introduce a profit-maximizing firm that designs an AI model with accuracy $\rho_t$ for patient type $t\in\{x,y\}$; once deployed, the model is the clinical algorithm the physician consults. Consistent with fee-for-service arrangements in which AI use is reimbursed at predetermined rates, we assume the firm receives a payment $f$ each time the physician deploys the tool in a patient's case \citep{Abrmoff2022,abramoff2024scaling}.\footnote{In an extension presented in \cref{sec:pricing_extn}, we assume the AI firm sets a profit-maximizing price $f$ and find that all our main insights continue to hold qualitatively under the endogenous pricing assumption.} Achieving accuracy $\rho_t$ for type $t$ entails a development cost $\kappa_t(\rho_t-\tfrac12)^2$, where $\kappa_t>0$ is type-specific. To capture the limited availability and higher acquisition cost of training data for disadvantaged patients, we assume $\kappa_x<\kappa_y$. We focus on settings in which the firm trains a new model on a fixed historical dataset, rather than incrementally updating an existing system, and thus treat the cost coefficients as exogenous and independent of downstream usage. \label{page:cost-asymmetry}The cost asymmetry $\kappa_x<\kappa_y$ describes the empirically common case in which improving performance for the disadvantaged group is harder.\footnote{If the asymmetry were reversed ($\kappa_y<\kappa_x$), the algorithmic accuracy for the disadvantaged group would not be harder to improve, and the disparity-contingent liability channel would be weak or inactive; in an application in which the type-$y$ group is the lower-cost group, the labels $x$ and $y$ can be exchanged.} The separable quadratic structure is a parsimonious way to capture the practical fact that improving performance for a harder-to-predict subgroup typically requires targeted effort (for example, subgroup-specific tuning or additional data), even when the deployed model is unified.

Given physician demand $(d_x^{\mathrm D},d_y^{\mathrm D})$ under a disparate design, the firm's profit is
\begin{align}\label{equ:developer}
\pi^{\mathrm D}(\rho_x,\rho_y):=(d_x^{\mathrm D}+d_y^{\mathrm D})f-\kappa_x(\rho_x-\tfrac12)^2-\kappa_y(\rho_y-\tfrac12)^2,
\end{align}
where $d_x^{\mathrm D}$ and $d_y^{\mathrm D}$ denote the volumes of type-$x$ and type-$y$ cases in which the physician deploys the tool (expressions are given in \cref{equ:AI demand}). When the firm supplies an equal-accuracy design, $\rho_x=\rho_y\equiv\rho$, profit can be written as
\begin{align}\label{equ:optimal_rho_fair}
\pi^{\mathrm E}(\rho):=(d_x^{\mathrm E}+d_y^{\mathrm E})f-\kappa_x(\rho-\tfrac12)^2-\kappa_y(\rho-\tfrac12)^2,
\end{align}
where $(d_x^{\mathrm E},d_y^{\mathrm E})$ are the corresponding demands under equal accuracy, given in \cref{equ:AI demand_undisparity}.

We assume $c>\theta r$, which eliminates cases in which the AI firm develops an arbitrarily bad algorithm but the physician still uses AI for all patients. Equivalently, we exclude cases in which the physician uses AI solely for private gain from reimbursement.  In addition, although the parameters $c$, $r$, and $f$ are likely interrelated in practice, they do not amount to a direct transfer of funds. Instead, funds typically flow from patients and insurers to providers (physician, hospital, or health system), who then pay AI firms. For clarity, we summarize the meaning, examples, and payment flow for the parameters $c$, $r$, and $f$ in \cref{tab:parameters}.
To maintain tractability and focus on core strategic trade-offs, our baseline setting treats the parameters $c$, $r$, and $f$ as independent. Nonetheless, in \cref{sec:pricing_extn}, we explore a variant in which the patient's AI-use cost is the same as the firm's price ($c = f$), and we confirm our main results continue to hold under this specification. 

\vspace{0.1in}
\begin{table}[th]
\centering
\caption{Meaning, Examples, and Payment Flow for Parameters $c$, $r$, and $f$}
\label{tab:parameters}
\small
\setstretch{1.1}
\begin{tabular*}{\textwidth}{@{\extracolsep{\fill}} c R{2.6cm} R{6.5cm} l @{}}
\toprule
 & \textbf{Meaning} & \textbf{Example} & \textbf{Payment Flow} \\
\midrule
$c$ & Patient's cost of AI use &
  Out-of-pocket expenses (copayments, deductibles, coinsurance)
  or disutility (distrust, privacy concerns) &
  Borne by patient \\[10pt]
$r$ & Provider's per-use revenue &
  Insurer reimbursement (for example, CMS, private insurers)
  and reputational benefit, if any, from AI use &
  Accrued by provider \\[10pt]
$f$ & AI firm's per-use fee &
  Payment from provider, hospital, or health system
  to AI firm for technology use or licensing &
  Provider $\rightarrow$ AI firm \\
\bottomrule
\vspace{0.03in}
\end{tabular*}
{\footnotesize\textit{Notes.} %
The parameter $c$ captures both monetary costs (for example, copayments) and non-monetary disutility borne by the patient. The parameter $r$ captures both financial reimbursement (for example, payments for services billed under CPT codes) and
non-monetary gains such as perceived competence or technological savviness. Only $f$ represents a direct inter-party transfer.}
\vspace{0.1in}
\end{table}

\vspace{-0.1in}
Next, we describe the patient's expected utility. First, consider the case in which the physician decides not to use AI. In this case, the physician follows her own prior belief to treat the patient. Given the physician's prior belief $\alpha$ of treatment $T_1$ being appropriate, the patient's expected utility from treatment $T_1$ is $\alpha b$, and from treatment $T_2$ is $(1-\alpha) b$. Now consider the case in which the physician uses AI. In this case, the patient incurs an AI-use cost $c$. The patient's expected benefit from treatment depends on the AI signal and the physician's treatment choice. Specifically, if the physician prescribes treatment $T_j$ after observing signal $T_i$, the patient's expected benefit is $Q_{j|i} b$, where $i,j \in \{1,2\}$. We assume the patient's expected utility is determined by the realized health outcome and any incurred cost. We do not include any gain from lawsuit payouts, to avoid unrealistic scenarios in which the physician might intentionally harm the patient to generate a financial transfer to the patient. \label{page:health-surplus}Equivalently, expected patient welfare for each type, denoted $W_x$ and $W_y$ and defined formally in \cref{sec:welfare}, measures patient health surplus (expected treatment benefit net of the AI-use cost) and excludes litigation compensation. If litigation compensation enters only as an ex post accounting transfer that no decision maker internalizes, equilibrium behavior is unaffected, because patients choose neither AI use nor litigation effort and the provider's exposure is already summarized by $\ell$; the transfer then raises compensation-inclusive type-$y$ surplus and cancels in total surplus apart from deadweight legal costs. If instead the physician internalizes the patient's compensation through patient utility, the behavioral model changes, and we analyze that formulation in \cref{sec:welfare}, where the weight on the liability term becomes $\theta-1$. Our welfare comparisons should therefore be read as comparisons of ex ante health surplus.
  The patient's expected utility is summarized in \cref{table:payoff_AI}.

\begin{table}[!th]
\vspace{0.15in}
\centering
\caption{The Patient's Expected Utility When the Physician Uses AI}
\label{table:payoff_AI}
\medskip
\begin{tabular}{lcc}
\toprule
 & \multicolumn{2}{c}{Physician's decision} \\
\cmidrule(lr){2-3}
AI signal & $T_1$ & $T_2$ \\
\midrule
$s_t = T_1$ & $Q_{1|1}\cdot b - c$ & $Q_{2|1}\cdot b - c$ \\[6pt]
$s_t = T_2$ & $Q_{1|2}\cdot b - c$ & $Q_{2|2}\cdot b - c$ \\
\bottomrule
\end{tabular}
\vspace{0.15in}
\end{table}

We now describe the physician's expected payoff from treating an individual patient. If the physician uses AI, she receives a per-use benefit $r$ and may face liability. For type-$y$ patients, liability arises only when the physician follows the AI recommendation and that recommendation is incorrect. Accordingly, following signal $T_1$ leads to liability with probability $Q_{2|1}$, whereas following signal $T_2$ leads to liability with probability $Q_{1|2}$. We interpret $\ell$ as a reduced-form expected exposure borne by the provider rather than a literal per-case statutory fine. HHS declined to impose strict liability on covered entities for their use of these tools, and it states that the rule does not set any standard of care \citep{cms2024nondiscrimination}. Accordingly, $\ell$ is a reduced-form stand-in for the expected cost of enforcement and litigation conditional on relying on a disparate tool in a type-$y$ case that yields an inappropriate treatment. Because the physician finds following the signal optimal whenever she uses AI, this event is equivalently described in terms of the realized treatment. The physician values both the patient's health outcome net of cost and her own nonclinical payoff, placing weight $\theta$ on the latter. \Cref{table:physician_payoff} summarizes the physician's payoff, where $\mathbf{1}_{t=y}$ equals 1 if $t=y$ and 0 otherwise.

\vspace{0.15in}
\begin{table}[H]
\centering
\caption{The Physician's Expected Payoff When Treating an Individual Patient Using AI}
\label{table:physician_payoff}
\medskip
\begin{tabular}{lcc}
\toprule
 & \multicolumn{2}{c}{Physician's decision} \\
\cmidrule(lr){2-3}
AI signal & $T_1$ & $T_2$ \\
\midrule
$s_t = T_1$ & $Q_{1|1}\cdot b - c + \theta(r - Q_{2|1}\cdot \mathbf{1}_{t=y}\,\ell)$ & $Q_{2|1}\cdot b - c + \theta r$ \\[6pt]
$s_t = T_2$ & $Q_{1|2}\cdot b - c + \theta r$ & $Q_{2|2}\cdot b - c + \theta(r - Q_{1|2}\cdot \mathbf{1}_{t=y}\,\ell)$ \\
\bottomrule
\end{tabular}
\vspace{0.15in}
\end{table}

\cref{fig:discrimination_time_sequence} summarizes the timing. First, the AI firm chooses whether to offer a disparate design or an equal-accuracy design and selects the corresponding accuracy level(s) $\rho_t$ for each patient type. Second, the physician decides whether to use AI for a given patient. If she uses AI, the algorithm generates a treatment signal, and the physician then decides whether to follow it; if she does not use AI, she selects treatment based only on her prior belief. Finally, the patient's health outcome is realized and provider liability is assessed under \S~92.210. We solve the game by backward induction. For every liability level at which the firm supplies a disparate design, we assume $\kappa_y<\frac{f\beta (b+\theta \ell)^2}{b(2c-2\theta r+\theta \ell)}$, which ensures that, in equilibrium, the physician uses AI for a strictly positive measure of type-$y$ patients (that is, ${d_y^{\mathrm D}}^*>0$).

\vspace{0.2in}
\begin{figure}[H]
    \centering
    \includegraphics[scale=0.8]{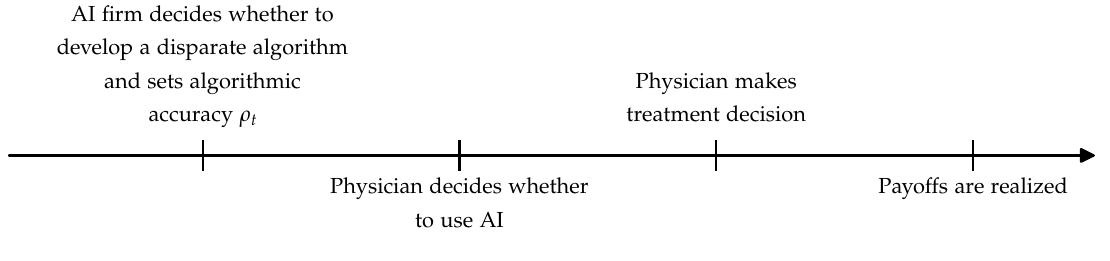}
    \caption{Timing of the Game. The AI firm's first-stage choice is between a disparate design ($\rho_x>\rho_y$) and an equal-accuracy design ($\rho_x=\rho_y$).}
    \label{fig:discrimination_time_sequence}
\end{figure}

\section{Analysis}\label{sec:analysis}
In this section, we first analyze in \cref{sec:downstream} the physician's decision of whether and how to use AI, taking AI accuracy as given. We then analyze in \cref{sec:upstream} how the AI firm chooses AI accuracy.%

\subsection{Downstream Physician Decisions}\label{sec:downstream}

We analyze the physician's decision of whether to use a disparate AI algorithm with $\rho_x> \rho_y>1/2$ for type-$x$ and type-$y$ patients. The ordering $\rho_x>\rho_y$ here defines the disparate-design \emph{subgame} whose downstream behavior we characterize; it is not an equilibrium restriction. \Cref{prop:developer_rho*} shows that, whenever the firm finds a disparate design optimal, $\rho_x^*>\rho_y^*$ emerges endogenously, so the ordering is derived rather than assumed.
Note that the analysis of the physician's decision to use an equal-accuracy algorithm ($\rho_x=\rho_y>1/2$) is analogous to the analysis for type-$x$ patients under a disparate algorithm.

If the physician does not use AI, then, by comparing the physician's payoff from prescribing treatment  $T_1$ (that is, $\alpha b$) and $T_2$ (that is, $(1-\alpha) b$), we know the physician prescribes treatment  $T_1$ if $\alpha>1/2$ and prescribes treatment  $T_2$ if $\alpha\le 1/2$. We can summarize the physician's expected payoff when she does not use AI as follows:
{
\begin{align}\label{equ:U}
    U=&
    \begin{cases}
(1-\alpha) b, & \text{if }\alpha\le 1/2 \\
\alpha b, & \text{if } \alpha  > 1/2.
\end{cases}
\end{align}
}

Now suppose the physician uses AI. In this case, the probability of the AI signal suggesting treatment $T_1$ is $\mathcal{P}(s_t=T_1)=\alpha \cdot \rho_t+ (1-\alpha)\cdot (1-\rho_t)$ and that of its suggesting treatment $T_2$ is $\mathcal{P}(s_t=T_2)=\alpha \cdot (1-\rho_t)+ (1-\alpha)\cdot \rho_t$. Next, we analyze the physician's decision of whether to use AI.

\subsubsection{Type-$x$ Patient}

Suppose the physician consults AI for a type-$x$ patient. If the signal is $s_x=T_1$, she compares the payoffs from prescribing $T_1$ and $T_2$, namely $Q_{1|1}b-c+\theta r$ and $Q_{2|1}b-c+\theta r$ (see \cref{table:physician_payoff}). She therefore prescribes $T_1$ if and only if $\alpha>1-\rho_x$, and prescribes $T_2$ otherwise. If the signal is $s_x=T_2$, she compares $Q_{1|2}b-c+\theta r$ and $Q_{2|2}b-c+\theta r$, and prescribes $T_1$ if and only if $\alpha>\rho_x$, and $T_2$ otherwise. Because $1/2<\rho_x<1$, we have $1-\rho_x<\rho_x$, so the prior space is partitioned into three regions. The physician's expected payoff from AI use for a type-$x$ patient is
\begin{align}\label{equ:Ux}
U_x=
\begin{cases}
\mathcal{P}(s_x=T_1)Q_{2|1}b+\mathcal{P}(s_x=T_2)Q_{2|2}b-c+\theta r, & \text{if }\alpha\le 1-\rho_x,\\
\mathcal{P}(s_x=T_1)Q_{1|1}b+\mathcal{P}(s_x=T_2)Q_{2|2}b-c+\theta r, & \text{if }1-\rho_x<\alpha\le\rho_x,\\
\mathcal{P}(s_x=T_1)Q_{1|1}b+\mathcal{P}(s_x=T_2)Q_{1|2}b-c+\theta r, & \text{if }\alpha>\rho_x.
\end{cases}
\end{align}

Next, we present the physician's decision of whether and how to use AI for type-$x$ patients. A comparison of the physician's expected payoff when using AI and not using AI for type-$x$ patients reveals the following lemma. (All proofs are in the Appendix.)

\begin{lemma}\label{lem:AI_decision_imaltruistic_xpatient}
For type-$x$ patients, the physician uses AI if and only if
\[
\frac{(1-\rho_x)b+c-\theta r}{b}<\alpha<\frac{\rho_x b-c+\theta r}{b}.
\]
Conditional on using AI, the physician follows the AI signal.
\end{lemma}

The physician's prior belief $\alpha$ represents the probability that treatment $T_1$ is appropriate before observing any AI recommendation. 
Without AI, the physician follows this prior and prescribes $T_1$ when $\alpha > 1/2$. 
When AI is used, the physician updates her belief using Bayes' rule, and the signal realization moves her posterior toward or away from $T_1$ by an amount that increases in the accuracy $\rho_x$. 
The use of AI creates value only when its signal can change the physician's treatment decision.

Intuitively, when $\alpha$ is very low, the physician already believes $T_2$ is likely appropriate, and AI's informational value cannot justify the additional patient cost $c$. 
When $\alpha$ is very high, the physician is confident in $T_1$ and does not expect AI to improve her decision. 
Only when $\alpha$ lies in the intermediate range does AI meaningfully reduce uncertainty, making its use optimal. 
Moreover, because AI is informative ($1/2 < \rho_x < 1$), whenever the physician chooses to use it, she rationally follows its recommendation.

\subsubsection{Type-$y$ Patient}

Now suppose the physician  uses AI to generate the signal for a type-$y$ patient. If the AI signal is $T_1$, the physician compares her payoffs from prescribing treatments $T_1$ and $T_2$, which are $Q_{1|1}\cdot  b-Q_{2|1}\cdot \theta \ell-c+\theta r$ and  $Q_{2|1}\cdot b-c+\theta r$ (see \cref{table:physician_payoff}), respectively, and prescribes treatment $T_1$ if $\alpha>\frac{(1-\rho_y) (b+\theta \ell)}{b+(1-\rho_y)\theta \ell}$ and treatment $T_2$ if $\alpha\le \frac{(1-\rho_y) (b+\theta \ell)}{b+(1-\rho_y)\theta \ell}$. 
If, instead, 
the AI signal is $T_2$, by comparing her payoffs from prescribing treatments  $T_1$ and $T_2$, which are $Q_{1|2}\cdot b-c+\theta r$ and    $Q_{2|2}\cdot  b-Q_{1|2} \cdot \theta \ell-c+\theta r$ (see \cref{table:physician_payoff}), respectively,   
the physician prescribes treatment $T_1$ if $\alpha>\frac{b \rho_y}{b+(1-\rho_y)\theta \ell}$ and  treatment  $T_2$ if $\alpha\le \frac{b \rho_y}{b+(1-\rho_y)\theta \ell}$. 
Given the assumption that the physician uses AI for at least some type-$y$ patients, we have $\rho_y>\frac{b + 2 c + 2 \theta \ell - 2 \theta r}{2 b + 2 \theta \ell}$. In addition, because $\frac{b + 2 c + 2 \theta \ell - 2 \theta r}{2 b + 2 \theta \ell}>\frac{b + \theta \ell}{2 b + \theta \ell}$, it follows that  $\rho_y>\frac{b + \theta \ell}{2 b + \theta \ell}$. It is then straightforward to verify that 
$\frac{(1-\rho_y) (b+\theta \ell)}{b+(1-\rho_y)\theta \ell}<\frac{1}{2}<\frac{b \rho_y}{b+(1-\rho_y)\theta \ell}$. Therefore, if the physician uses AI for a type-$y$ patient, her expected payoff is given by 

{\scriptsize
\begin{align}\label{equ:Uy}
    U_y=&
    \begin{cases}
\mathcal P(s_y=T_1)\cdot Q_{2|1}\cdot b+ \mathcal P(s_y=T_2)\cdot \left(Q_{2|2}\cdot  b-Q_{1|2}\theta \ell\right)-c+\theta r, & \text{if }\alpha\le \frac{(1-\rho_y) (b+\theta \ell)}{b+(1-\rho_y)\theta \ell} \\
\mathcal P(s_y=T_1)\cdot \left(Q_{1|1}\cdot b-Q_{2|1} \theta \ell\right)+ \mathcal P(s_y=T_2)\cdot 
\left(Q_{2|2}\cdot  b-Q_{1|2}\theta \ell\right)-c+\theta r, & \text{if } \frac{(1-\rho_y) (b+\theta \ell)}{b+(1-\rho_y)\theta \ell} < \alpha \le \frac{b \rho_y}{b+(1-\rho_y)\theta \ell}\\
\mathcal P(s_y=T_1)\cdot \left(Q_{1|1}\cdot b-Q_{2|1} \theta \ell\right)+\mathcal  P(s_y=T_2)\cdot Q_{1|2}\cdot  b-c+\theta r, & \text{otherwise.}
\end{cases}
\end{align}
} 
The following lemma describes the physician's decision of whether to use AI for type-$y$ patients.

\begin{lemma}\label{lem:AI_decision_imaltruistic_ypatient}
For type-$y$ patients, the physician uses AI if and only if
\[
\frac{(1-\rho_y)(b+\theta \ell)+c-\theta r}{b}<\alpha<\frac{\rho_y b-(1-\rho_y)\theta \ell-c+\theta r}{b}.
\]
Conditional on using AI, the physician follows the AI signal.
\end{lemma}

As in the type-$x$ case, the physician uses AI for type-$y$ patients only when the informational value of consultation is sufficiently high. Equivalently, AI use occurs over an intermediate range of $\alpha$, in which clinical uncertainty about the appropriate treatment is greatest. For an existing algorithm (that is, exogenous $\rho_y$), liability unambiguously makes AI less attractive for type-$y$ patients. Holding $\rho_y$ fixed, an increase in liability $\ell$ contracts the set of priors for which AI is used.

Because liability is triggered when a physician follows an erroneous recommendation from a disparate algorithm in type-$y$ cases, one might expect strategic rejection of AI advice. Our model shows otherwise: conditional on consulting AI, the physician optimally follows the signal. If she consults AI but ignores the signal regardless of its realization, she incurs the consultation cost $c$ without obtaining any informational benefit; with $c>\theta r$, she is strictly better off not consulting at all. If she follows only when the signal is $T_1$ (or only when it is $T_2$), her behavior is equivalent to always choosing $T_1$ (or always choosing $T_2$), which is again weakly dominated by making that fixed treatment choice without AI. Hence, when the physician uses AI for type-$y$ patients, she follows the AI signal.

\label{page:usegap-implication}Next, we examine how the regulation affects the physician's relative use of AI for type-$x$ and type-$y$ patients. Comparing the AI-use intervals in \cref{lem:AI_decision_imaltruistic_xpatient,lem:AI_decision_imaltruistic_ypatient} shows that, in the presence of liability, the physician is (weakly) less likely to use AI for type-$y$ patients than for type-$x$ patients.

In our model of the Section~1557 provision, a physician is exposed to liability when her reliance on a disparate AI tool leads to an inappropriate treatment for a type-$y$ patient. This liability adds an expected cost to AI use in type-$y$ cases, making adoption and reliance more selective for that group, even when the same tool is available for type-$x$ patients. By contrast, type-$x$ cases do not activate this liability channel; therefore, AI use for type-$x$ patients is governed only by the accuracy--cost--incentive trade-off in \cref{lem:AI_decision_imaltruistic_xpatient}. The resulting asymmetry in expected liability generates unequal utilization in equilibrium, with lower AI use among disadvantaged patients, the very group the policy is intended to protect.

\subsection{Upstream AI Model Design}\label{sec:upstream}

Anticipating the physician's downstream use of the algorithm, the AI firm chooses whether to offer a disparate design or an equal-accuracy design, and selects the associated accuracy levels $\rho_t^*$ for each patient type. By \cref{lem:AI_decision_imaltruistic_xpatient,lem:AI_decision_imaltruistic_ypatient}, under a disparate algorithm, the volumes of type-$x$ and type-$y$ patients for whom the physician uses AI are
\begin{align}\label{equ:AI demand}
d_x^{\mathrm D} &= \frac{(2\rho_x-1)b - 2c + 2\theta r}{b} \text{ and } d_y^{\mathrm D} = \beta \cdot \frac{(2\rho_y-1)b - 2(1-\rho_y)\theta \ell - 2c + 2\theta r}{b}.
\end{align}
Here $0<\beta\le 1$ scales the mass of type-$y$ patients relative to type-$x$ patients, and the superscript $\mathrm D$ denotes the disparate design. The firm then chooses $(\rho_x,\rho_y)$ to maximize expected profit in \cref{equ:developer}.

Next, consider the equal-accuracy design, under which the algorithm attains the same accuracy for both patient types, $\rho_x=\rho_y\equiv \rho$. The resulting volumes of AI use for type-$x$ and type-$y$ patients are
\begin{align}\label{equ:AI demand_undisparity}
d_x^{\mathrm E} = \frac{(2\rho-1)b - 2c + 2\theta r}{b}  \text{ and } 
d_y^{\mathrm E} = \beta \cdot \frac{(2\rho-1)b - 2c + 2\theta r}{b},
\end{align}
where the superscript $\mathrm E$ denotes the equal-accuracy design. The firm then chooses $\rho$ to maximize expected profit in \cref{equ:optimal_rho_fair}.

For ease of presentation, we define a threshold $\widetilde \ell$ such that the AI firm's expected profit for a disparate algorithm (which is strictly decreasing in $\ell$) exceeds its expected profit for an equal-accuracy algorithm when $\ell<\widetilde\ell$. Otherwise, the AI firm develops an equal-accuracy algorithm in equilibrium. The following proposition presents the equilibrium AI accuracy set by the firm.

\begin{proposition}\label{prop:developer_rho*}
Suppose $\beta(\beta+2)\kappa_x > \kappa_y$ and  \( \max\{\frac{2\beta(c-\theta r)\kappa_x(\kappa_x+\kappa_y)}{b(\beta(\beta+2)\kappa_x-\kappa_y)}, \frac{b \kappa_y (2c-2\theta r+\theta \ell)}{\beta (b+\theta \ell)^2}\} < f < \min\{\frac{\kappa_x}{2}, \frac{b \kappa_y}{2\beta (b+\theta \ell)}\} \). A unique cutoff $\widetilde{\ell}$ exists in the feasible region such that:\footnote{The bounds that involve $\ell$ need to hold only for $\ell<\widetilde{\ell}$.}
\begin{enumerate}
\item[(a)] if \( \ell < \widetilde{\ell} \), the AI firm develops a disparate algorithm with \( \rho_x^*=\frac{1}{2}+\frac{f}{\kappa_x} \) and \( \rho_y^*=\frac{1}{2}+\frac{\beta f (b+\theta \ell)}{b\kappa_y} \), and \( \rho_x^*>\rho_y^* \);
\item[(b)] if \( \ell \ge \widetilde{\ell} \), the AI firm develops an equal-accuracy algorithm with \( \rho^*=\frac{1}{2}+\frac{(1+\beta)f}{\kappa_x+\kappa_y} \).
\end{enumerate}

\end{proposition}

Recall that the AI firm earns a per-use fee \(f\) whenever the physician consults AI and incurs a cost \( \kappa_t(\rho_t-\tfrac{1}{2})^2 \) to deliver accuracy \( \rho_t \) for type-\(t\) patients. Because improving accuracy for type-$y$ patients is more costly (\( \kappa_y>\kappa_x \)) and demand from that segment is smaller (\( \beta\le 1 \)), the firm optimally allocates more performance to type-$x$ patients in the disparate regime, so \( \rho_x^*>\rho_y^* \).

As liability \( \ell \) rises, physician demand for AI in type-$y$ cases becomes more sensitive to expected legal exposure, reducing the profitability of a disparate design. Once \( \ell \) exceeds \( \widetilde{\ell} \), the firm optimally switches to an equal-accuracy design, which removes the disparity-triggered liability channel and helps preserve overall use. Although \( \rho_y^* \) increases with \( \ell \) within the disparate regime, it does not overtake \( \rho_x^* \); before that point, the firm prefers to switch to an equal-accuracy design.

 \cref{prop:developer_rho*} implies that stronger liability pushes design toward fairness, potentially at the cost of aggregate efficiency in accuracy investment. Weak liability sustains a profit-driven disparate design, whereas strong liability induces convergence to equal accuracy.

\section{Managerial and Policy Insights}\label{sec:result}

We now examine how liability influences AI-use disparity, appropriate AI use, and patient welfare.

\subsection{AI-Use Disparity}\label{sec:disparity}

Differential costs or constraints can make physicians' AI use differ systematically across patient types. We refer to such differences as \emph{AI-use disparity}, defined as the gap between the share of type-$x$ cases and the share of type-$y$ cases in which the physician uses AI.

AI-use disparity may arise under the liability channel we model. Under certain conditions, liability may fail to mitigate disparities originating from algorithm development and may instead reduce physicians' incentives to use AI for disadvantaged patients. As a result, the patients whom the rule is designed to protect may be less likely to benefit from AI-assisted decision-making. The following proposition characterizes when liability exacerbates AI-use disparity and when such disparity narrows. For ease of presentation, we define a threshold $f_l$ in the proof of \cref{prop:disparity}.

\begin{proposition}\label{prop:disparity}
Suppose $\beta (\beta+2) \kappa_x > \kappa_y$ and 
$f_l
< f< \min \left\{
\frac{\kappa_x}{2},  \frac{\kappa_y}{4\beta}\right\}.
$
   Then, type-$y$ AI use is non-monotone in liability: as $\ell$ increases, the physician uses AI for fewer type-$y$ patients when $\ell<\frac{b(\kappa_y-4\beta f)}{4\theta\beta f}$ and for more type-$y$ patients when $\frac{b(\kappa_y-4\beta f)}{4\theta\beta f}\le \ell<\widetilde{\ell}$. When $\ell>\widetilde{\ell}$, the firm switches to an equal-accuracy design, in which case type-$y$ AI use does not depend on $\ell$.
\end{proposition}

\begin{figure}[htbp]
    \centering
    \includegraphics[scale=0.65]{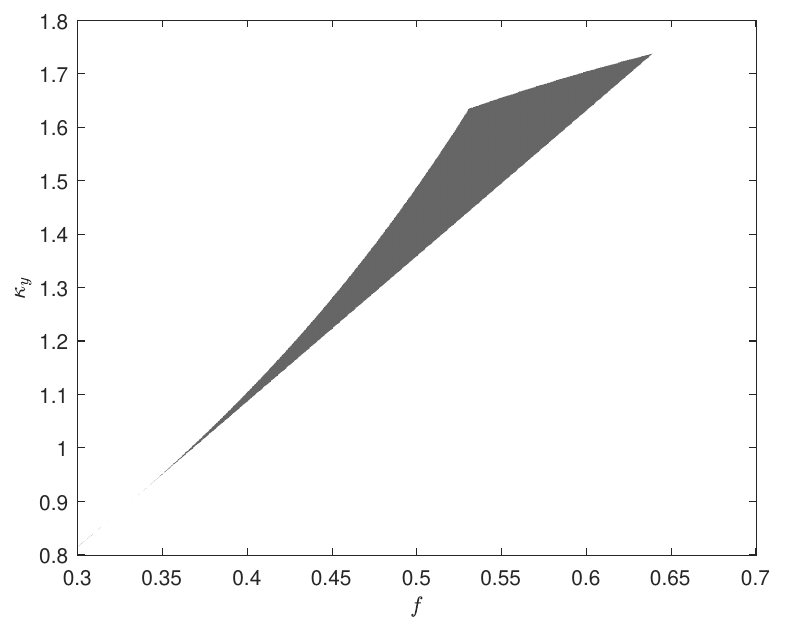}
    \caption{Parameter region in which type-$y$ AI use is non-monotone in liability (\cref{prop:disparity}).}
    \label{fig:regionProp3}
\end{figure}

Since HHS proposed \S~92.210 in August 2022, commentators have warned that liability exposure could lead providers to pull back from clinical algorithms, to the detriment of patient care \citep[see, e.g.,][]{goodman2023clinical,shachar2023prevention}. \cref{prop:disparity} sharpens this point: an increase in liability $\ell$ can either exacerbate or mitigate AI-use disparity, depending on its level. The mechanism is a tug-of-war between two forces. Downstream, higher liability discourages the physician's AI use for type-$y$ patients by directly raising the physician's expected liability cost. Upstream, higher liability increases the firm's incentive to improve accuracy for type-$y$ patients, which makes AI more attractive to the physician and can offset (or reverse) the deterrence effect.

When liability $\ell$ is small, an increase in its level causes a limited reduction in the physician's AI use for type-$y$ patients. In this case, the firm has a small incentive to invest in type-$y$ accuracy. Therefore, the direct deterrence channel dominates, and the physician's AI use for type-$y$ patients falls as $\ell$ rises, thus amplifying the AI-use disparity. When liability $\ell$ is large, the contraction in the physician's AI use for type-$y$ patients becomes salient for the firm, prompting greater investment in type-$y$ accuracy. Because higher accuracy reduces the likelihood of inappropriate treatment and therefore reduces the physician's expected liability cost, the physician's AI use for type-$y$ patients can increase as $\ell$ rises, narrowing AI-use disparity.

The bounds on $f$ in \cref{prop:disparity} delimit the range over which the two forces are comparable in magnitude. When $f$ is small, the disparate design stops being optimal before liability reaches the turning point: the firm switches to an equal-accuracy design or, for still smaller $f$, stops serving type-$y$ patients, so type-$y$ use declines throughout the disparate range. When $f$ is large, $\rho_y^*$ responds strongly to $\ell$, and the resulting accuracy gains outweigh the deterrence effect at every liability level, so type-$y$ use increases monotonically in $\ell$. Only over the intermediate range does the firm both prefer a disparate design and respond moderately enough for deterrence to dominate at low liability; this combination produces the turning point. \cref{fig:regionProp3} shows the parameter region in which non-monotone AI use occurs.

Liability that attaches to disparities in outcomes gives the physician a reason to consult AI less often for patients in the protected group. A provider that did this openly, by conditioning AI use on a patient's race or another protected characteristic, would face a claim of disparate treatment. The incentive is therefore more likely to show up where decisions are hard to audit, for example in how readily a clinician turns to the tool when the diagnosis is uncertain. A rule aimed at disparate impact can, in this way, push providers toward conduct that the same statute prohibits as disparate treatment.

\subsection{(In)Appropriate AI Use}\label{sec:appropriate}

Consider an altruistic physician whose only concern is a patient's health outcome and costs (that is, $\theta=0$). From \cref{equ:Ux,equ:Uy}, the expected payoff of an altruistic physician who uses AI for an individual type-$t$ patient is given as follows:

{\footnotesize
\begin{align}\label{eq:Ual}
    U^{\text{al}}=&
    \begin{cases}
\mathcal P(s_t=T_1)\cdot Q_{2|1}\cdot b+ \mathcal P(s_t=T_2)\cdot Q_{2|2}\cdot  b-c, & \text{if }\alpha\le 1-\rho_t \\
\mathcal P(s_t=T_1)\cdot Q_{1|1}\cdot b+ \mathcal P(s_t=T_2)\cdot Q_{2|2}\cdot  b-c, & \text{if } 1-\rho_t < \alpha \le \rho_t \\
\mathcal P(s_t=T_1)\cdot Q_{1|1}\cdot b+ \mathcal P(s_t=T_2)\cdot Q_{1|2}\cdot  b-c, & \text{otherwise}
\end{cases}
\end{align}
}
A comparison of the altruistic physician's expected payoffs from following and disregarding AI's recommendation leads to the following lemma.

\begin{lemma}\label{lem:AI_decision_altruistic}
Suppose $c<(\rho_t-\tfrac12)b$ for $t\in\{x,y\}$. For type-$t$ patients, an altruistic physician uses AI if and only if
\[
\frac{(1-\rho_t)b+c}{b}<\alpha<\frac{\rho_t b-c}{b}.
\]
Conditional on using AI, the physician follows the AI signal.
\end{lemma}

We define a physician's AI use for a specific patient type $t$ as \emph{appropriate} if it matches the behavior of an altruistic physician (with $\theta=0$), a benchmark defined within our model. If the physician uses AI for more patients than an altruistic physician would, we refer to this as \emph{overuse} of AI; if the physician uses AI for fewer patients, we refer to it as \emph{underuse} of AI. The following proposition compares AI use under the two physician types.  We define a threshold $\ell_s$ (expression provided in the Appendix) at which the physician's AI use for type-$y$ patients switches from overuse to underuse.

\begin{proposition} The physician's use of AI relative to the clinically appropriate benchmark satisfies the following. \label{prop:fairness_appropriateUse}
\begin{itemize}
\item[(a)] When the AI firm supplies an equal-accuracy algorithm, the physician may overuse AI for both patient types.

\item[(b)] Suppose the interior disparate design is optimal at $\ell$, with $\pi_y(\rho_y^*)>0$ and $0<r/\ell<1/2$. If $\frac{1}{2}-\frac{\beta f(b+\theta \ell)}{b\kappa_y}=\frac{r}{\ell}$, then $\rho_y^*=1-\frac{r}{\ell}$ and the physician uses AI appropriately for type-$y$ patients. In this case, $\rho_x^*=\frac{1}{2}+\frac{f}{\kappa_x}$ and the physician overuses AI for type-$x$ patients.

\item[(c)] Suppose $\pi_y(\rho_y^*)>0$ and a disparate design is optimal at $\ell=\ell_s$ (specifically, conditions \eqref{eq:pos_profit_y_ls} and \eqref{eq:optimal_bias_ls} in the Appendix are satisfied). If $\beta (\beta+2) \kappa_x > \kappa_y$ and  $\frac{2 \beta (c -  \theta r) \kappa_x (\kappa_x+\kappa_y)} {b (\beta (\beta+2) \kappa_x-\kappa_y)} <f < \min\{\frac{\kappa_x}{2},  \frac{ b \kappa_y + 4 \kappa_y \theta r}{2 b \beta} - 
   \frac{\kappa_y \sqrt{2b   \theta r+ 4  \theta^2 r^2}}{b \beta}\}$, then as liability $\ell$
increases from a low level, the physician's AI use for type-$y$ patients is non-monotone: she overuses AI for small $\ell$, underuses AI for intermediate $\ell$, and overuses AI again for large $\ell$.\footnote{\cref{fig:regionProp4} provides a numerical illustration of the parameter region in which tri-phase non-monotone AI use occurs.}

\label{page:prop3-end}\end{itemize}
\end{proposition}

\begin{figure}[htbp]
    \centering
    \includegraphics[scale=0.65]{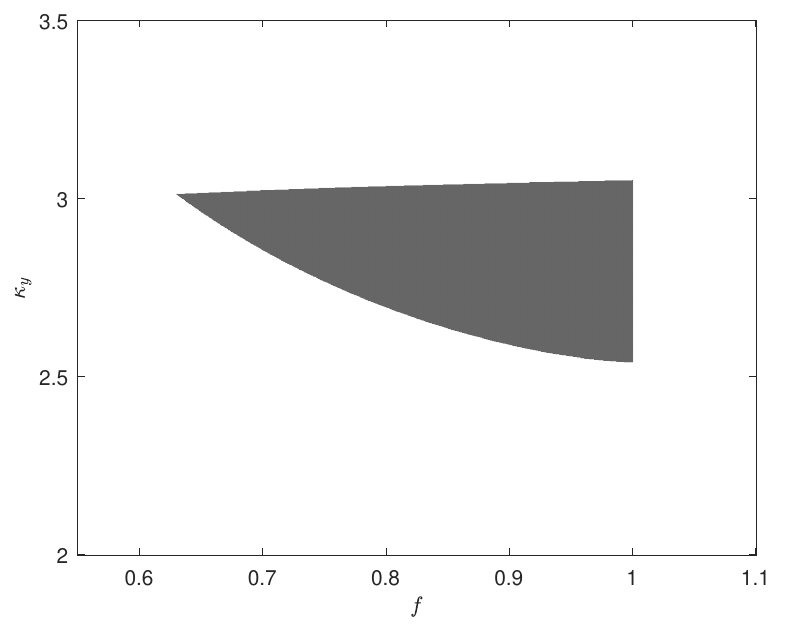}
    \caption{Parameter region in which type-$y$ AI use switches between overuse and underuse (\cref{prop:fairness_appropriateUse}(c)).}
    \label{fig:regionProp4}
\end{figure}

In \cref{prop:fairness_appropriateUse}(a), an equal-accuracy design removes the type-$y$-specific liability channel in our setup but leaves intact the physician's private incentive to use the tool when use is rewarded. When reimbursement $r$ remains in place, the physician's marginal calculus can tilt toward reliance even when AI is not clinically warranted, leading to overuse for both patient types. Equalizing algorithmic performance does not, by itself, align deployment incentives; if liability no longer disciplines use under equal accuracy, reimbursement can become the dominant force shaping utilization. Policy may then need complementary instruments that better tie payment and accountability to clinical value.

\cref{prop:fairness_appropriateUse}(b) shows that appropriate use for disadvantaged patients can arise even when the supplied algorithm has unequal accuracy across groups, provided reimbursement exactly offsets the physician's expected liability exposure from relying on AI in type-$y$ cases. The margin condition is $r=(1-\rho_y^*)\ell$. Using the equilibrium expression for $\rho_y^*$, this is equivalent to $\frac{r}{\ell}=\frac{1}{2}-\frac{\beta f(b+\theta \ell)}{b\kappa_y}$, which characterizes the combinations of liability and reimbursement that restore appropriate use for type-$y$ patients.

We illustrate \cref{prop:fairness_appropriateUse}(c) in \cref{fig:backfire_underuse_c}. The solid lines plot the belief cutoffs in $\alpha$ that delimit when the physician uses AI for type-$y$ patients; the dash-dot lines report the corresponding cutoffs for an altruistic physician ($\theta=0$). As liability $\ell$ rises, equilibrium use for type-$y$ patients can be non-monotone: overuse at low $\ell$, underuse at intermediate $\ell$, and overuse again when $\ell$ is large. This pattern contrasts with the concern that liability primarily discourages reliance and leads the most risk-averse providers to stop using clinical algorithms that incorporate protected traits \citep[see, e.g.,][]{goodman2023clinical}. The mechanism is a shifting balance between downstream deterrence and upstream design responses. For small $\ell$, reimbursement dominates expected liability, so the physician relies on the tool too often. For intermediate $\ell$, liability becomes first-order and suppresses reliance, generating underuse. For sufficiently large $\ell$, the firm is induced to supply an equal-accuracy design, which relaxes the liability channel for type-$y$ patients and restores the reimbursement-driven incentive to use AI, again producing overuse. In the low-$\ell$ region in which the physician overuses a disparate tool for type-$y$ patients, increasing $\ell$ can widen utilization differences across patient types even as the policy is intended to protect the disadvantaged group.

\begin{figure}[htbp]
    \centering
    \includegraphics[scale=0.65]{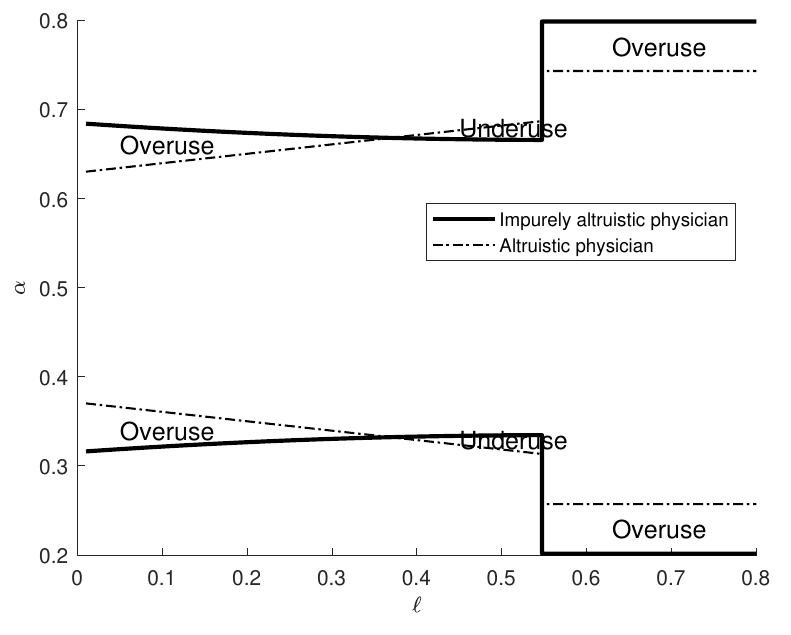}
    \caption{AI use for type-$y$ patients: for a given $\ell$, the physician uses AI when $\alpha$ lies between the cutoffs. Dash-dot lines correspond to an altruistic physician ($\theta=0$); solid lines correspond to an impurely altruistic physician.}
    \label{fig:backfire_underuse_c}
\end{figure}

\subsection{Effect of Mandating Equal Accuracy}\label{sec:welfare}

Policy discussions have emphasized reducing performance differences across protected groups \citep[e.g.,][]{WhiteHouse2023}, and an equal-accuracy requirement is the natural limit of that goal. In our setting, however, requiring equal accuracy across patient types reshapes both the firm's investment incentives and physicians' deployment incentives, with direct consequences for patient welfare. Holding fixed the liability framework, this section examines the welfare effects of a hypothetical equal-accuracy requirement, which is distinct from the obligations in \S~92.210.

We consider two policy routes to equal accuracy. The first is a direct mandate that forces equal accuracy in cases in which the firm would otherwise choose a disparate design. The second is an indirect approach that raises physician liability so that the firm endogenously prefers an equal-accuracy design. Because liability does not affect the firm's accuracy choice conditional on an equal-accuracy design, these two routes deliver the same equilibrium accuracy and thus the same expected patient welfare. We therefore focus, without loss of generality, on the direct mandate and characterize how imposing equal accuracy changes equilibrium accuracy and welfare.

Let $W_x(\rho_x^*)$ and $W_y(\rho_y^*)$ denote expected welfare for type-$x$ and type-$y$ patients, respectively, under the disparate design evaluated at the firm's optimal accuracies $(\rho_x^*,\rho_y^*)$. Let $\rho^m$ denote the accuracy level chosen under an equal-accuracy requirement, even in cases in which the firm would otherwise prefer a disparate design. Define $W_x(\rho^m)$ and $W_y(\rho^m)$ as the corresponding welfare levels under the equal-accuracy design. Closed-form expressions for $W_x(\rho_x^*)$, $W_y(\rho_y^*)$, $W_x(\rho^m)$, and $W_y(\rho^m)$ are provided in the proof of \cref{prop:mandate_fairness_worse}. For ease of exposition, we define a threshold $\bar f$ by the indifference condition $W_y(\rho^m)=W_y(\rho_y^*)$ evaluated at $f=\bar f$.

\begin{proposition}\label{prop:mandate_fairness_worse} When equal accuracy is mandated: 
\begin{itemize}
\item[(a)] The optimal AI accuracy  satisfies \( \rho_y^* \leq \rho^m \leq \rho_x^* \).  

\item[(b)] The expected patient surplus remains unchanged for both type-$x$ and type-$y$ patients when the physician is fairly certain about the appropriate treatment (that is, when $\alpha$ is close to 0 or 1). It worsens (improves) for type-$x$ (type-$y$) patients when the physician is highly uncertain (that is, when $\alpha$ is close to $1/2$). Otherwise, that is, when the uncertainty is moderate, type-$x$ (type-$y$) patients are better (worse) off.\footnote{The expressions for the thresholds (which differ for type-$x$ and type-$y$ patients) that specify the range of $\alpha$ in which patient welfare improves, worsens, or remains unchanged are provided in the proof of \cref{prop:mandate_fairness_worse}. }  

\item[(c)]  Suppose $c <\frac{\kappa_y ((2 + \beta) \kappa_x + \kappa_y) \theta r}{\kappa_y (2 \kappa_x + \kappa_y)-\beta \kappa_x^2}$ and $b > \frac{4 c \kappa_y (\kappa_x + \kappa_y)}{\kappa_x ( \kappa_y + \beta (\kappa_x +2 \kappa_y))}$. When $\ell$ is small, mandating equal accuracy yields three regimes. (1) 
If $\frac{b \kappa_y (2c-2\theta r+\theta \ell)}{\beta (b+\theta \ell)^2}<f<\frac{2 c \kappa_x (\kappa_x +\kappa_y)}{b ((2 + \beta) \kappa_x + \kappa_y)}$, then $W_x(\rho^m)>W_x(\rho_x^*)$ and $W_y(\rho^m)<W_y(\rho_y^*)$.  (2) If $\frac{2 c \kappa_x (\kappa_x +\kappa_y)}{b ((2 + \beta) \kappa_x + \kappa_y)} < f< \bar f$, then $W_x(\rho^m)<W_x(\rho_x^*)$ and $W_y(\rho^m)<W_y(\rho_y^*)$; at $f=\frac{2 c \kappa_x (\kappa_x +\kappa_y)}{b ((2 + \beta) \kappa_x + \kappa_y)}$, type-$x$ welfare is unchanged, and at $f=\bar f$, type-$y$ welfare is unchanged. (3) Finally, if $\bar f<f<\min\{\frac{\kappa_x}{2}, \frac{b \kappa_y}{2\beta (b+\theta \ell)}\}$, then $W_x(\rho^m)<W_x(\rho_x^*)$ and $W_y(\rho^m)>W_y(\rho_y^*)$.%
\footnote{Comparisons of patient welfare over a wide range of $(\beta,\ell,f)$ values are provided in Section~\ref{sec:OA_welfare} of the Online Appendix.}
 
\label{page:prop4-end}\end{itemize} \end{proposition}

\cref{prop:mandate_fairness_worse}(a) shows that imposing an equal-accuracy requirement reshuffles performance across groups: relative to the firm's preferred disparate design, accuracy falls for type-$x$ patients (so $\rho^m\le \rho_x^*$) and rises for type-$y$ patients (so $\rho^m\ge \rho_y^*$). Whether a given patient benefits, however, depends both on accuracy and on how the requirement changes the physician's use of AI along the margin, as \cref{prop:mandate_fairness_worse}(b) shows.

The mechanism works through utilization. By raising accuracy for type-$y$ patients and lowering it for type-$x$ patients, the equal-accuracy requirement expands AI use for some type-$y$ patients and contracts AI use for some type-$x$ patients. For the newly treated type-$y$ margin, welfare can fall when physician uncertainty is moderate: the incremental health improvement from relying on AI is then small, although the fixed cost of using AI, $c$, is paid whenever the physician deploys the tool. For the dropped type-$x$ margin, welfare can rise for the symmetric reason: these patients forgo a modest accuracy benefit but avoid incurring the cost $c$. The ranges of moderate uncertainty differ by type, and wherever type-$y$ patients are worse off, type-$x$ patients are worse off as well.

Turning to aggregate welfare by patient type, \cref{prop:mandate_fairness_worse}(c) shows that when $\ell$ is small, an equal-accuracy requirement can reduce aggregate welfare for both groups when the per-use payment $f$ lies in an intermediate range. Two forces drive this outcome. First, the accuracy adjustment required to equalize performance is asymmetric: because improving type-$y$ accuracy is more expensive, parity is achieved primarily by cutting type-$x$ accuracy, so $\rho_x^*-\rho^m>\rho^m-\rho_y^*$. Second, the requirement shifts utilization incentives. For type-$y$ patients, higher accuracy raises welfare directly but also expands AI use, which can push deployment beyond the clinical margin; both effects strengthen with $f$, and the overuse channel can dominate when $f$ is moderate. For type-$x$ patients, reduced use can mitigate overuse, but the loss from lower accuracy becomes increasingly important as $f$ induces larger cuts. When $f$ is neither too small to matter nor large enough to generate substantial type-$y$ accuracy gains, the combination of a sizable decline in type-$x$ accuracy and expanded (and potentially excessive) use for type-$y$ patients can leave both groups worse off.

When $f$ is sufficiently large, the balance shifts. The equal-accuracy requirement then induces a substantial reallocation of accuracy from type-$x$ to type-$y$, yielding the intuitive aggregate ordering: $W_x(\rho^m)<W_x(\rho_x^*)$, whereas $W_y(\rho^m)>W_y(\rho_y^*)$. At the same time, as \cref{prop:mandate_fairness_worse}(b) emphasizes, these aggregate comparisons can mask heterogeneity: some type-$y$ patients can still be worse off under the requirement because expanded reliance on AI occurs precisely in cases in which the clinical gains are small relative to the use cost.

Finally, when $f$ is sufficiently small, welfare differences are driven less by accuracy and more by utilization. In this region, both gaps $\rho_x^*-\rho^m$ and $\rho^m-\rho_y^*$ are small, so the equal-accuracy requirement mainly shifts the frequency of use: it reduces AI use for type-$x$ patients and increases it for type-$y$ patients. Type-$x$ patients can therefore benefit from fewer exposures to a low-accuracy tool on the margin, whereas type-$y$ patients can be harmed because the increase in use is not accompanied by a commensurate improvement in accuracy. In this case, we may have $W_x(\rho^m)>W_x(\rho_x^*)$ yet $W_y(\rho^m)<W_y(\rho_y^*)$.

We next consider an alternative setting in which patient utility accounts for potential liability transfers.\label{page:transfer-surplus} Under this formulation, the patient's utility and the physician's payoff are updated according to \cref{table:patient_utility_withl,table:physician_payoff_withl}.

\vspace{0.15in}
\begin{table}[H]
\centering
\caption{The Patient's Expected Utility When the Physician Uses AI and Liability Transfers Are Included}
\label{table:patient_utility_withl}
\medskip
\begin{tabular}{lcc}
\toprule
 & \multicolumn{2}{c}{Physician's decision} \\
\cmidrule(lr){2-3}
AI signal & $T_1$ & $T_2$ \\
\midrule
$s_t = T_1$ & $Q_{1|1}\cdot b - c + Q_{2|1}\cdot \mathbf{1}_{t=y}\,\ell$ & $Q_{2|1}\cdot b - c $ \\[6pt]
$s_t = T_2$ & $Q_{1|2}\cdot b - c$ & $Q_{2|2}\cdot b - c +  Q_{1|2}\cdot \mathbf{1}_{t=y}\,\ell$ \\
\bottomrule
\end{tabular}
\vspace{0.15in}
\end{table}

\vspace{0.15in}
\begin{table}[H]
\centering
\caption{The Physician's Expected Payoff When Treating an Individual Patient Using AI and Liability Transfers Are Included}
\label{table:physician_payoff_withl}
\medskip
\begin{tabular}{lcc}
\toprule
 & \multicolumn{2}{c}{Physician's decision} \\
\cmidrule(lr){2-3}
AI signal & $T_1$ & $T_2$ \\
\midrule
$s_t = T_1$ & $Q_{1|1}\cdot b - c + \theta r -(\theta-1) Q_{2|1}\cdot \mathbf{1}_{t=y}\,\ell$ & $Q_{2|1}\cdot b - c + \theta r$ \\[6pt]
$s_t = T_2$ & $Q_{1|2}\cdot b - c + \theta r$ & $Q_{2|2}\cdot b - c + \theta r - (\theta-1)Q_{1|2}\cdot \mathbf{1}_{t=y}\,\ell$ \\
\bottomrule
\end{tabular}
\vspace{0.15in}
\end{table}

Note that if $\theta < 1$, the physician 
may use AI solely to exploit the patient's financial gain 
from the liability transfer, rather than to improve clinical outcomes. 
To rule out this uninteresting regime, we impose the parametric 
restriction $\theta > 1$. Under this condition, the physician's payoff 
matrix in \cref{table:physician_payoff_withl} is structurally isomorphic 
to that in \cref{table:physician_payoff}, with the weight on the 
liability term updated from $\theta$ to $\theta-1$.\label{page:transfer-theta} Consequently, \cref{lem:AI_decision_imaltruistic_ypatient} 
continues to hold directly by updating the liability-adjusted selfishness 
parameter to $\theta - 1$, and the subsequent analytical insights 
follow analogously.

\section{Extensions}\label{sec:extensions}
In this section, we present three extensions of the baseline. In the first extension, the AI firm also chooses the per-use price in addition to the two accuracy levels. In the second extension, we explore the liability level that maximizes aggregate patient welfare. Finally, we consider patient-type-dependent priors for the physician. Proofs of the results in this section are in the Online Appendix (Sections~\ref{sec:OA_pricing}, \ref{sec:OA_welfare_max}, and~\ref{sec:OA_priors}).

\subsection{AI Firm's Pricing Decision}\label{sec:pricing_extn}

This section endogenizes the per-use fee $f$ in an extension in which patients pay for AI directly, so $c=f$.\label{page:cf-pricing} This structure captures emerging usage-based arrangements in which providers charge an out-of-pocket fee for AI-enhanced services. For example, RadNet (a large diagnostic imaging company) launched its Enhanced Breast Cancer Detection program in late 2022 at a \$60 out-of-pocket fee for an AI mammography read and has since reduced the fee to \$40 \citep{Cheatham2024Who,Andrews2024Mammography}. Setting $c=f$ assumes the patient's fee passes through one for one to the firm; in practice, the provider's charge and the firm's price need not coincide. The timing parallels the baseline: the firm first chooses accuracy levels; it then sets the fee $f$; the physician observes $(\rho_x,\rho_y,f)$ when deciding whether to use AI for a given patient. All other assumptions are unchanged.
We summarize the core implications here and defer the technical derivations and additional comparative-statics details to Section~\ref{sec:OA_pricing} of the Online Appendix.

\begin{proposition}\label{prop:rho_f}
Let $\ell^\dagger$ denote the design-switch cutoff at which the firm is indifferent between supplying a disparate design and an equal-accuracy design. Define
\[
D(\ell):=2 b^2\!\left(\beta^2 \kappa_x + \kappa_y \right)-4b \kappa_x \!\left( (\beta + 1)\kappa_y - \beta^2 \theta \ell \right) + 2 \beta^2 \kappa_x \theta^2 \ell^2
\]
and the common factors $\mathcal{T}:={\theta\left(\beta \ell - 2(\beta + 1) r\right)}/{D(\ell)}$ and $\mathcal{S}:={\theta r}/\!\left(2(\kappa_x + \kappa_y) - b(\beta + 1)\right)$.
 There exist parameter values such that the equilibrium accuracies are interior, $1/2<\rho^*,\rho_x^*,\rho_y^*<1$, and the following holds: 

(a) If $\ell<\ell^\dagger$, the firm chooses a disparate design with
\[
\rho_x^*  = \frac{1}{2} + b \kappa_y \mathcal{T},
\qquad
\rho_y^*  = \frac{1}{2} + \beta \kappa_x (b + \theta \ell)\, \mathcal{T},
\qquad
f^{\mathrm D}  = b \kappa_x \kappa_y \mathcal{T}.
\]
Moreover, $\rho_x^*>\rho_y^*$, and there are parameter values for which both $f^{\mathrm D}$ and $\rho_x^*$ are non-monotone in $\ell$.

(b) If $\ell\ge \ell^\dagger$, the firm supplies an equal-accuracy design with
\[
\rho^* = \frac{1}{2} + (\beta + 1)\,\mathcal{S}
\qquad\text{and}\qquad
f^{\mathrm E} = (\kappa_x + \kappa_y)\,\mathcal{S}.
\]
\vspace{-22pt}%
\end{proposition}

One implication of \cref{prop:rho_f} is that endogenizing the per-use fee can make both price and utilization respond non-monotonically to liability. When $\ell$ rises from a low level, expected utilization among type-$y$ patients falls, weakening marginal demand; the firm then cuts the fee and tilts investment toward type-$y$ accuracy relative to type-$x$ accuracy. Because $f^{\mathrm D}$ and $\rho_x^*-\frac{1}{2}$ are both proportional to $\mathcal{T}$, the fee cut comes with lower type-$x$ accuracy, and type-$x$ usage falls along with both. As $\ell$ becomes larger, the return to improving type-$y$ accuracy strengthens, and type-$y$ utilization rebounds; the firm then raises the fee and type-$x$ accuracy together, and type-$x$ usage recovers.

We also find that, within the disparate-design regime, type-$y$ utilization is non-monotone in $\ell$ for some parameter values: it decreases when $\ell$ is below a threshold but increases once $\ell$ exceeds that threshold. Thus, the non-monotonicity in AI use for type-$y$ patients persists even when the firm endogenizes the per-use fee $f$ and patients bear that fee, so that $c=f$. As in the baseline, an equal-accuracy requirement can still reduce welfare for both patient types when the AI firm endogenizes the per-use price.

The endogenous-pricing extension yields one further result: because $f$ is common to both
segments, liability also affects type-$x$ welfare. When liability makes the type-$y$ segment
unprofitable, the firm raises both its fee $f$ and type-$x$ accuracy, so mandating equal accuracy
harms type-$x$ patients over a wider range of $\beta$ than with exogenous $f$ (see \cref{fig:welfare_equalacc_fequalc}).

\subsection{Aggregate-Welfare-Maximizing Liability}\label{sec:welfaremax_extn}
This subsection studies how the liability level $\ell$ should be set to maximize aggregate expected welfare. Two features of the model simplify the problem. First, holding the firm's AI design fixed, expected total welfare for type-$x$ patients under a disparate design does not depend on $\ell$, and aggregate welfare under an equal-accuracy design does not depend on $\ell$. Second, \cref{prop:developer_rho*} implies that increasing $\ell$ can change the firm's preferred design, inducing a switch from a disparate design to an equal-accuracy design at a threshold $\widetilde \ell$. As a result, liability affects aggregate welfare through type-$y$ welfare within the disparate regime and through whether it triggers this design transition.
We present the main economic intuition below and refer readers to Section~\ref{sec:OA_welfare_max} of the Online Appendix for computational details and supplementary parameterizations.

Across the parameterizations we examine numerically, the welfare-maximizing liability lies at the boundary of this switch. Specifically, it is either the largest $\ell$ for which the firm still prefers a disparate design (see \cref{fig:fig1}) or the smallest $\ell$ for which the firm prefers an equal-accuracy design (see \cref{fig:fig2}). Thus, from the perspective of aggregate welfare, choosing $\ell$ amounts to choosing the design regime: keeping the system in the disparate-design region or inducing the shift to equal accuracy. Throughout, aggregate welfare refers to total patient health surplus, summed across both patient groups. We do not interpret it as full social welfare: doing so would require a common criterion encompassing patient health surplus, physician utility, firm profit, and any deadweight legal or compliance costs. We regard a full social-welfare analysis as a natural direction for future work.

\begin{figure}[htbp]
    \centering
    \begin{subfigure}[b]{0.4\textwidth}
        \centering
        \includegraphics[width=\textwidth]{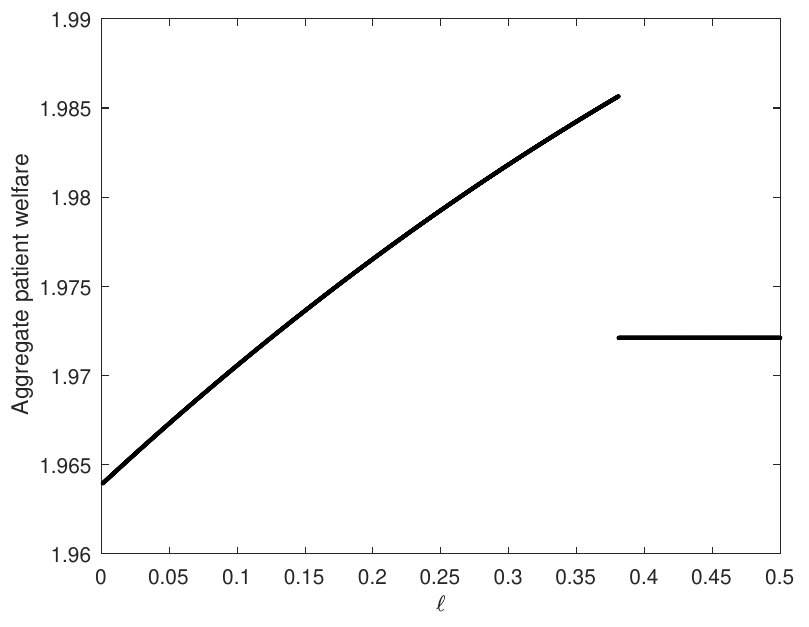}
        \caption{Case in which welfare is higher under a disparate design.}
       \label{fig:fig1}
    \end{subfigure}
    \quad
    \begin{subfigure}[b]{0.4\textwidth}
        \centering
        \includegraphics[width=\textwidth]{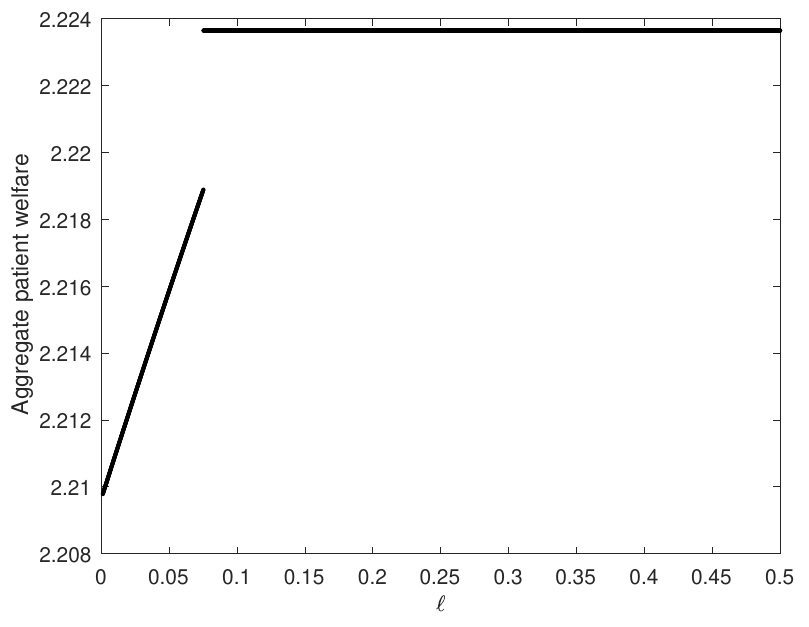}
        \caption{Case in which welfare is higher under an equal-accuracy design.}
        \label{fig:fig2}
    \end{subfigure}
    \caption{Aggregate welfare as a function of liability $\ell$. In both examples, the upward-sloping segment is welfare under the firm's preferred disparate design; the flat segment is welfare once the firm switches to an equal-accuracy design, which does not vary with $\ell$.}
    \label{fig:welfare_max}
\end{figure}

\subsection{Type-Specific Priors}\label{sec:type-specific-priors}
In this section, we allow physicians to hold type-specific priors about whether treatment $T_1$ is appropriate. This extension addresses the concern that imposing $\alpha_x=\alpha_y$ may understate human-driven disparities in baseline clinical beliefs. %
Specifically, we assume $\alpha_x\sim U(0,\overline{\alpha})$ and $\alpha_y\sim U(\underline{\alpha},1)$, where $\underline{\alpha}<1/2<\overline{\alpha}$. All other assumptions remain as in the baseline. In particular, we maintain parameter restrictions under which, at every optimum characterized in this extension (and, under the disparate design, at every liability level up to the design-switch cutoff), demand for each patient type and the firm's profit are strictly positive, and the optimal accuracies are interior in $(1/2,1)$. The numerical example reported in this extension satisfies these restrictions; \cref{fig:welfare_alphat} in the Online Appendix also marks parameter values at which the firm serves type-$x$ patients only.
Throughout this extension, we define
$\tilde f := f/\overline{\alpha}$ and $\tilde{\beta} := \beta\overline{\alpha}/(1-\underline{\alpha})$
for the density-adjusted fee and composition. We assume $\overline{\alpha}+\underline{\alpha}>1$ and $\tilde{\beta}<1$; the latter is the counterpart of the baseline restriction $\beta\le 1$.

We summarize our key findings and relegate our technical details to Section~\ref{sec:OA_priors} of the Online Appendix.
The physician's optimal AI-use rule retains the same threshold form as in the baseline: AI is used only for an interior range of priors, and, conditional on use, the physician follows the AI signal. Thus, downstream behavior is structurally unchanged, although the relevant thresholds shift with the group-specific prior distributions. Aggregating over these priors changes equilibrium AI-use volumes, but the firm's design problem remains qualitatively the same. The AI firm still chooses between a disparate design ($\rho_x>\rho_y$) and an equal-accuracy design ($\rho_x=\rho_y$), and a liability cutoff $\widetilde{\ell}$ exists such that the firm prefers a disparate design for $\ell<\widetilde{\ell}$ and an equal-accuracy design otherwise. Comparative statics with respect to liability and reimbursement therefore continue to follow the baseline logic over the parameter regions we characterize in the Online Appendix: heterogeneous priors change the size of the adoption region rather than the underlying mechanism.

The welfare analysis carries over, with the integration limits adjusted to the new prior supports, though which patient group bears the welfare loss can change. The main qualitative insights continue to hold: at low levels, liability can reduce AI use for type-$y$ patients, whereas at higher levels it can increase use through stronger upstream accuracy incentives; and an equal-accuracy mandate can still induce overuse and reduce welfare. Allowing $\alpha_x\neq\alpha_y$ thus preserves the use-distortion mechanism, and an equal-accuracy mandate can again harm both patient groups over a substantial region of the parameter space.

\section{Concluding Remarks}\label{sec:conclusion}

Policymakers have sought to address uneven algorithm performance across patient groups, yet the primary levers available in practice often operate at deployment: hospitals and physicians remain accountable for discrimination that results from their use of clinical decision-support tools. This paper studies the equilibrium consequences of such deployment-facing accountability by linking downstream physician reliance to upstream AI model design. In a game between an AI firm and a physician, we capture the Section~1557 provision as an asymmetric liability exposure that is triggered when reliance on a tool with unequal performance leads to inappropriate treatment for disadvantaged patients.

Liability intended to protect disadvantaged patients can instead reduce their access to AI. By increasing the expected cost of relying on AI for type-$y$ patients, liability can induce physicians to use AI less for the very group the policy aims to safeguard, even when the same tool would be deployed for type-$x$ patients. The relationship between liability and use is also non-monotone. As liability rises, it can initially suppress reliance on AI for disadvantaged patients, but beyond a threshold it can induce the firm to reallocate investment toward their accuracy, improving performance and expanding adoption. A single instrument can therefore produce overuse at low levels of liability, underuse at intermediate levels, and renewed overuse at high levels, depending on how strongly it feeds back to design incentives.

An equal-accuracy requirement is not an appropriate-use guarantee. Requiring equal accuracy across patient types necessarily reallocates model quality across groups; because improving type-$y$ accuracy is more costly, parity is often achieved largely by reducing type-$x$ accuracy. Equalizing measured performance also relaxes the liability channel that previously discouraged reliance for disadvantaged patients, while reimbursement continues to reward use. The result is a utilization response that can move patients onto (or off) AI precisely where clinical gains are small relative to the fixed cost of use. The mandate can therefore reduce aggregate welfare for both groups over parameter regions we characterize analytically: type-$x$ patients lose from a substantial accuracy decline, whereas type-$y$ patients can be harmed by expanded reliance on a still-imperfect tool.

When deployment incentives are shaped by reimbursement and accountability, standards aimed at algorithm performance need to be paired with instruments that govern use. In our setting, liability standards that encourage firms to invest in disadvantaged-group accuracy need not align physicians' deployment choices with clinical value. Several complementary levers are available: reimbursement rules that attenuate incentives for indiscriminate use, auditing and monitoring requirements aimed at clinically meaningful endpoints, and other accountability mechanisms that discipline overuse.

\label{page:non-healthcare} Although the motivating application is healthcare, the mechanism is broader: whenever professionals remain accountable for decisions informed by algorithms, deployment-facing accountability can generate feedback from use to design and back again. Analogous forces arise when judges rely on risk scores, when lenders and managers use credit and screening systems, and when employers deploy hiring tools. In each case, performance differences across groups can trigger legal or reputational exposure and thereby distort equilibrium reliance. Future work could extend the analysis to richer organizational environments (for example, hospitals with multiple clinicians), alternative contracting and pricing arrangements, and empirical measurement of the predicted non-monotone responses of both adoption and design to liability exposure.

\label{page:futurework}Our analysis also points to several extensions that are especially relevant for AI. First, once an AI system has been trained, its marginal cost of use is typically low, whereas obtaining a second opinion from a human expert can entail a nontrivial cost. Allowing decision makers to choose between AI assistance and human peer review would introduce an additional margin of substitution and raise a distinct set of questions that we leave for future research. %
Second, clinical AI systems can be retrained as new data arrive. If liability discourages use for disadvantaged patients, those patients may become underrepresented in future training data, reducing subsequent accuracy and reinforcing the disparity the policy seeks to mitigate. Third, when physicians can adjust their own diagnostic effort, greater reliance on AI may reduce independent clinical scrutiny and create a new moral-hazard problem.
Future work could also examine institutional arrangements that jointly determine reimbursement, pricing, utilization costs, and liability; allocate responsibility across physicians, hospitals, and developers through shared liability or ex ante subsidies; and link reputational returns to realized performance over time. A full social-welfare analysis is a further direction: treating our health-surplus measure as social welfare would require a common criterion spanning patient surplus, physician utility, firm profit, and deadweight legal and compliance costs.

\newpage

\section*{Funding and Competing Interests}
All authors certify that they have no affiliations with or involvement in any organization or entity
with any financial interest or non-financial interest in the subject matter or materials discussed in
this manuscript. The authors have no funding to report.

\newpage

\section*{Appendix}

\medskip

\setcounter{lemma}{1}
\setcounter{figure}{1}
\setcounter{section}{0}
\setcounter{proposition}{0}
\setcounter{equation}{0}
\setcounter{corollary}{0}

\renewcommand{\thelemma}{A\arabic{lemma}}
\renewcommand{\thefigure}{A\arabic{figure}}
\renewcommand{\thesection}{A\arabic{section}}
\renewcommand{\theproposition}{A\arabic{proposition}}
\renewcommand{\thecorollary}{A\arabic{corollary}}
\renewcommand{\theequation}{A\arabic{equation}}
\renewcommand{\thetable}{A\arabic{table}}

\normalfont

\noindent \textsc{Proof of \cref{lem:AI_decision_imaltruistic_xpatient}}.
Using the physician's expected payoff when not using AI (from \cref{equ:U}) and when using AI (from \cref{equ:Ux}), we compute the difference $U_x-U$, which is given by
{\footnotesize
\begin{align}\label{equ:uti_com_x}
    U_x-U=&
    \begin{cases}
\theta r-c, & \text{if }\alpha\le 1-\rho_x \\
(\rho_x+\alpha-1)\cdot  b-c+\theta r, & \text{if } 1-\rho_x  < \alpha \le 1/2\\
(\rho_x-\alpha)\cdot  b-c+\theta r, & \text{if } 1/2 < \alpha \le \rho_x\\
\theta r-c, & \text{otherwise.}
\end{cases}
\end{align}
}
 
Note that $\theta r-c<0$; therefore, the physician does not use AI for $\alpha\le1-\rho_x$ and for $\alpha>\rho_x$. In the second case of \cref{equ:uti_com_x}, that is, if $1-\rho_x  < \alpha \le 1/2$, the physician does not use AI when  
$\alpha\le \frac{(1-\rho_x)b+c-\theta r}{b}$  and 
uses AI when $\alpha\ge \frac{(1-\rho_x)b+c-\theta r}{b}$. In the third case of \cref{equ:uti_com_x}, that is, if $1/2 < \alpha \le \rho_x$, the physician does not use AI when  
$\alpha>\frac{\rho_x b-(c-\theta r)}{b}$  and 
uses AI when $\alpha\le \frac{\rho_x b-(c-\theta r)}{b}$. Therefore, we have 
\begingroup\footnotesize %
\begin{itemize}
    \item When $c-\theta r<(\rho_x-1/2)  b$, we have $\frac{(1-\rho_x)b+c-\theta r}{b}< 1/2$ and $\frac{\rho_x b-(c-\theta r)}{b}>1/2$. Then,     
    the physician does not use AI if $\alpha\le \frac{(1-\rho_x)b+c-\theta r}{b}$, uses AI if $\frac{(1-\rho_x)b+c-\theta r}{b}<\alpha\le \frac{\rho_x b-(c-\theta r)}{b}$, and does not use AI again if $\alpha> \frac{\rho_x b-(c-\theta r)}{b}$.
    \item When $c-\theta r\ge (\rho_x-1/2) b$, we have $\frac{(1-\rho_x)b+c-\theta r}{b}\ge  1/2$ and $\frac{\rho_x b-(c-\theta r)}{b}\le 1/2$. Then,     
    the physician does not use AI for any patient $\alpha\in [0,1]$.
\end{itemize}
\endgroup

Under the assumption that the physician uses AI for at least some type-$y$ patients (that is, ${d_y^{\mathrm D}}>0$, equivalently, $c -\theta 
 r<(\rho_y-\frac{1}{2})b - (1-\rho_y)\theta  \ell$)
and given  $\rho_x>\rho_y$, we have $c-\theta r<(\rho_x-1/2) b$. Therefore, we derive the physician's AI use for type-$x$ patients in this lemma.

Next, we check whether the physician will follow or reject the AI signal when using AI. In the second and third cases of \cref{equ:uti_com_x}, the physician follows the AI signal whichever treatment it recommends.  These are also the only scenarios in which the physician uses AI. Therefore, the physician follows the AI signal for type-$x$ patients whenever using AI. \hfill \emph{Q.E.D.}
\vspace{.1in}

\noindent \textsc{Proof of \cref{lem:AI_decision_imaltruistic_ypatient}}.
Using the physician's expected payoff when not using AI (from \cref{equ:U}) and when using AI (from \cref{equ:Uy}), we compute the difference $U_y-U$, which is given by
{\footnotesize
\begin{align}\label{equ:uti_com_y}
    U_y-U=&
    \begin{cases}
\theta r-c-\alpha(1-\rho_y)\theta\ell, & \text{if }\alpha\le \frac{(1-\rho_y) (b+\theta \ell)}{b+(1-\rho_y)\theta \ell} \\
(\rho_y+\alpha-1)\cdot  b-c+\theta r-(1-\rho_y)\theta \ell, & \text{if } \frac{(1-\rho_y) (b+\theta \ell)}{b+(1-\rho_y)\theta \ell}< \alpha \le 1/2\\
(\rho_y-\alpha)\cdot  b-c+\theta r-(1-\rho_y)\theta \ell, & \text{if } 1/2 < \alpha \le \frac{b \rho_y}{b+(1-\rho_y)\theta \ell}\\
\theta r-c-(1-\alpha)(1-\rho_y)\theta \ell, & \text{otherwise.}
\end{cases}
\end{align}
}
 
Note that $\theta r-c-\alpha(1-\rho_y)\theta\ell<0$; therefore, the physician does not use AI for $\alpha\le \frac{(1-\rho_y) (b+\theta \ell)}{b+(1-\rho_y)\theta \ell}$. In addition, $\theta r-c-(1-\alpha)(1-\rho_y)\theta \ell<0$, which implies the physician does not use AI for $\alpha>\frac{b \rho_y}{b+(1-\rho_y)\theta \ell}$.

In the second case of \cref{equ:uti_com_y}, that is, if $\frac{(1-\rho_y) (b+\theta \ell)}{b+(1-\rho_y)\theta \ell}< \alpha \le 1/2$, the physician does not use AI when  
$\alpha\le \frac{(1-\rho_y) (b+\theta \ell) +c-\theta r}{b}$  and 
uses AI when $\alpha\ge \frac{(1-\rho_y) (b+\theta \ell) +c-\theta r}{b}$. In the third case of \cref{equ:uti_com_y}, that is, if $1/2 < \alpha \le \frac{b \rho_y}{b+(1-\rho_y)\theta \ell}$, the physician does not use AI when  
$\alpha>\frac{\rho_y b-(1-\rho_y)\theta \ell -c+\theta r}{b}$  and 
uses AI when $\alpha\le \frac{\rho_y b-(1-\rho_y)\theta \ell -c+\theta r}{b}$. Therefore, we have 
\begingroup\footnotesize %
\begin{itemize}
    \item When $c -\theta 
 r<(\rho_y-\frac{1}{2})b - (1-\rho_y)\theta  \ell$, we have $\frac{(1-\rho_y) (b+\theta \ell) +c-\theta r}{b}< 1/2<\frac{\rho_y b-(1-\rho_y)\theta \ell -c+\theta r}{b}$. Then,     
    the physician does not use AI if $\alpha\le \frac{(1-\rho_y) (b+\theta \ell) +c-\theta r}{b}$, uses AI if $\frac{(1-\rho_y) (b+\theta \ell) +c-\theta r}{b}<\alpha<\frac{\rho_y b-(1-\rho_y)\theta \ell -c+\theta r}{b}$, and does not use AI again if $\alpha> \frac{\rho_y b-(1-\rho_y)\theta \ell -c+\theta r}{b}$.
    \item When $c -\theta 
 r>(\rho_y-\frac{1}{2})b - (1-\rho_y)\theta  \ell$, we have $\frac{(1-\rho_y) (b+\theta \ell) +c-\theta r}{b}>1/2$ and $\frac{\rho_y b-(1-\rho_y)\theta \ell -c+\theta r}{b} <1/2$. Then,     
    the physician does not use AI for any patient $\alpha\in [0,1]$.
\end{itemize}
\endgroup

Under the assumption that the physician uses AI for at least some type-$y$ patients (that is, ${d_y^{\mathrm D}}>0$, equivalently, $c -\theta 
 r<(\rho_y-\frac{1}{2})b - (1-\rho_y)\theta  \ell$), we derive the physician's AI use for type-$y$ patients in this lemma.

Next, we check whether the physician will follow or reject the AI signal when using AI. In the second and third cases of \cref{equ:uti_com_y}, the physician follows the AI signal whichever treatment it recommends.  These are also the only scenarios in which the physician might choose to use AI. Therefore, the physician follows the AI signal for type-$y$ patients whenever using AI. \hfill \emph{Q.E.D.}
\vspace{.1in}

\noindent \textsc{Proof of \cref{prop:developer_rho*}}.
(a) Suppose the AI firm supplies a disparate algorithm. It sets $\rho_x$ and $\rho_y$ to maximize the profit in \cref{equ:developer}, where $d_x^{\mathrm D}$ and $d_y^{\mathrm D}$ are given in \cref{equ:AI demand}. The objective is strictly concave in $\rho_x$ and $\rho_y$, so the first-order conditions are sufficient. Setting $\partial \pi^{\mathrm D}/\partial \rho_x=0$ and $\partial \pi^{\mathrm D}/\partial \rho_y=0$ yields
$\rho_x^*=\frac{1}{2}+\frac{f}{\kappa_x}$ and $\rho_y^*=\frac{1}{2}+\frac{\beta f(b+\theta \ell)}{b\kappa_y}$. The interior conditions $\rho_x^*<1$ and $\rho_y^*<1$ are equivalent to $f<\kappa_x/2$ and $f<\frac{b\kappa_y}{2\beta(b+\theta \ell)}$, respectively.

Next, to ensure that the firm earns a positive profit from serving type-$y$ patients at $\rho_y^*$, that is, $\pi_y(\rho_y^*)>0$, it suffices that
$f>\frac{b\kappa_y(2c-2\theta r+\theta \ell)}{\beta(b+\theta \ell)^2}$. This condition also guarantees that the physician uses AI for at least some type-$y$ patients under the induced accuracy choice (equivalently, $d_y^{\mathrm D}|_{\rho_y=\rho_y^*}>0$).  Otherwise, the optimal choice sets $\rho_y$ arbitrarily close to $\frac{1}{2}$.

We first note that no strictly reverse-disparate design ($\rho_y>\rho_x$) can be optimal. In that region, the disparity-triggered liability channel is inactive, so the objective is the no-liability profit, which is strictly concave with unconstrained maximizers $\rho_x^{0}=\tfrac12+\tfrac{f}{\kappa_x}$ and $\rho_y^{0}=\tfrac12+\tfrac{\beta f}{\kappa_y}$. Because $\kappa_y>\kappa_x$ and $\beta\le1$, we have $\tfrac{f}{\kappa_x}>\tfrac{\beta f}{\kappa_y}$, so $\rho_x^{0}>\rho_y^{0}$ and the unconstrained maximizer lies outside this region. By strict concavity, the constrained maximizer lies on the boundary $\rho_x=\rho_y$, which is the equal-accuracy design. It therefore suffices to compare the disparate and equal-accuracy designs. We then show that whenever a disparate design is optimal, it must satisfy $\rho_x^*>\rho_y^*$. If $\rho_x^*=\rho_y^*$ under the disparate-design best response, then the firm can instead supply the same accuracy level as an equal-accuracy design, which eliminates the liability channel and weakly increases physician demand; hence $\pi^{\mathrm D}(\rho_x^*,\rho_y^*)<\pi^{\mathrm E}(\rho_x^*)\le \max_{\rho}\pi^{\mathrm E}(\rho)=\pi^{\mathrm E}(\rho^*)$. This contradicts the optimality of a disparate design.   To preclude the reversal $\rho_x^*<\rho_y^*$, note that the gap $\rho_x^*-\rho_y^*=f\big(1/\kappa_x-\beta(b+\theta\ell)/(b\kappa_y)\big)$ is positive at $\ell=0$, strictly decreasing in $\ell$, and zero at $\ell_{eq}:=b(\kappa_y-\beta\kappa_x)/(\beta\theta\kappa_x)$. At $\ell_{eq}$, the disparate design is strictly dominated, because $\pi^{\mathrm D}(\rho,\rho)-\pi^{\mathrm E}(\rho)=2\beta\theta f\ell(\rho-1)/b<0$. Because $\pi^{\mathrm D}(\rho_x^*,\rho_y^*)-\pi^{\mathrm E}(\rho^*)$ is strictly decreasing in $\ell$ (shown below), the switch occurs before $\ell_{eq}$, so $\rho_x^*>\rho_y^*$ wherever the disparate design is optimal.

(b) Now suppose the firm supplies an equal-accuracy algorithm, so $\rho_x=\rho_y\equiv \rho$. Maximizing \cref{equ:optimal_rho_fair} and setting $\partial \pi^{\mathrm E}(\rho)/\partial \rho=0$ yields
$\rho^*=\frac{1}{2}+\frac{(1+\beta)f}{\kappa_x+\kappa_y}$. The interior condition $\rho^*<1$ is equivalent to $f<\frac{\kappa_x+\kappa_y}{2(1+\beta)}$.  
Note that $\frac{\kappa_x+\kappa_y}{2(1+\beta)}>\frac{\kappa_x}{2}$ because $\kappa_y>\beta\kappa_x$; hence the standing condition $f<\kappa_x/2$ (which ensures $\rho_x^*<1$) already implies $f<\frac{\kappa_x+\kappa_y}{2(1+\beta)}$, and the interior constraint for the equal-accuracy design can be omitted.

Next, we show that there exists a cutoff $\widetilde{\ell}$  such that a disparate design is optimal for $\ell<\widetilde{\ell}$ and an equal-accuracy design is optimal for $\ell\ge \widetilde{\ell}$. As $\ell\to 0$, the disparate-design profit converges to the equal-accuracy objective evaluated at potentially different accuracies across types, that is, $\pi^{\mathrm D}(\rho_x,\rho_y)|_{\ell\to 0}\to \pi^{\mathrm E}(\rho_x,\rho_y)$, whereas the equal-accuracy design restricts the firm to $\rho_x=\rho_y$. Because the feasible set under a disparate design weakly contains that under equal accuracy, we have $\pi^{\mathrm D}(\rho_x^*,\rho_y^*)|_{\ell\to 0}\ge \pi^{\mathrm E}(\rho^*)$, with strict inequality whenever the equality constraint $\rho_x=\rho_y$ binds.

Next, under a disparate design, the firm's optimal profit is decreasing in $\ell$. The type-$x$ component is independent of $\ell$, so it suffices to study the type-$y$ component. By the envelope theorem,
\[
\frac{d}{d\ell}\pi^{\mathrm D}(\rho_x^*,\rho_y^*)=\left.\frac{\partial \pi^{\mathrm D}(\rho_x,\rho_y)}{\partial \ell}\right|_{\rho_x=\rho_x^*,\,\rho_y=\rho_y^*}
=-\frac{2\theta f\beta}{b}\,(1-\rho_y^*)<0,
\]
where the negative sign strictly holds in the interior region in which $\rho_y^* < 1$. By contrast, $\pi^{\mathrm E}(\rho^*)$ does not depend on $\ell$. Because $\pi^{\mathrm D}(\rho_x^*,\rho_y^*)-\pi^{\mathrm E}(\rho^*)$ is strictly decreasing in $\ell$ under this interior condition, there is at most one cutoff $\widetilde{\ell}$ at which $\pi^{\mathrm D}(\rho_x^*,\rho_y^*)=\pi^{\mathrm E}(\rho^*)$, 
and the profit ranking switches at that point. 
Outside this interior region (where $\rho_y^* \geq 1$), although $\pi^{\mathrm D}(\rho_x^*,\rho_y^*)-\pi^{\mathrm E}(\rho^*)$ becomes a convex parabola in $\ell$ that could theoretically admit a second crossing, such a region is structurally infeasible and can be disregarded, thereby establishing the unique design cutoff.

We conclude the proof by establishing the existence of the cutoff $\widetilde{\ell}$ within the feasible domain. A sufficient condition for $\widetilde{\ell}$ to be interior is $\pi^{\mathrm E}(\rho^*)> \pi_x(\rho^*_x)$, which simplifies to $\beta(\beta+2)\kappa_x > \kappa_y$ and 
$f > \frac{2\beta(c-\theta r)\kappa_x(\kappa_x+\kappa_y)}{b(\beta(\beta+2)\kappa_x-\kappa_y)}$.
These parameter restrictions ensure that the profit from type-$y$ patients remains strictly positive ($\pi_y(\rho^*_y) > 0$) at the point of indifference where $\pi^{\mathrm D}(\rho_x^*,\rho_y^*)=\pi^{\mathrm E}(\rho^*)$ (i.e., at $\ell = \widetilde{\ell}$). Consequently, both $\pi_y(\rho^*_y) > 0$ and $\rho^*_y < 1$ hold for all $\ell < \widetilde{\ell}$, where the latter inequality follows from our earlier result that $\rho^*_y<\rho^*_x <1 $ whenever a disparate design strictly dominates. Conversely, if this condition is violated, an interior cutoff $\widetilde{\ell}$ fails to exist. In that scenario, an equal-accuracy design never becomes optimal as $\ell$ increases; instead, the firm continuously deploys a disparate algorithm while completely abandoning the type-$y$ market to serve type-$x$ patients exclusively. \hfill \emph{Q.E.D.}

\vspace{.1in}

\noindent \textsc{Proof of \cref{prop:disparity}}.
Under the disparate design, AI use for type-$y$ patients is given by \cref{equ:AI demand} as $d_y^{\mathrm D}=\beta\cdot \frac{(2\rho_y^*-1)b-2(1-\rho_y^*)\theta \ell-2c+2\theta r}{b}$, where $\rho_y^*=\frac{1}{2}+\frac{\beta f(b+\theta\ell)}{b\kappa_y}$ from \cref{prop:developer_rho*}. Note $d_y^{\mathrm D}$ is convex in $\ell$, implying that it has at most one turning point. Taking the first-order derivative of $d_y^{\mathrm D}$ with respect to $\ell$ yields $\frac{\partial d_y^{\mathrm D}}{\partial \ell}>0$ if and only if $\ell>\bar\ell:=\frac{b(\kappa_y-4\beta f)}{4\theta\beta f}$. We require $\bar \ell>0$ to ensure that $d_y^{\mathrm D}$ has a unique turning point at $\ell=\bar \ell$, which gives us  $f<\frac{\kappa_y}{4\beta}$. 
Thus, holding the AI firm's design fixed at the disparate regime, liability raises type-$y$ AI use when $\ell$ exceeds $\bar\ell$ and lowers it when $\ell<\bar\ell$.

To establish the claimed pattern around $\bar\ell$ in equilibrium, note that the parameter restrictions in \cref{prop:developer_rho*} imply: (i) $\frac{b \kappa_y (2c-2\theta r+\theta \ell)}{\beta (b+\theta \ell)^2}<f<\min\{\frac{\kappa_x}{2}, \frac{b\kappa_y}{2\beta(b+\theta\ell)}\}$, (ii) the AI firm strictly prefers the disparate design at $\ell=\bar\ell$ (that is, $\pi^{\mathrm D}(\rho_x^*,\rho_y^*)>\pi^{\mathrm E}(\rho^*)$ at $\ell=\bar\ell$),  
  and  (iii) $\beta(\beta+2)\kappa_x > \kappa_y$ and  $f>\frac{2\beta(c-\theta r)\kappa_x(\kappa_x+\kappa_y)}{b(\beta(\beta+2)\kappa_x-\kappa_y)}$.

We first examine restriction (i). We can verify that $f<\frac{b\kappa_y}{2\beta(b+\theta\ell)}|_{\ell=\bar\ell}$, which means $\rho_y^*<1$ at $\ell=\bar\ell$ (by \cref{prop:developer_rho*}), which also means that $\rho_y^*<1$ must hold for  $\ell<\bar\ell$ (because $\rho_y^*$ is increasing in $\ell$).  
Furthermore, 
given that $f > \frac{b\kappa_y(2c-2\theta r+\theta\ell)}{\beta(b+\theta\ell)^2}$ is equivalent to $\pi_y(\rho_y^*) > 0$, we can omit the consideration of  $f > \frac{b\kappa_y(2c-2\theta r+\theta\ell)}{\beta(b+\theta\ell)^2}|_{\ell=\bar\ell}$. The reason is that  $\pi_y(\rho_y^*)|_{\ell=\bar\ell}=[\pi^{\mathrm D}(\rho_x^*,\rho_y^*)-\pi^{\mathrm E}(\rho^*)]|_{\ell=\bar\ell}+(\pi^{\mathrm E}(\rho^*)-\pi_x(\rho_x^*))>0$, where the inequality holds because the first term is greater than 0 (we will guarantee this by discussing the restriction (ii) in what follows) and  the second term is also greater than 0 by the restriction (iii) above. 
Moreover, $\frac{\partial \pi_y(\rho_y^*)}{\partial \ell}=\frac{2\beta f\theta}{b}\left(\rho_y^*-1\right)<0$ whenever $\rho_y^*<1$, which holds on $[0,\bar\ell]$ as just established. Hence $\pi_y(\rho_y^*)|_{\ell}>\pi_y(\rho_y^*)|_{\ell=\bar\ell}>0$ for every $\ell<\bar\ell$, so the participation constraint is slack throughout. Therefore, for the restriction (i), the only constraint is $f<\frac{\kappa_x}{2}$.

We then examine restriction (ii). We obtain $[\pi^{\mathrm D}(\rho_x^*,\rho_y^*)-\pi^{\mathrm E}(\rho^*)]|_{\ell=\bar\ell}=A \cdot f^2 +\beta  f-\frac{3 \kappa_y}{16}$, where $A=\frac{\kappa_y-\beta (\beta+2) \kappa_x}{\kappa_x(\kappa_x+\kappa_y)}$. 
Letting $[\pi^{\mathrm D}(\rho_x^*,\rho_y^*)-\pi^{\mathrm E}(\rho^*)]|_{\ell=\bar\ell}=0$ yields $f=f_a$ or $f=f_b$,  where $f_a:=\frac{2 \beta  \kappa_x (\kappa_x+\kappa_y)}{4 \beta  (\beta +2) \kappa_x-4 \kappa_y}-\frac{\sqrt{\kappa_x (\kappa_x+\kappa_y) \left(\beta ^2 \kappa_x (4 \kappa_x+\kappa_y)-6 \beta  \kappa_x
   \kappa_y+3 \kappa_y^2\right)}}{4 \beta  (\beta +2) \kappa_x-4 \kappa_y}$ and $f_b:=\frac{2 \beta  \kappa_x (\kappa_x+\kappa_y)}{4 \beta  (\beta +2) \kappa_x-4 \kappa_y}+\frac{\sqrt{\kappa_x (\kappa_x+\kappa_y) \left(\beta ^2 \kappa_x (4 \kappa_x+\kappa_y)-6 \beta  \kappa_x
   \kappa_y+3 \kappa_y^2\right)}}{4 \beta  (\beta +2)\kappa_x -4 \kappa_y}$. 
   According to the restriction (iii), we have $A<0$, and we can further obtain that $f_a<f_b$.  Then $[\pi^{\mathrm D}(\rho_x^*,\rho_y^*)-\pi^{\mathrm E}(\rho^*)]|_{\ell=\bar\ell}>0$ is equivalent to $f_a<f<f_b$. We can verify that $f_b>\min\{\frac{\kappa_x}{2}, \frac{\kappa_y}{4\beta}\}$. Thus, to ensure $[\pi^{\mathrm D}(\rho_x^*,\rho_y^*)-\pi^{\mathrm E}(\rho^*)]|_{\ell=\bar\ell}>0$, we only need to impose an additional condition $f>f_a$. 
   Therefore, under the aforementioned conditions,  we can obtain $\pi^{\mathrm D}(\rho_x^*,\rho_y^*)|_{\ell<\bar \ell}>\pi^{\mathrm D}(\rho_x^*,\rho_y^*)|_{\ell=\bar \ell}>\pi^{\mathrm E}(\rho^*)$, which means that the AI firm strictly prefers the disparate design if $\ell<\bar \ell$.

Define $f_l:= \max \left\{ 
 \frac{2 \beta (c -  \theta r) \kappa_x (\kappa_x+\kappa_y)} {b (\beta (\beta+2) \kappa_x-\kappa_y)}, \;
f_a\right\}$. So far, we have proved that when $\beta(\beta+2)\kappa_x > \kappa_y$ and  $f_l<f< \min\{\frac{\kappa_x}{2},   \frac{\kappa_y}{4\beta}\}$, AI use for disadvantaged patients is decreasing in $\ell$ for $\ell\le \bar \ell$.  We next prove that AI use for disadvantaged patients is (weakly) increasing in $\ell$ for $\ell>\bar \ell$.

From the preceding analysis, when $\ell$ increases marginally beyond $\bar{\ell}$, the interior conditions $\pi_y(\rho_y^*) > 0$, $\rho_y^* < 1$, and $\pi^{\mathrm D}(\rho_x^*,\rho_y^*)>\pi^{\mathrm E}(\rho^*)$  continue to hold. By continuity, there exists a neighborhood to the right of $\bar{\ell}$ in which AI use among type-$y$ patients is strictly increasing in $\ell$. As $\ell$ increases further, one of these performance or participation constraints must eventually bind. 
From the proof of \cref{prop:developer_rho*}, $\pi^{\mathrm D}(\rho_x^*,\rho_y^*)>\pi^{\mathrm E}(\rho^*)$ implies $\rho_x^*>\rho_y^*$. Given the interior requirement $\rho_x^* < 1$, it follows immediately that $\rho_y^* < 1$ is satisfied whenever the disparate design remains strictly optimal. Hence, either the regime-switching condition 
$\pi^{\mathrm D}(\rho_x^*,\rho_y^*)>\pi^{\mathrm E}(\rho^*)$ or  the market-viability condition  
$\pi_y(\rho_y^*)>0$  must be the first to bind. Because parameter restriction (iii) ensures that $\pi^{\mathrm E}(\rho^*)>\pi_x(\rho_x^*)$, the profit under the equal-accuracy algorithm strictly dominates the profit from serving type-$x$ patients alone.  
Consequently, $\pi^{\mathrm D}(\rho_x^*,\rho_y^*)>\pi^{\mathrm E}(\rho^*)$ must be the first condition to fail. At this crossing point, the equal-accuracy algorithm becomes optimal, and AI use among type-$y$ patients remains constant thereafter. Because AI use may exhibit a discrete upward jump at this regime switch (by the same argument as in the proof of part~(a) of \cref{prop:mandate_fairness_worse}), AI use among type-$y$ patients is weakly increasing in $\ell$ for all $\ell > \bar{\ell}$. %
The proposition then follows\footnote{A numerical instance illustrates the existence of this region. When $b=2$, $\theta=1$, $c=0.60$, $r=0.5$, $\kappa_x=1.3$, $\kappa_y=1.705$, $f=0.6$, and $\beta=0.68$, the physician uses AI for fewer type-$y$ patients if $0<\ell<0.0895$ and for more type-$y$ patients if $0.0895<\ell<0.3747$. The equal-accuracy design becomes optimal if $\ell\ge 0.3747$.}. \hfill \emph{Q.E.D.}

\vspace{.1in}

\noindent \textsc{Proof of \cref{lem:AI_decision_altruistic}}.
Using the altruistic physician's expected payoff when not using AI (from \cref{equ:U}) and when using AI (from \cref{eq:Ual}), we compute the difference $U^{\text{al}}-U$, given by
{\footnotesize
\begin{align}\label{equ:al_uti_com}
    U^{\text{al}}-U=&
    \begin{cases}
-c, & \text{if }\alpha\le 1-\rho_t \\
(\rho_t+\alpha-1)\cdot  b-c, & \text{if } 1-\rho_t  < \alpha \le 1/2\\
(\rho_t-\alpha)\cdot  b-c, & \text{if } 1/2 < \alpha \le \rho_t\\
-c, & \text{otherwise.}
\end{cases}
\end{align}
}

The altruistic physician does not use AI in the first and fourth cases, that is, when $\alpha\le 1-\rho_t$ or $\alpha>\rho_t$. In the second case of \cref{equ:al_uti_com}, the physician does not use AI when $\alpha\le \frac{(1-\rho_t)b+c}{b}$ and uses AI otherwise. In the third case of \cref{equ:al_uti_com}, the physician does not use AI when $\alpha>\frac{\rho_t b-c}{b}$ and uses AI otherwise. Therefore, we have
\begin{itemize}
    \item When $c<(\rho_t-1/2) \cdot b$, we have $\frac{(1-\rho_t)b+c}{b}< 1/2$ and $\frac{\rho_t b-c}{b}>1/2$. Then,     
    the physician does not use AI if $\alpha\le \frac{(1-\rho_t)b+c}{b}$, uses AI if $\frac{(1-\rho_t)b+c}{b}<\alpha\le \frac{\rho_t b-c}{b}$, and does not use AI again if $\alpha> \frac{\rho_t b-c}{b}$.
    \item When $c\ge (\rho_t-1/2) \cdot b$, we have $\frac{(1-\rho_t)b+c}{b}\ge  1/2$ and $\frac{\rho_t b-c}{b}\le 1/2$. Then,     
    the physician does not use AI for any patient $\alpha\in (0,1)$.
\end{itemize}

Next, we check whether the physician will follow or reject the AI signal when using AI. In the second and third cases of \cref{equ:al_uti_com}, the physician follows the AI signal whichever treatment it recommends.  These are also the only scenarios in which the physician might choose to use AI. Therefore, the physician follows the AI signal for type-$t$ patients whenever using AI. \hfill \emph{Q.E.D.}
\vspace{.1in}

\noindent \textsc{Proof of \cref{prop:fairness_appropriateUse}}.
{(a)} Under an equal-accuracy design, the disparity-triggered liability channel does not apply. An impurely altruistic physician therefore bases AI use for both patient types on the decision rule in \cref{lem:AI_decision_imaltruistic_xpatient}. Comparing this rule with the corresponding benchmark for an altruistic physician in \cref{lem:AI_decision_altruistic} implies that the impurely altruistic physician uses AI for a (weakly) larger set of patients, and hence overuses AI.

{(b)} Under a disparate design, comparing \cref{lem:AI_decision_imaltruistic_ypatient} with \cref{lem:AI_decision_altruistic} shows that the impurely altruistic physician uses AI appropriately for type-$y$ patients if and only if the reimbursement exactly offsets the expected liability cost on the margin, that is, $(1-\rho_y^*)\ell=r$, which is equivalent to $\frac{1}{2}-\frac{\beta f(b+\theta\ell)}{b\kappa_y}=\frac{r}{\ell}$. By part (a), the physician overuses AI for type-$x$ patients.\footnote{A numerical example is given by  $b=1.8$, $\theta=1$, $r=0.1$, $\kappa_x=2$, $\kappa_y=3$, $f=0.8$, $\beta=0.6$, $c=0.11$, and $\ell=\ell_s\approx 0.321$, under which AI is used appropriately for type-$y$ patients.}

{(c)} From the argument in part (b), under the disparate design, the physician overuses AI for type-$y$ patients when $\Big(\frac{1}{2}-\frac{\beta f(b+\theta\ell)}{b\kappa_y}\Big)\ell<r$ and underuses AI when $\Big(\frac{1}{2}-\frac{\beta f(b+\theta\ell)}{b\kappa_y}\Big)\ell>r$. The left-hand side is concave in $\ell$, so the equality $\Big(\frac{1}{2}-\frac{\beta f(b+\theta\ell)}{b\kappa_y}\Big)\ell=r$ has (at most) two solutions; let $\ell_s$ denote the smaller root and $\ell_h$ denote the larger root,
\begin{align*}
    \ell_s=\frac{b\kappa_y\left(1-\sqrt{\frac{b(\kappa_y-2\beta f)^2-16\beta f\kappa_y \theta r}{b\kappa_y^2}}\right)}{4\beta f\theta}-\frac{b}{2\theta}
    ~\text{and } ~
     \ell_h=\frac{b\kappa_y\left(1+\sqrt{\frac{b(\kappa_y-2\beta f)^2-16\beta f\kappa_y \theta r}{b\kappa_y^2}}\right)}{4\beta f\theta}-\frac{b}{2\theta}.
\end{align*}
In what follows, we prove that the physician first overuses AI when $\ell<\ell_s$ , then underuses AI when $\ell_s<\ell<\ell_m$, and overuses AI again when $\ell>\ell_m$, where $\ell_m$ is a threshold.

We start by establishing the claimed pattern around $\ell_s$ in equilibrium.  Requiring $\ell_s>0$ yields $f < \frac{ b \kappa_y + 4 \kappa_y \theta r}{2 b \beta} - 
  \frac{\kappa_y \sqrt{2b   \theta r+ 4  \theta^2 r^2}}{b \beta}$.  In addition, note that the parameter restrictions in \cref{prop:developer_rho*} imply: (i) $\frac{b \kappa_y (2c-2\theta r+\theta \ell)}{\beta (b+\theta \ell)^2}<f<\min\{\frac{\kappa_x}{2}, \frac{b\kappa_y}{2\beta(b+\theta\ell)}\}$,  (ii) the AI firm strictly prefers the disparate design at $\ell=\ell_s$ (that is, $\pi^{\mathrm D}(\rho_x^*,\rho_y^*)>\pi^{\mathrm E}(\rho^*)$ at $\ell=\ell_s$), and  (iii) $\beta(\beta+2)\kappa_x > \kappa_y$ and  $f>\frac{2\beta(c-\theta r)\kappa_x(\kappa_x+\kappa_y)}{b(\beta(\beta+2)\kappa_x-\kappa_y)}$.

We first examine restriction (i). We can verify that $f<\frac{b\kappa_y}{2\beta(b+\theta\ell)}|_{\ell=\ell_s}$, which is equivalent to  $\rho_y^*|_{\ell=\ell_s}=1-r/\ell_s<1$; this also means that $\rho_y^*<1$ must hold for  $\ell<\ell_s$ (because $\rho_y^*$ is increasing in $\ell$). In addition, when $\ell=\ell_s$, considering $f>\frac{b \kappa_y (2c-2\theta r+\theta \ell)}{\beta (b+\theta \ell)^2}$ or equivalently $\pi_y(\rho_y^*)>0$ (by \cref{prop:developer_rho*}) yields 
\begin{align}\label{eq:pos_profit_y_ls}
\frac{4 \beta^2 f^2}{\kappa_y} + 4 \beta f + (\kappa_y - 2 \beta f) \sqrt{\frac{b(\kappa_y-2\beta f)^2-16\beta f\kappa_y \theta r}{b\kappa_y^2}}>\kappa_y + \frac{8 \beta f (2 c - \theta r)}{b};
\end{align}
therefore, under the aforementioned condition, we can derive $\pi_y(\rho_y^*)|_{\ell<\ell_s}>\pi_y(\rho_y^*)|_{\ell=\ell_s}>0$, which means that the firm earns a positive profit from serving type-$y$ patients at $\rho_y^*$ if $\ell<\ell_s$. 

We then examine restriction (ii). Considering  $[\pi^{\mathrm D}(\rho_x^*,\rho_y^*)-\pi^{\mathrm E}(\rho^*)]|_{\ell=\ell_s}$ yields 
{\small\begin{align}\label{eq:optimal_bias_ls}
\frac{8 f^2}{\kappa_x} + \frac{4 \beta f (\beta f+\kappa_y) }{\kappa_y} + (\kappa_y - 2 \beta f) \sqrt{\frac{b(\kappa_y-2\beta f)^2-16\beta f\kappa_y \theta r}{b\kappa_y^2}}>\kappa_y + \frac{8 \beta f \theta r}{b} + \frac{8 (1+\beta)^2 f^2 }{\kappa_x+\kappa_y}.
\end{align}}
Therefore, under the aforementioned condition,  we can obtain
\[
\pi^{\mathrm D}(\rho_x^*,\rho_y^*)|_{\ell<\ell_s}>\pi^{\mathrm D}(\rho_x^*,\rho_y^*)|_{\ell=\ell_s}>\pi^{\mathrm E}(\rho^*),
\]
which means that the AI firm strictly prefers the disparate design if $\ell<\ell_s$.

So far, we have the following conditions: \eqref{eq:pos_profit_y_ls}, \eqref{eq:optimal_bias_ls},  and  $f < \min\{\frac{\kappa_x}{2},  \frac{ b \kappa_y + 4 \kappa_y \theta r}{2 b \beta} - 
  \frac{\kappa_y \sqrt{2b   \theta r+ 4  \theta^2 r^2}}{b \beta}\}$, under which the physician overuses AI when $\ell<\ell_s$.   We next prove that the physician underuses AI when $\ell_s<\ell<\ell_m$ and overuses AI again when $\ell>\ell_m$.

From the above proof, when $\ell$ increases marginally from $\ell_s$, the conditions $\pi_y(\rho_y^*)>0$, $\rho_y^*<1$, and $\pi^{\mathrm D}(\rho_x^*,\rho_y^*)>\pi^{\mathrm E}(\rho^*)$ continue to hold. By continuity, there therefore exists a neighborhood to the right of $\ell_s$ in which the physician underuses AI as $\ell$ increases from $\ell_s$. As $\ell$ increases further, one of the preceding conditions must eventually fail.  By an argument similar to that in the proof of \cref{prop:disparity}, the restriction (iii) (i.e., $\beta (\beta+2) \kappa_x > \kappa_y$ and $f >\frac{2 \beta (c -  \theta r) \kappa_x (\kappa_x+\kappa_y)} {b (\beta (\beta+2) \kappa_x-\kappa_y)}$) ensures that   $\pi^{\mathrm D}(\rho_x^*,\rho_y^*)>\pi^{\mathrm E}(\rho^*)$ is the first and only condition that can be violated as  $\ell$ increases. Once this condition is violated, part (a) implies that physician overuse of AI reappears. Depending on the relationship between the violation point and $\ell_h$, two cases arise:
\begin{itemize}
    \item  Suppose $\pi^{\mathrm D}(\rho_x^*,\rho_y^*)>\pi^{\mathrm E}(\rho^*)$  is violated before $\ell=\ell_h$. Then $\ell_m$ coincides with the violation point. In this case, the physician overuses AI for $\ell<\ell_s$, underuses AI for $\ell_s<\ell<\ell_m$, and overuses AI again for $\ell>\ell_m$. The first two regions arise under the disparate algorithm, whereas the final overuse region is induced by the equal-accuracy algorithm.
    \item  Suppose $\pi^{\mathrm D}(\rho_x^*,\rho_y^*)>\pi^{\mathrm E}(\rho^*)$  is violated after $\ell=\ell_h$. Then $\ell_m=\ell_h$. In this case, the physician overuses AI for $\ell<\ell_s$, underuses AI for $\ell_s<\ell<\ell_m$, and overuses AI again for $\ell>\ell_m$. The re-emergence of overuse initially occurs under the disparate algorithm and subsequently persists under the equal-accuracy algorithm after the algorithmic switch.
\end{itemize}

Combining the above two cases establishes the proposition under the conditions $\beta (\beta+2) \kappa_x > \kappa_y$,  \eqref{eq:pos_profit_y_ls}, \eqref{eq:optimal_bias_ls}, and  $\frac{2 \beta (c -  \theta r) \kappa_x (\kappa_x+\kappa_y)} {b (\beta (\beta+2) \kappa_x-\kappa_y)} <f < \min\{\frac{\kappa_x}{2},   \frac{ b \kappa_y + 4 \kappa_y \theta r}{2 b \beta} - 
  \frac{\kappa_y \sqrt{2b   \theta r+ 4  \theta^2 r^2}}{b \beta}\}$. This completes the proof\footnote{For example, when $b=1.8$, $\theta=1$, $r=0.1$, $\kappa_x=2$, $\kappa_y=3$, $f=0.8$, $\beta=0.6$, and $c=0.11$, the physician overuses AI when $\ell<0.321$, underuses AI when $0.321<\ell<0.402$, and overuses AI again when $\ell\ge 0.402$, the level at which the equal-accuracy design becomes optimal.}. \hfill \emph{Q.E.D.}
     
\vspace{.1in}

\noindent \textsc{Proof of \cref{prop:mandate_fairness_worse}}.
Using the expected utilities for different patient types in \cref{table:payoff_AI} and the AI accuracy expressions in \cref{prop:developer_rho*}, we compute the expected patient welfare. By \cref{lem:AI_decision_imaltruistic_xpatient,lem:AI_decision_imaltruistic_ypatient}, when the physician relies on an accuracy $\rho$ without the disparity-liability channel, she uses AI exactly on $(\alpha^{-}(\rho),\alpha^{+}(\rho))$, where
\[
\alpha^{-}(\rho):=\frac{(1-\rho) b +c-\theta r}{b},\qquad \alpha^{+}(\rho):=\frac{\rho b -c+\theta r}{b};
\]
under a disparate design, she uses AI for type-$y$ patients exactly on $(\alpha_y^{-},\alpha_y^{+})$, where
\[
\alpha_y^{-}:=\frac{(1-\rho_y^*) (b+\theta \ell) +c-\theta r}{b},\qquad \alpha_y^{+}:=\frac{\rho_y^* b-(1-\rho_y^*)\theta \ell -c+\theta r}{b}.
\]
Accordingly,
\begin{align}\label{equ:wel}
W_x(\rho)&=\int_{0}^{\alpha^{-}(\rho)} \Big((1-\alpha)b\Big) d\alpha
+\int_{\alpha^{-}(\rho)}^{\alpha^{+}(\rho)} \Big(\rho b-c\Big) d\alpha
+\int_{\alpha^{+}(\rho)}^1 \Big(\alpha b\Big) d\alpha, \notag\\
W_y(\rho_y^*)&=\beta\cdot \Big[\int_{0}^{\alpha_y^{-}} \Big((1-\alpha)b\Big) d\alpha
+\int_{\alpha_y^{-}}^{\alpha_y^{+}} \Big(\rho_y^*b-c\Big) d\alpha
+\int_{\alpha_y^{+}}^1 \Big(\alpha b\Big) d\alpha\Big].
\end{align}
The four objects we compare are $W_x(\rho_x^*)$, $W_y(\rho_y^*)$, $W_x(\rho^m)$, and, because the mandate deactivates the disparity-liability channel for type-$y$ patients, $W_y(\rho^m)=\beta\, W_x(\rho^m)$,
where $\rho_x^*=\frac{1}{2} +  \frac{f}{\kappa_x}$, $\rho_y^*=\frac{1}{2} + \frac{\beta f (b + \theta \ell)}{b \kappa_y}$, and $\rho^m=\frac{1}{2} + \frac{(1 + \beta) f}{\kappa_x+\kappa_y}$.  
In each of the above welfare expressions, the first and third terms capture the expected welfare of patients for whom the physician does not use AI, whereas the second term represents the expected welfare of patients for whom the physician uses AI.

(a) Comparing the equilibrium accuracies gives \(\rho^m \leq \rho_x^*\), so it suffices to prove \(\rho_y^* \leq \rho^m\).  

Define \(\Delta\) as the difference between the AI firm's optimal expected profits under a disparate algorithm and under an equal-accuracy algorithm: $\Delta:=\pi^{\mathrm D}(\rho_x^*,\rho_y^*)-\pi^{\mathrm E}(\rho^m)$, with the profit functions in \cref{equ:developer,equ:optimal_rho_fair} evaluated at the equilibrium accuracies above.
Let \(\widehat{\ell}\) be the value satisfying \(\rho_y^* = \rho^m\); we can derive $\widehat{\ell} = \frac{b (\kappa_y-\beta \kappa_x)}{\beta (\kappa_x + \kappa_y) \theta}$.  

Because \(\rho_x^*\) and \(\rho^m\) are independent of \(\ell\), and \(\rho_y^*\) is increasing in \(\ell\), we  only need to prove  \(\widehat{\Delta} := \Delta|_{\ell = \widehat{\ell}} < 0\). Substituting \(\ell = \widehat{\ell}\) and differentiating,
\begin{align*}
\widehat{\Delta} &= \frac{f (\kappa_y-\beta \kappa_x ) (f ((2 + \beta) \kappa_x + \kappa_y)-\kappa_x (\kappa_x + \kappa_y) )}{ \kappa_x (\kappa_x + \kappa_y)^2},\\
\frac{\partial \widehat{\Delta}}{\partial f} &= -\frac{ (\kappa_y-\beta \kappa_x) (\kappa_x (\kappa_x + \kappa_y) -2f ((2 + \beta) \kappa_x + \kappa_y))}{ \kappa_x (\kappa_x + \kappa_y)^2}.
\end{align*}
Because \(\frac{\partial \widehat{\Delta}}{\partial f}\) is increasing in \(f\), $\widehat{\Delta}$ is convex in $f$. Because $\widehat{\Delta}|_{f=0}=0$ and direct substitution yields $\widehat{\Delta}|_{f=\kappa_x/2}<0$, convexity implies $\widehat{\Delta}<0$ on the entire admissible range $0<f<\frac{\kappa_x}{2}$. Finally, because $\Delta$ is decreasing in $\ell$ and $\widehat{\Delta}<0$, the firm abandons the disparate design strictly before $\ell$ reaches $\widehat{\ell}$; because $\rho_y^*$ is increasing in $\ell$ with $\rho_y^*|_{\ell=\widehat{\ell}}=\rho^m$, it follows that $\rho_y^*<\rho^m$ whenever the disparate design is optimal, completing the proof of (a).  

 (b) We first restate this part formally, writing $\alpha_x^{\pm}:=\alpha^{\pm}(\rho_x^*)$ and $\alpha_m^{\pm}:=\alpha^{\pm}(\rho^m)$: under a disparate algorithm, the physician uses AI for type-$x$ patients exactly when $\alpha\in(\alpha_x^{-},\alpha_x^{+})$ and for type-$y$ patients exactly when $\alpha\in(\alpha_y^{-},\alpha_y^{+})$, and under the mandate, she uses AI for both types exactly when $\alpha\in(\alpha_m^{-},\alpha_m^{+})$. Comparing these cutoffs yields the ordering
\[
\alpha_x^{-}<\alpha_m^{-}<\alpha_y^{-}<\alpha_y^{+}<\alpha_m^{+}<\alpha_x^{+}.
\]
Moreover, each patient-indifference point sits a wedge $\tfrac{\theta r}{b}$ inside the corresponding use cutoff, the wedge being the physician's private benefit: $(1-\alpha)b=\rho_x^* b-c$ at $\alpha=\alpha_x^{-}+\tfrac{\theta r}{b}$, $\alpha b=\rho_x^* b-c$ at $\alpha=\alpha_x^{+}-\tfrac{\theta r}{b}$, and likewise for $\rho^m$. The formal statement is:

\emph{The expected patient welfare is affected as follows. For type-$x$ patients, welfare remains unchanged if $\alpha<\alpha_x^{-}$ or $\alpha>\alpha_x^{+}$; improves if $\alpha_x^{-}<\alpha<\min\{\alpha_m^{-},\,\alpha_x^{-}+\tfrac{\theta r}{b}\}$ or $\max\{\alpha_m^{+},\,\alpha_x^{+}-\tfrac{\theta r}{b}\}<\alpha<\alpha_x^{+}$; and worsens if $\min\{\alpha_m^{-},\,\alpha_x^{-}+\tfrac{\theta r}{b}\}<\alpha<\max\{\alpha_m^{+},\,\alpha_x^{+}-\tfrac{\theta r}{b}\}$. For type-$y$ patients, welfare remains unchanged if $\alpha<\alpha_m^{-}$ or $\alpha>\alpha_m^{+}$; worsens if $\alpha_m^{-}<\alpha<\min\{\alpha_y^{-},\,\alpha_m^{-}+\tfrac{\theta r}{b}\}$ or $\max\{\alpha_y^{+},\,\alpha_m^{+}-\tfrac{\theta r}{b}\}<\alpha<\alpha_m^{+}$; and improves if $\min\{\alpha_y^{-},\,\alpha_m^{-}+\tfrac{\theta r}{b}\}<\alpha<\max\{\alpha_y^{+},\,\alpha_m^{+}-\tfrac{\theta r}{b}\}$.}

We verify the statement interval by interval. For type-$x$ patients, the mandate can only remove AI use: the use interval shrinks from $(\alpha_x^{-},\alpha_x^{+})$ to $(\alpha_m^{-},\alpha_m^{+})$. Outside $(\alpha_x^{-},\alpha_x^{+})$ the physician never uses AI, so welfare is unchanged. On $(\alpha_x^{-},\alpha_m^{-})$, where use stops, welfare changes from $\rho_x^* b-c$ to $(1-\alpha)b$, an improvement precisely when $\alpha<\alpha_x^{-}+\tfrac{\theta r}{b}$; on $(\alpha_m^{+},\alpha_x^{+})$ it changes from $\rho_x^* b-c$ to $\alpha b$, an improvement precisely when $\alpha>\alpha_x^{+}-\tfrac{\theta r}{b}$. On the middle interval, where use continues, welfare falls from $\rho_x^* b-c$ to $\rho^m b-c\le \rho_x^* b-c$. Intersecting each comparison with its interval yields exactly the type-$x$ intervals in the statement, with each $\min$ and $\max$ recording whether the indifference point falls inside that interval.

For type-$y$ patients, the pattern reverses: the mandate can only add AI use, the use interval expanding from $(\alpha_y^{-},\alpha_y^{+})$ to $(\alpha_m^{-},\alpha_m^{+})$. Outside $(\alpha_m^{-},\alpha_m^{+})$ welfare is unchanged. On $(\alpha_m^{-},\alpha_y^{-})$, where use begins, welfare changes from $(1-\alpha)b$ to $\rho^m b-c$, an improvement precisely when $\alpha>\alpha_m^{-}+\tfrac{\theta r}{b}$; on $(\alpha_y^{+},\alpha_m^{+})$ it changes from $\alpha b$ to $\rho^m b-c$, an improvement precisely when $\alpha<\alpha_m^{+}-\tfrac{\theta r}{b}$. Where use already occurred, welfare rises from $\rho_y^* b-c$ to $\rho^m b-c\ge \rho_y^* b-c$ by part~(a). Intersecting as before yields the stated type-$y$ intervals.

(c) Substituting the equilibrium accuracies below \cref{equ:wel} into the welfare expressions, we obtain $W_x(\rho^m)>W_x(\rho_x^*) \Longleftrightarrow f < f_x$, where $f_x:=\frac{2 c \kappa_x (\kappa_x +\kappa_y)}{b ((2 + \beta) \kappa_x + \kappa_y)}$. 
Analogously, when  \(\ell \to 0\),  we have  $W_y(\rho^m)>W_y(\rho_y^*)|_{\ell\to 0} \Longleftrightarrow f > f_y$, where $f_y:=\frac{2 c \kappa_y (\kappa_x +\kappa_y)}{b ((2\beta+1) \kappa_y + \beta\kappa_x)}$.
We can verify that $f_x<f_y$.  Verifying $\Delta|_{\ell\to 0}>0$ is straightforward: the disparate design strictly dominates near $\ell=0$, so we may drop the restrictions of \cref{prop:developer_rho*} that guarantee profit dominance ($\beta(\beta+2)\kappa_x>\kappa_y$ and the associated lower bound on $f$). Write $f_\pi(\ell):=\frac{b \kappa_y (2c-2\theta r+\theta \ell)}{\beta (b+\theta \ell)^2}$ for the type-$y$ participation bound from \cref{prop:developer_rho*}. Then, 
\begin{align*}
    f_x> f_\pi(0) & \iff c<\frac{\kappa_y ((2 + \beta) \kappa_x +\kappa_y)  \theta r}{\kappa_y (2 \kappa_x +\kappa_y)-\beta \kappa_x^2 }\\ 
    f_y<\frac{\kappa_x}{2} & \iff  b> \frac{4 c \kappa_y (\kappa_x +\kappa_y)}{\kappa_x (\kappa_y+ \beta (\kappa_x + 2 \kappa_y))}\\    
  f_y<\frac{b \kappa_y}{2\beta (b+\theta \ell)}  & \iff b>\frac{4 \beta c (\kappa_x +\kappa_y)}{\kappa_y + \beta (\kappa_x +2 \kappa_y)}.
\end{align*}
We can verify that 
$\frac{4 c \kappa_y (\kappa_x +\kappa_y)}{\kappa_x (\kappa_y+ \beta (\kappa_x + 2 \kappa_y))}>\frac{4 \beta c (\kappa_x +\kappa_y)}{\kappa_y + \beta (\kappa_x +2 \kappa_y)}$.
Hence, under the two preconditions in part~(c), all three inequalities hold. As $\ell\to 0$, the mandate yields three regimes:
 \begin{itemize}
     \item if $f_\pi(0)<f<f_x$, then $W_x(\rho^m)>W_x(\rho_x^*)$, whereas $W_y(\rho^m)<W_y(\rho_y^*)$;
     \item if $f_x < f< f_y$, then both types are worse off under the mandate, that is, $W_x(\rho^m)<W_x(\rho_x^*)$ and $W_y(\rho^m)<W_y(\rho_y^*)$ (each endpoint leaves one type exactly indifferent);
     \item if $f_y<f<\min\{\frac{\kappa_x}{2}, \frac{b \kappa_y}{2\beta (b+\theta \ell)}\}|_{\ell\to 0}$, then $W_x(\rho^m)<W_x(\rho_x^*)$, whereas $W_y(\rho^m)>W_y(\rho_y^*)$. 
 \end{itemize}
Every welfare expression and equilibrium condition above is continuous in $\ell$. Hence, for each parameter tuple satisfying these conditions, there exists $\varepsilon>0$ such that, for all $0\le\ell<\varepsilon$, the same three regimes persist, with the boundaries evaluated at that $\ell$ and the type-$y$ indifference point replaced by a threshold $\bar f$, defined by $W_y(\rho^m)=W_y(\rho_y^*)$ (so that $\bar f\to f_y$ as $\ell\to 0$), exactly as stated in part~(c) of the proposition.
This completes the proof. \hfill \emph{Q.E.D.}

\newpage
\section*{\large Online Appendix for ``AI Model Design and Physician Liability''}
\ifblinded\else
\noindent Shujie Luan, Shubhranshu Singh, and Tinglong Dai

\medskip
\fi

\newcounter{OAsection}
\renewcommand{\theOAsection}{OA\arabic{OAsection}}

\makeatletter
\newcommand{\OAsection}[1]{%
  \refstepcounter{OAsection}%
  \def\@currentlabel{\theOAsection}%
  \subsection*{\normalsize\bfseries \theOAsection.\ #1}}
\makeatother

This Online Appendix has four parts: extended welfare comparisons under equal-accuracy requirements, technical derivations for the endogenous-pricing extension, the computations of the welfare-maximizing liability, and the extension with type-specific priors and its additional comparative statics.

\medskip

\setcounter{lemma}{0}
\setcounter{figure}{0}
\setcounter{section}{0}
\setcounter{proposition}{0}
\setcounter{equation}{0}
\setcounter{corollary}{0}
\setcounter{page}{1}

\renewcommand{\thelemma}{OA\arabic{lemma}}
\renewcommand{\thefigure}{OA\arabic{figure}}
\renewcommand{\thesection}{OA\arabic{section}}
\renewcommand{\theproposition}{OA\arabic{proposition}}
\renewcommand{\thecorollary}{OA\arabic{corollary}}
\renewcommand{\theequation}{OA\arabic{equation}}
\renewcommand{\thetable}{OA\arabic{table}}
\renewcommand{\thepage}{oa\arabic{page}}

\normalfont

\OAsection{Effect of Mandating Equal Accuracy on Patient Welfare (\cref{sec:welfare})}\label{sec:OA_welfare}

To assess how an equal-accuracy requirement affects aggregate welfare over a broader parameter range, we report numerical comparisons across $(\beta,\ell,f)$ in \cref{fig:mandating_impact}.

\begin{figure}[htbp]
    \centering
    \begin{subfigure}[b]{0.43\textwidth}
        \centering
        \includegraphics[width=\textwidth]
        {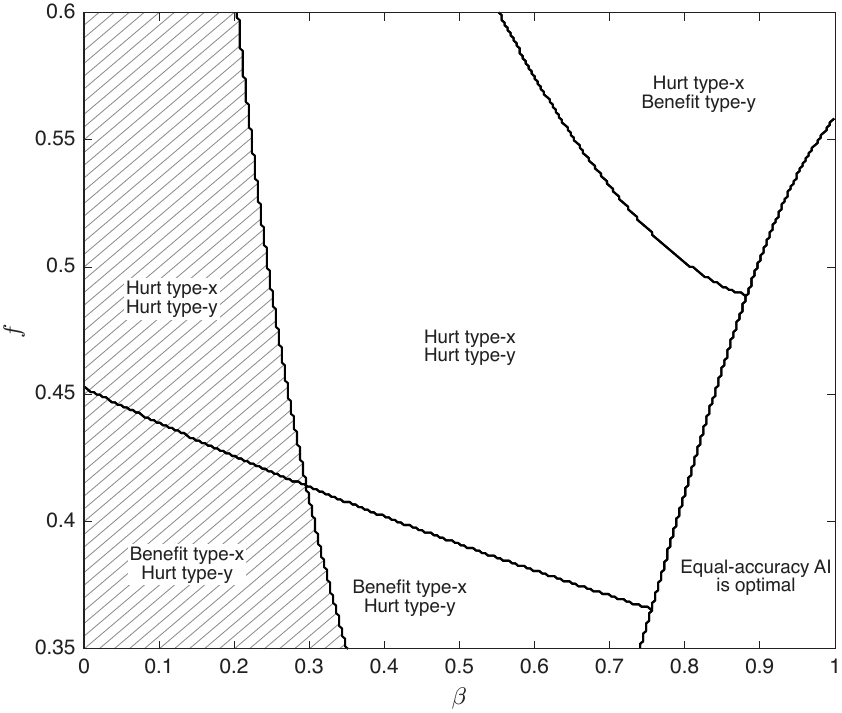}
        \caption{Variation in $(\beta,f)$}
        \label{fig:f_beta}
    \end{subfigure}
    \hfill
    \begin{subfigure}[b]{0.43\textwidth}
        \centering
        \includegraphics[width=\textwidth]{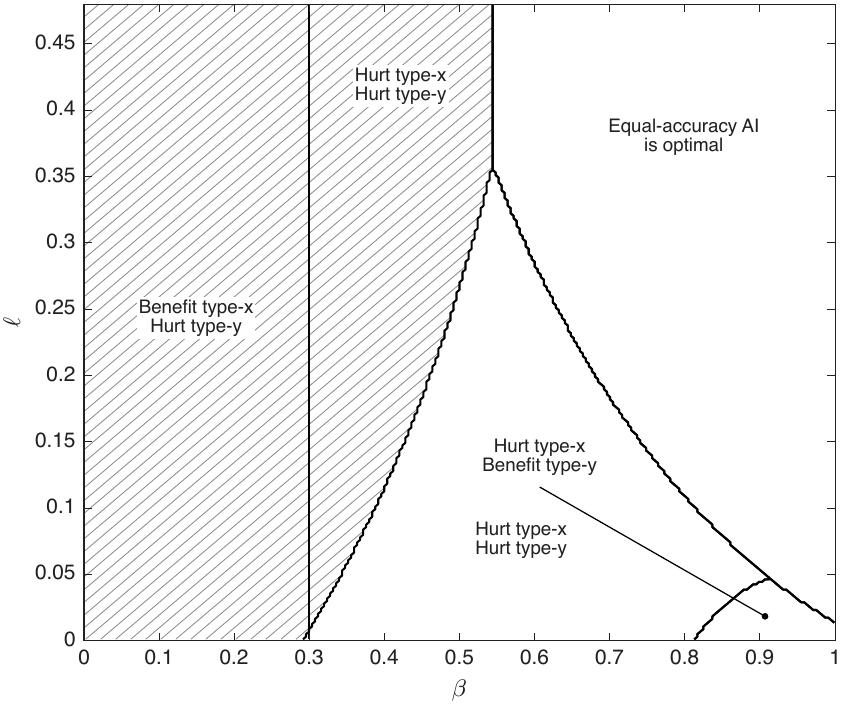}
        \caption{Variation in $(\ell,\beta)$}
        \label{fig:l_beta}
    \end{subfigure}
    \caption{Welfare effects of an equal-accuracy requirement for each patient type (for panel (a): $b=2$, $\theta=1.1$, $c=0.6$, $r=0.5$, $\kappa_x=1.1$, $\kappa_y=1.3$, $\ell=0.1$; for panel (b): $b=1.6$, $\theta=1$, $c=0.6$, $r=0.5$, $\kappa_x=1.3$, $\kappa_y=1.4$, $f=0.6$). In the hatched region, the firm serves type-$x$ patients only under the disparate design. Regions labeled ``Equal-accuracy AI is optimal'' are those in which the equal-accuracy design is optimal.}
    \label{fig:mandating_impact}
\end{figure}

\cref{fig:f_beta} complements \cref{prop:mandate_fairness_worse}(c). As $f$ increases, the welfare effect of the equal-accuracy mandate typically transitions from harming only type-$y$ patients, to harming both types, and then to harming only type-$x$ patients. Intuitively, higher $f$ strengthens the AI firm's incentive to invest in accuracy for type-$y$ patients under the equal-accuracy requirement; once these gains become sufficiently large, the probability that type-$y$ patients are harmed falls.

The same figure shows how these welfare patterns vary with the relative size of the disadvantaged segment. As $\beta$ increases, the equal-accuracy requirement again tends to move from harming only type-$y$ patients, to harming both types, and eventually to harming primarily type-$x$ patients. When $\beta$ is large, the AI firm internalizes a larger market on the type-$y$ side, which increases the return to improving type-$y$ accuracy under the requirement and limits the scope for welfare loss in that group. When $\beta$ is small, the requirement can instead improve welfare for type-$x$ patients by tempering overuse on the margin, partially offsetting the loss from reduced accuracy.

Finally, \cref{fig:l_beta} shows that higher liability can increase the likelihood that type-$y$ patients are harmed by the equal-accuracy requirement. When $\ell$ is large, the disparate-design equilibrium already induces substantial investment in type-$y$ accuracy, reducing the incremental benefit of further improvements under the requirement, although utilization effects remain. As a result, combining a high-liability environment with an equal-accuracy requirement (two instruments intended to protect disadvantaged patients) can be counterproductive in equilibrium.

\OAsection{AI Firm's Pricing Decision (\cref{sec:pricing_extn})}\label{sec:OA_pricing}
\label{sec:tech-details-pricing-extn}

\noindent \textsc{Proof of \cref{prop:rho_f}}.
The AI firm's profit functions under the disparate and equal-accuracy designs are given in \cref{equ:developer,equ:optimal_rho_fair}. We first consider the case in which the firm offers a disparate design. 
 The Hessian matrix $\mathbb{H}^{\mathrm D}$ of the profit function $\pi^{\mathrm D}(\rho_x,\rho_y,f)$ is given by:
\begin{align}
    \mathbb{H}^{\mathrm D}:=\begin{bmatrix}
-\frac{4(1 + \beta)}{b} & 2 & \frac{2\beta(b + \theta \ell)}{b} \\
2 & -2\kappa_x & 0 \\
\frac{2\beta(b + \theta \ell)}{b} & 0 & -2\kappa_y 
\end{bmatrix}
\end{align}
The corresponding leading principal minors of $\mathbb{H}^{\mathrm D}$ are evaluated as:
\begin{align*}
    \mathbb D_1=-\frac{4(1 + \beta)}{b}, ~~
       \mathbb D_2=\frac{8(1 + \beta)\kappa_x}{b}-4 ~\text{and}~
       \mathbb D_3=\frac{4D(\ell)}{b^2},
\end{align*}
where \[
D(\ell)
:=2 b^2(\beta^2\kappa_x+\kappa_y)-4 b\kappa_x\big((\beta+1)\kappa_y-\beta^2\theta\ell\big)+2\beta^2\kappa_x\theta^2\ell^2 .
\]
To ensure the profit function $\pi^{\mathrm D}(\rho_x,\rho_y,f)$ is jointly concave in $(f, \rho_x, \rho_y)$, the Hessian matrix must be negative definite, which requires $\mathbb D_1< 0$, $\mathbb D_2> 0$, and $\mathbb D_3< 0$. These conditions yield the parameter-space boundaries  $b < \frac{2 (1+\beta)\kappa_x \kappa_y}{\beta^2 \kappa_x + \kappa_y} $ and $\ell < \sqrt{\frac{ 2b \kappa_x \kappa_y -b^2 \kappa_y + 2b\beta \kappa_x \kappa_y}{\beta^2 \theta^2 \kappa_x }} - \frac{b}{\theta}$. 

Under these concavity conditions, we derive the optimal decision vector $(f, \rho_x, \rho_y)$.  The first-order condition $\frac{\partial \pi^{\mathrm D}(\rho_x,\rho_y,f)}{\partial f}=0$ yields the unique best-response price
\[
f(\rho_x,\rho_y)
=\frac{b\big(\beta(2\rho_y-1)+2\rho_x-1\big)-2\theta\beta\ell(1-\rho_y)+2(\beta+1)\theta r}{4(\beta+1)}.
\]
Substituting $f(\rho_x,\rho_y)$ back into $\pi^{\mathrm D}$ and imposing the first-order conditions with respect to $\rho_x$ and $\rho_y$ yields the optimal accuracies $(\rho_x^*,\rho_y^*)$ and the induced optimal price $f^{\mathrm D}:=f(\rho_x^*,\rho_y^*)$: 
\[
\rho_x^*
=\frac{1}{2}+\frac{b\kappa_y\theta\big(\beta\ell-2(\beta+1)r\big)}{D(\ell)},
\qquad
\rho_y^*
=\frac{1}{2}+\frac{\beta\kappa_x\theta(b+\theta\ell)\big(\beta\ell-2(\beta+1)r\big)}{D(\ell)},
\]
and
\[
f^{\mathrm D}
=\frac{b\kappa_x\kappa_y\theta\big(\beta\ell-2(\beta+1)r\big)}{D(\ell)}.
\]
To simplify the exposition, we define the auxiliary variable $\mathcal{T} = \frac{\theta(\beta\ell-2(\beta+1)r)}{D(\ell)}$, allowing the equilibrium expressions to be rewritten compactly as $\rho_x^* = \frac{1}{2}+b\kappa_y \mathcal{T}$, $\rho_y^* = \frac{1}{2}+\beta\kappa_x(b+\theta\ell)\mathcal{T}$, and $f^{\mathrm D}=b\kappa_x\kappa_y \mathcal{T}$.

We next identify the parameter regimes that guarantee the validity of these interior optimal solutions. First, given $D(\ell)<0$, the requirements $\rho_x^* > 1/2$, $\rho_y^* > 1/2$, and $f^{\mathrm D} > 0$ hold if and only if $\mathcal{T}>0$, that is, if and only if $\ell < \frac{2(1+\beta) r}{\beta}$. Additionally, imposing the logical upper bounds on accuracy levels ($\rho_x^* < 1$ and $\rho_y^* < 1$) requires $\mathcal{T} < \frac{1}{2b \kappa_y}$ and $\mathcal{T} < \frac{1}{2\beta \kappa_x (b +  \theta \ell)}$, respectively. To guarantee that $f^{\mathrm D} > \theta r$, we require $\mathcal{T} > \frac{\theta r}{b \kappa_x \kappa_y}$.  
Finally, substituting $(\rho_y^*, f^{\mathrm D})$ into the type-$y$ component of the profit function yields $\pi_y(\rho_y^*, f^{\mathrm D}) = y_1 \mathcal{T}^2 + y_2 \mathcal{T}$, where:
\[
y_1 = \beta \kappa_x^2 \kappa_y \left(b^2 \beta - 2 b \kappa_y + 2 b \beta \theta \ell + \beta  \theta^2 \ell^2\right) \quad \text{and} \quad y_2 = -\beta \kappa_x \kappa_y (\ell - 2 r) \theta.
\]
To ensure a well-behaved optimization problem, we impose $y_1 < 0$, which holds if and only if $\kappa_y > \frac{b^2 \beta + 2b \beta \theta\ell  + \beta  \theta^2 \ell^2}{2b}$ (the standing assumption $\kappa_y>\kappa_x$ is maintained throughout). Setting $y_2 > 0$ yields the condition $\ell < 2r$. Under these restrictions, the unique non-zero root of $\pi_y(\rho_y^*, f^{\mathrm D}) = 0$ with respect to $\mathcal{T}$ is strictly positive:
\[
\mathcal{T} = \frac{(\ell - 2 r) \theta}{\kappa_x \left(b^2 \beta - 2b \kappa_y + 2b \beta  \theta \ell + \beta  \theta^2 \ell^2\right)} > 0.
\]
Consequently, ensuring positive profits from the type-$y$ segment ($\pi_y(\rho_y^*, f^{\mathrm D}) > 0$) is mathematically equivalent to bounding the auxiliary variable such that $\mathcal{T} < \frac{(\ell - 2r)\theta}{\kappa_x \left(b^2 \beta - 2b \kappa_y + 2b \beta \theta \ell  + \beta  \theta^2 \ell^2\right)}$.

Next, consider the equal-accuracy design, imposing $\rho_x=\rho_y\equiv\rho$. 
The Hessian matrix $\mathbb{H}^{\mathrm E}$ of the profit function $\pi^{\mathrm E}(\rho,f)$ is given by:
\begin{align}
    \mathbb{H}^{\mathrm E}:=\begin{bmatrix}
-\frac{4(1 + \beta)}{b}  & 2(1+\beta) \\
2(1+\beta) & -2(\kappa_x+\kappa_y)
\end{bmatrix}
\end{align}
The corresponding leading principal minors of $\mathbb{H}^{\mathrm E}$ are evaluated as:
\begin{align*}
    \mathbb D_1=-\frac{4(1 + \beta)}{b}  ~\text{and}~
       \mathbb D_2=-\frac{4(1+\beta)(b(1+\beta)-2(\kappa_x+\kappa_y))}{b}.
\end{align*}
To ensure the profit function $\pi^{\mathrm E}(\rho,f)$ is jointly concave in $(f, \rho)$, the Hessian matrix must be negative semi-definite, which requires $\mathbb D_1\le 0$ and $\mathbb D_2\ge 0$; these conditions yield $b < \frac{2 (\kappa_x +\kappa_y)}{1+\beta} $. 
Under these concavity conditions, we derive the optimal decision vector $(f, \rho)$. The first-order condition $\frac{\partial\pi^{\mathrm E}(\rho,f)}{\partial f}=0$ gives
\[
f(\rho)=\Big(\frac{\rho}{2}-\frac{1}{4}\Big)b+\frac{1}{2}\theta r.
\]
Substituting into $\pi^{\mathrm E}$ and imposing $\frac{\partial \pi^{\mathrm E}(\rho,f(\rho))}{\partial \rho}=0$ yields
\[
\rho^*
=\frac{1}{2}+\frac{(\beta+1)\theta r}{2(\kappa_x+\kappa_y)-b(\beta+1)},
\qquad
f^{\mathrm E}:=f(\rho^*)
=\frac{(\kappa_x+\kappa_y)\theta r}{2(\kappa_x+\kappa_y)-b(\beta+1)}.
\]
 Ensuring that the equilibrium accuracy level is interior, $1/2 < \rho^* < 1$, and that the condition $f^{\mathrm E} > \theta r$ holds yields the following joint parameter restrictions:
\[
r < \frac{\kappa_x + \kappa_y}{2\theta(1 + \beta)} \quad \text{and} \quad \frac{\kappa_x + \kappa_y}{1 + \beta} < b < \frac{2(\kappa_x + \kappa_y) - 2r \theta(1 + \beta)}{1 + \beta}.
\]

In addition, requiring $\pi^{\mathrm D}(\rho_x^*,\rho_y^*,f^{\mathrm D})>\pi^{\mathrm E}(\rho^*,f^{\mathrm E})$ leads to (the left-hand side below equals $4\big[\pi^{\mathrm D}(\rho_x^*,\rho_y^*,f^{\mathrm D})-\pi^{\mathrm E}(\rho^*,f^{\mathrm E})\big]$): 
\begin{align*}
 & 4b \kappa_x \kappa_y \left(  b\left(\beta^2 \kappa_x + \kappa_y \right) -2(1+\beta)\kappa_x \kappa_y \right) \mathcal{T}^2 \\
& + 4\theta\kappa_x \kappa_y \mathcal{T} \left( 2r(1 + \beta) + \beta \ell ( 2b \beta \kappa_x \mathcal{T} -1) \right)  \\
& + 4 \theta^2\left( \frac{(1+\beta)(\kappa_x+\kappa_y)r^2}{b\left(b(1+\beta) - 2(\kappa_x+\kappa_y)\right)} + \beta^2 \kappa_x^2 \kappa_y \ell^2 \mathcal{T}^2 \right)  \ > 0.
\end{align*}
Taken together, the conditions above characterize when a disparate AI design is optimal; see \cref{fig:welfare_equalacc_endogf} for a numerical illustration.

We now compare the maximized payoffs as $\ell$ varies. By the envelope theorem, the maximized payoff under the equal-accuracy design, $\pi^{\mathrm E}(\rho^*,f^{\mathrm E})$, is independent of $\ell$. Under the disparate design, the envelope theorem implies
\[
\frac{\partial \pi^{\mathrm D}(\rho_x^*,\rho_y^*,f^{\mathrm D})}{\partial \ell}
=\left.\frac{\partial \pi^{\mathrm D}(\rho_x,\rho_y,f)}{\partial \ell}\right|_{(\rho_x,\rho_y,f)=(\rho_x^*,\rho_y^*,f^{\mathrm D})}
=-\frac{2\beta\theta f^{\mathrm D}}{b}\,(1-\rho_y^*)<0,
\]
where the inequality holds within the feasible region we discussed earlier; thus,  the maximized disparate-design payoff is strictly decreasing in $\ell$. Moreover, as $\ell\to 0$, the disparate-design problem converges to the environment without the liability channel, and allowing $(\rho_x,\rho_y)$ to differ is (weakly) valuable under the cost asymmetry. In particular, the disparate-design optimum yields a strictly higher payoff than the best equal-accuracy design for $\ell$ near zero. Because $\pi^{\mathrm E}(\rho^*,f^{\mathrm E})$ is constant whereas $\pi^{\mathrm D}(\rho_x^*,\rho_y^*,f^{\mathrm D})$ declines in $\ell$, at most one cutoff $\ell^\dagger$ exists within the feasible region such that the firm prefers a disparate design for $\ell<\ell^\dagger$. If no such $\ell^\dagger$  falls within the feasible region, the disparate design remains globally optimal in the region.  Conversely, when $\ell\ge \ell^\dagger$,  the equal-accuracy design can become optimal; see \cref{fig:welfare_equalacc_endogf} for a numerical illustration.

We next show that whenever a disparate design is optimal, it must satisfy $\rho_x^*>\rho_y^*$. At $\ell\to 0$, the expressions above imply $\rho_x^*>\rho_y^*$. Indeed, $\rho_x^*-\rho_y^*=\mathcal{T}\big(b\kappa_y-\beta\kappa_x(b+\theta\ell)\big)$; with $\mathcal{T}>0$, the bracket is strictly decreasing in $\ell$ and positive at $\ell=0$ because $\kappa_y>\beta\kappa_x$, so the gap changes sign at most once. Suppose instead that for some $\hat\ell$ the disparate-design optimum satisfies $\rho_x^*=\rho_y^*$. Then the candidate solution is itself equal-accuracy, and the firm can (weakly) improve by switching to the equal-accuracy problem; moreover, because $\pi^{\mathrm D}(\rho_x^*,\rho_y^*,f^{\mathrm D})$ is decreasing in $\ell$ while $\pi^{\mathrm E}(\rho^*,f^{\mathrm E})$ is independent of $\ell$, once $\rho_x^*=\rho_y^*$ holds at some $\hat\ell$, the equal-accuracy design is optimal at $\hat\ell$ and for all larger $\ell$. Hence, if the disparate design is optimal, it must be that $\rho_x^*>\rho_y^*$.

Finally, $\rho_x^*$ (and similarly $f^{\mathrm D}$) can be non-monotone in $\ell$. Differentiating yields
\[
\frac{\partial \rho_x^*}{\partial \ell}
=\frac{b\beta\kappa_y\theta\cdot E_1(\ell)}
{2\Big(b^2(\beta^2\kappa_x+\kappa_y)-2b\kappa_x\big((\beta+1)\kappa_y-\beta^2\theta\ell\big)+\beta^2\kappa_x\theta^2\ell^2\Big)^2},
\]
where
\begin{align}\label{equ:E_1}
    E_1(\ell)
=b^2(\beta^2\kappa_x+\kappa_y)-2b(\beta+1)\kappa_x(\kappa_y-2\beta\theta r)
+\beta\kappa_x\theta^2\ell\big(4(\beta+1)r-\beta\ell\big).
\end{align}
Because $E_1(\ell)$ is  concave in $\ell$, there exist parameter values for which $E_1(0)<0$ but $E_1(\epsilon)>0$ for some small $\epsilon>0$ below the minimum of the design-switch cutoff $\ell^\dagger$ and $\frac{4(\beta+1)r}{\beta}$. To establish this, it suffices to ensure that $b^2(\beta^2\kappa_x+\kappa_y)-2b(\beta+1)\kappa_x(\kappa_y-2\beta\theta r)$ is strictly negative but sufficiently close to zero, so that  $E_1(\epsilon)>0$ for a small $\epsilon>0$.  This implies that $\rho_x^*$ initially decreases and subsequently increases with $\ell$. The same logic applies to $f^{\mathrm D}$. \hfill \emph{Q.E.D.}

\vspace{.1in}

\begin{proposition}\label{prop:disparity_f_c}
There exist parameter values under which, as $\ell$ increases, equilibrium AI use for type-$y$ patients first decreases and then increases.
\end{proposition}

\noindent \textsc{Proof of \cref{prop:disparity_f_c}}.
\noindent With endogenous pricing and a disparate design,
\[
d_y^{\mathrm D}=\beta\cdot \frac{(2\rho_y^*-1)b-2(1-\rho_y^*)\theta \ell-2 f^{\mathrm D}+2\theta r}{b},
\]
where $(\rho_y^*,f^{\mathrm D})$ are given in \cref{prop:rho_f}. Differentiating $d_y^{\mathrm D}$ with respect to $\ell$ and simplifying yields
\[
\frac{\partial d_y^{\mathrm D}}{\partial \ell}
=\frac{\beta\kappa_y\theta\big((\beta+2)\kappa_x-b\big)\cdot E_1(\ell)}
{\Big(b^2(\beta^2\kappa_x+\kappa_y)-2b\kappa_x\big((\beta+1)\kappa_y-\beta^2\theta\ell\big)+\beta^2\kappa_x\theta^2\ell^2\Big)^2},
\]
where $E_1(\ell)$ is defined in \cref{equ:E_1} in the proof of \cref{prop:rho_f} and 
is a concave quadratic in $\ell$.
Let $\ell_{\min}$ denote the smaller root of $E_1(\ell)=0$:
\[
\ell_{\min}
=\frac{2(\beta+1) r}{\beta}
-\frac{\sqrt{\,b^2\!\left(\beta^2+\kappa_y/\kappa_x\right)-2 b(\beta+1)\left(\kappa_y-2\beta r \theta\right)+4(\beta+1)^2 r^2 \theta^2\,}}{\beta \theta}.
\]
If $(\beta+2)\kappa_x>b$, $E_1(0)<0$, and $E_1(0)>-4\beta(\beta+2)\kappa_x\theta^2r^2$ (so that $E_1$ has a root in the feasible region), then $\frac{\partial d_y^{\mathrm D}}{\partial \ell}<0$ for $\ell$ near zero and becomes positive for $\ell$ just above $\ell_{\min}$. The last condition cannot be dispensed with: $E_1$ is concave with vertex at $\ell=2(\beta+1)r/\beta$, which lies beyond the feasible region because positive type-$y$ profit requires $\ell<2r$; hence $E_1$ is strictly increasing throughout the feasible region, and its supremum there is $E_1(2r)=E_1(0)+4\beta(\beta+2)\kappa_x\theta^2r^2$. Consequently, the crossing at $\ell_{\min}$ is unique and the sign change is global rather than local. Provided $\ell_{\min}$ lies below the design-switch cutoff $\ell^\dagger$ (so the firm still supplies a disparate design in this neighborhood), the comparative static in \cref{prop:disparity_f_c} follows. \hfill \emph{Q.E.D.}
\vspace{.1in}

The mechanism behind \cref{prop:disparity_f_c} parallels that of \cref{prop:disparity}: liability changes the physician's marginal willingness to rely on AI for type-$y$ patients, and the firm responds by jointly adjusting price and accuracy, which can reverse the direction of utilization as $\ell$ increases.

We next examine whether an equal-accuracy requirement can still reduce welfare for both patient types when the AI firm endogenizes the per-use price (that is, when $f$ becomes a choice variable and we impose $c=f$). \cref{fig:welfare_equalacc_fequalc} reports the comparison. Panel (a) corresponds to the baseline with exogenous $(c,f)$, whereas panel (b) corresponds to the extension in which the firm chooses $f$ and $c$ moves one for one with it. The two panels use the same parameter values and differ only in whether $(c,f)$ are fixed or endogenously determined. The central patterns from the baseline carry over: the same welfare regions arise, although
their boundaries shift, and the mandate harms both patient groups over a substantial region.

Endogenous pricing adds one feature. Because the fee $f$ is common to both segments, a firm that
forgoes type-$y$ patients re-optimizes its fee and type-$x$ accuracy for type-$x$ patients
alone (here, $f=0.65$ and $\rho_x\rightarrow 1$). This raises type-$x$ welfare under the disparate
design and therefore makes the mandate less favorable to type-$x$ patients. Each regime thus
has its own type-$x$ boundary: the mandate benefits type-$x$ patients only if $\beta<0.18$
when the firm serves type-$x$ patients only, and only if $\beta<0.29$ when it serves both
types. In contrast, with exogenous $f$, type-$x$ accuracy does not depend on whether type-$y$
patients are served, and a single boundary applies ($\beta\approx0.30$). Type-$y$ patients
benefit only if the mandated accuracy exceeds $\frac12+\frac{c}{b}$, which holds in both
panels if and only if $b(1+\beta)>\kappa_x+\kappa_y$, and only when $\ell$ is small;
elsewhere, the use the mandate induces occurs in cases in which the clinical gain falls short of the use
cost.

\begin{figure}[htbp]
    \centering
    \begin{subfigure}[b]{0.43\textwidth}
        \centering
        \includegraphics[width=\textwidth]
        {fig/FigOA1b.pdf}
        \caption{Exogenous $f$}
        \label{fig:welfare_equalacc_exogf}
    \end{subfigure}
    \hfill
    \begin{subfigure}[b]{0.43\textwidth}
        \centering
        \includegraphics[width=\textwidth]
         {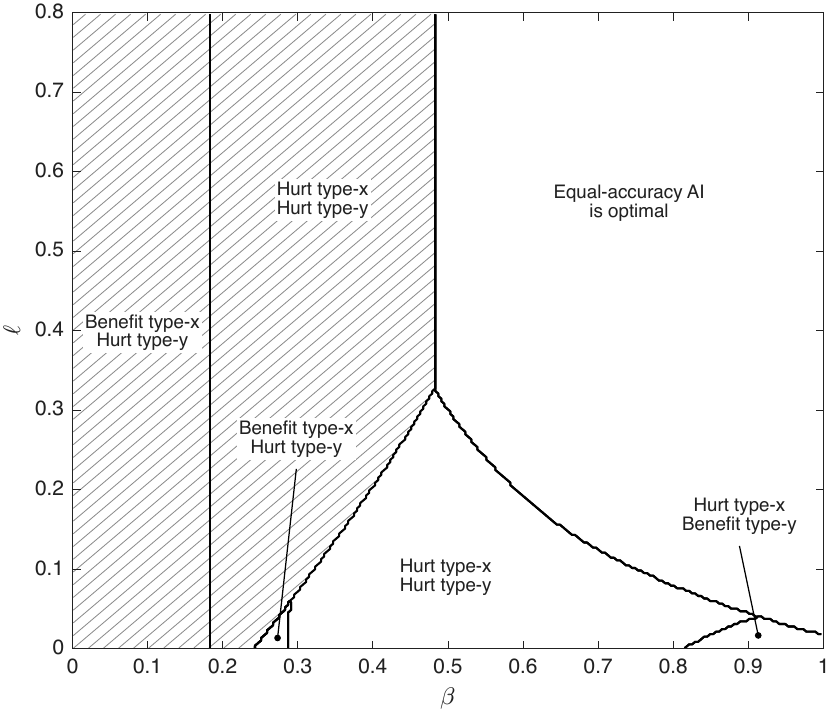}
        \caption{Endogenous $f$}
        \label{fig:welfare_equalacc_endogf}
    \end{subfigure}
    \caption{Welfare effects of an equal-accuracy requirement under exogenous and endogenous pricing. Panel (a) reproduces the baseline with exogenous $(c,f)=(0.6,0.6)$; panel (b)
endogenizes $f$ and imposes $c=f$. Both panels use $b=1.6$, $\theta=1$, $r=0.5$, $\kappa_x=1.3$, and $\kappa_y=1.4$. In the hatched regions, the firm serves type-$x$ patients only under the disparate design. Regions labeled ``Equal-accuracy AI is optimal'' are those in
which the equal-accuracy design is optimal.}
    \label{fig:welfare_equalacc_fequalc}
\end{figure}

\medskip

\OAsection{Welfare-Maximizing Liability (\cref{sec:welfaremax_extn})}\label{sec:OA_welfare_max}
\label{sec:tech-details-welfaremax-extn}

For each parameter tuple, we compute equilibrium aggregate welfare under the disparate-design and equal-accuracy-design regimes as functions of liability and then compare these values at the relevant design boundary. Operationally, we first identify the design-switch cutoff implied by \cref{prop:developer_rho*}, and then evaluate welfare on each side of that cutoff using the corresponding equilibrium design.

Across the numerical configurations reported in the main text, the welfare-maximizing liability is attained at the regime boundary: either at the largest $\ell$ for which the disparate design remains optimal or at the smallest $\ell$ that induces the equal-accuracy design. This procedure yields the two benchmark patterns illustrated in \cref{fig:fig1,fig:fig2}.

\OAsection{Type-Specific Priors (\cref{sec:type-specific-priors})}\label{sec:OA_priors}
\label{sec:tech-details-type-specific-priors}

As in the baseline, the physician's AI-use rule is characterized by two prior cutoffs, and, conditional on use, the physician follows the AI signal.

\begin{lemma}\label{lem:AI_decision_imaltruistic_xpatient_alphat}
    (a) For type-$x$ patients, the physician uses AI if and only if
    \[
    \frac{(1-\rho_x)b+c-\theta r}{b}
    < \alpha_x <
    \min\left\{\overline{\alpha},\frac{\rho_x b-c+\theta r}{b}\right\}.
    \]
    Moreover, conditional on AI use, the physician follows the AI signal.
    
    (b) For type-$y$ patients, the physician uses AI if and only if
    \[
    \max\left\{\underline{\alpha},\frac{(1-\rho_y)(b+\theta \ell)+c-\theta r}{b}\right\}
    < \alpha_y <
    \frac{\rho_y b-(1-\rho_y)\theta \ell-c+\theta r}{b}.
    \]
    Moreover, conditional on AI use, the physician follows the AI signal.
\end{lemma} 

The proof of \cref{lem:AI_decision_imaltruistic_xpatient_alphat} follows the same logic as the proofs of \cref{lem:AI_decision_imaltruistic_xpatient,lem:AI_decision_imaltruistic_ypatient} and is therefore omitted. Relative to the baseline, the ordering of type-specific AI use need not be uniform because use is jointly determined by type-dependent priors and payoff trade-offs. Nevertheless, the baseline pattern remains feasible: parameter values exist for which the physician uses AI less for type-$y$ patients than for type-$x$ patients.
Under a disparate algorithm, integrating the physician's AI-use rule in \cref{lem:AI_decision_imaltruistic_xpatient_alphat} against the prior densities ($1/\overline{\alpha}$ on $(0, \overline{\alpha})$ for type-$x$ patients and $1/(1-\underline{\alpha})$ on $(\underline{\alpha},1)$ for type-$y$ patients) yields expected AI demand for each patient type. To ease notation, for type-$x$ patients, let $\bar d_x(\rho_x)=\overline{\alpha}-((1-\rho_x)b+c-\theta r)/b$ and $\tilde d_x(\rho_x)=((2\rho_x-1)b-2c+2\theta r)/b$. For type-$y$ patients, let $\bar d_y(\rho_y)=((\rho_y b-(1-\rho_y)\theta \ell-c+\theta r)/b)-\underline{\alpha}$ and $\tilde d_y(\rho_y)=((2\rho_y-1)b-2(1-\rho_y)\theta \ell-2c+2\theta r)/b$.
\begin{align}\label{equ:AI demand prior}
d_x^{\mathrm D}&=\min\{\bar d_x(\rho_x),\tilde d_x(\rho_x)\}, \notag\\
d_y^{\mathrm D}&=\tilde{\beta}\min\{\bar d_y(\rho_y),\tilde d_y(\rho_y)\}.
\end{align}
Note that all optima characterized in this extension lie in the region of strictly positive demand
(this holds in the numerical example we report; in the hatched region of \cref{fig:welfare_alphat}, the firm instead serves type-$x$ patients only), so the
positive-part operator is omitted from the piecewise expressions without
affecting the analysis.
The AI firm then chooses $(\rho_x,\rho_y)$ to maximize expected profit in \cref{equ:developer}. Under an equal-accuracy algorithm, define $\bar d_x^{\mathrm E}(\rho)=\overline{\alpha}-((1-\rho)b+c-\theta r)/b$, $\bar d_y^{\mathrm E}(\rho)=(\rho b-c+\theta r)/b-\underline{\alpha}$, and $\tilde d^{\mathrm E}(\rho)=((2\rho-1)b-2c+2\theta r)/b$. Then
\begin{align*}
d_x^{\mathrm E}&=\min\{\bar d_x^{\mathrm E}(\rho),\tilde d^{\mathrm E}(\rho)\}, \notag\\
d_y^{\mathrm E}&=\tilde{\beta}\min\{\bar d_y^{\mathrm E}(\rho),\tilde d^{\mathrm E}(\rho)\}.
\end{align*}
The AI firm chooses $\rho\in(1/2,1)$ to maximize expected profit in \cref{equ:optimal_rho_fair}.

The next proposition characterizes the resulting equilibrium accuracy choices.
\begin{proposition}\label{prop:developer_rho*_alpha_t}  
There exist parameter regions in which the AI firm develops a disparate algorithm when $\ell$ is below a cutoff and an equal-accuracy algorithm when $\ell$ exceeds that cutoff.  Moreover, if $b\kappa_y > 2\tilde{\beta}\kappa_x(b + \theta\ell)$, then the associated accuracies satisfy $\rho^{*}_x > \rho^{*}_y$ whenever the disparate design is optimal. The proof characterizes the optimal accuracies under both designs. %
\end{proposition}

\noindent \textsc{Proof of \cref{prop:developer_rho*_alpha_t}}. The proof proceeds in three steps. Step~1 characterizes the optimal accuracies under the disparate design and establishes the ranking $\rho^{*}_x > \rho^{*}_y$ under $b\kappa_y > 2\tilde{\beta}\kappa_x(b + \theta\ell)$. Step~2 characterizes the optimal accuracy under the equal-accuracy design. Step~3 establishes the liability cutoff.

\noindent\textbf{Step 1: Disparate design.} We begin with the disparate design. Because the physician's usage rule changes discretely at an accuracy cutoff, the AI firm's objective is piecewise in $\rho_t$. Conditional on supplying a disparate design, the problem separates by patient type. We therefore solve for $\rho_x^*$ and $\rho_y^*$ in turn. For type-$x$ patients, define the cutoff accuracy as $\bar\rho_x:=\frac{\overline{\alpha}b+c-\theta r}{b}$.
The firm chooses $\rho_x\in(1/2,1)$ to maximize
\begin{align*}
\pi_x(\rho_x)=
\begin{cases}
\tilde{f}\cdot \tilde d_x(\rho_x)-\kappa_x(\rho_x-\tfrac12)^2,
& \text{if } \tfrac12<\rho_x\le \bar\rho_x,\\[6pt]
\tilde{f}\cdot\bar d_x(\rho_x)-\kappa_x(\rho_x-\tfrac12)^2,
& \text{if } \bar\rho_x<\rho_x<1.
\end{cases}
\end{align*}
On each region, $\pi_x(\rho_x)$ is strictly concave in $\rho_x$. The first-order condition yields the interior candidate
$\rho_x=\tfrac12+\tfrac{\tilde{f}}{\kappa_x}$ in the first region and $\rho_x=\tfrac12+\tfrac{\tilde{f}}{2\kappa_x}$ in the second. The global maximizer is obtained by checking whether the relevant candidate lies in its region; otherwise, the optimum is attained at the boundary $\bar\rho_x$. This yields
\begin{align}\label{eq:rhoxstar}
\rho_x^*=
\begin{cases}
\tfrac12+\dfrac{\tilde{f}}{\kappa_x}, & \text{if } \tfrac12+\dfrac{\tilde{f}}{\kappa_x}\le \bar\rho_x,\\[6pt]
\bar\rho_x, & \text{if } \tfrac12+\dfrac{\tilde{f}}{2\kappa_x}\le \bar\rho_x<\tfrac12+\dfrac{\tilde{f}}{\kappa_x},\\[6pt]
\tfrac12+\dfrac{\tilde{f}}{2\kappa_x}, & \text{if } \bar\rho_x<\tfrac12+\dfrac{\tilde{f}}{2\kappa_x}.
\end{cases}
\end{align}

For type-$y$ patients, define the corresponding cutoff as $\bar\rho_y:=1-\frac{\underline{\alpha}b-(c-\theta r)}{b+\theta\ell}$.
The AI firm chooses $\rho_y\in(1/2,1)$ to maximize
\begin{align*}
\pi_y(\rho_y;\ell)=
\begin{cases}
\tilde{f}\tilde{\beta}\cdot \tilde d_y(\rho_y)-\kappa_y(\rho_y-\tfrac12)^2,
& \text{if } \tfrac12<\rho_y\le \bar\rho_y,\\[6pt]
\tilde{f}\tilde{\beta}\cdot\bar d_y(\rho_y)-\kappa_y(\rho_y-\tfrac12)^2,
& \text{if } \bar\rho_y<\rho_y<1.
\end{cases}
\end{align*}
Again, strict concavity holds on each region. The first-order condition yields the interior candidate
$\rho_y=\tfrac12+\dfrac{\tilde{\beta} \tilde{f}(b+\theta\ell)}{b\kappa_y}$ in the first region and
$\rho_y=\tfrac12+\dfrac{\tilde{\beta} \tilde{f}(b+\theta\ell)}{2b\kappa_y}$ in the second. Comparing these candidates with $\bar\rho_y$ gives
\begin{align}\label{eq:rhoystar}
\rho_y^*=
\begin{cases}
\tfrac12+\dfrac{\tilde{\beta} \tilde{f}(b+\theta\ell)}{b\kappa_y},
& \text{if } \tfrac12+\dfrac{\tilde{\beta} \tilde{f}(b+\theta\ell)}{b\kappa_y}\le \bar\rho_y,\\[6pt]
\bar\rho_y,
& \text{if } \tfrac12+\dfrac{\tilde{\beta} \tilde{f}(b+\theta\ell)}{2b\kappa_y}\le \bar\rho_y<\tfrac12+\dfrac{\tilde{\beta} \tilde{f}(b+\theta\ell)}{b\kappa_y},\\[6pt]
\tfrac12+\dfrac{\tilde{\beta} \tilde{f}(b+\theta\ell)}{2b\kappa_y},
& \text{if } \bar\rho_y<\tfrac12+\dfrac{\tilde{\beta} \tilde{f}(b+\theta\ell)}{2b\kappa_y}.
\end{cases}
\end{align}

\emph{Accuracy ranking.} We now show that the assumption $b\kappa_y > 2\tilde{\beta}\kappa_x(b + \theta\ell)$ 
implies $\rho^{*}_x > \rho^{*}_y$ across all branch combinations of
\eqref{eq:rhoxstar} and \eqref{eq:rhoystar}. First, \eqref{eq:rhoxstar} implies
\begin{equation*}
\rho^{*}_x \;\ge\; \tfrac12 + \tfrac{\tilde{f}}{2\kappa_x};
\end{equation*}
this holds with equality on the third branch, holds on the second branch
by its defining condition
$\bar\rho_x \ge \tfrac12 + \tfrac{\tilde{f}}{2\kappa_x}$, and holds on the first
branch because $\tfrac{\tilde{f}}{\kappa_x} > \tfrac{\tilde{f}}{2\kappa_x}$. Second, \eqref{eq:rhoystar} implies
\begin{equation*}
\rho^{*}_y \;\le\; \tfrac12 + \tfrac{\tilde{\beta} \tilde{f}(b+\theta\ell)}{b\kappa_y};
\end{equation*}
this holds with equality on the first branch, holds on the second branch
by its defining condition
$\bar\rho_y < \tfrac12 + \tfrac{\tilde{\beta} \tilde{f}(b+\theta\ell)}{b\kappa_y}$, and
holds on the third branch because
$\tfrac{\tilde{\beta} \tilde{f}(b+\theta\ell)}{2b\kappa_y}
< \tfrac{\tilde{\beta} \tilde{f}(b+\theta\ell)}{b\kappa_y}$.
The assumption
$b\kappa_y > 2\tilde{\beta}\kappa_x(b+\theta\ell)$ is equivalent to
\begin{equation*}
\frac{\tilde{\beta} \tilde{f}(b+\theta\ell)}{b\kappa_y} \;<\; \frac{\tilde{f}}{2\kappa_x},
\end{equation*}
and therefore
\begin{equation*}
\rho^{*}_y \;\le\; \tfrac12 + \tfrac{\tilde{\beta} \tilde{f}(b+\theta\ell)}{b\kappa_y}
\;<\; \tfrac12 + \tfrac{\tilde{f}}{2\kappa_x} \;\le\; \rho^{*}_x,
\end{equation*}
so that $\rho^{*}_x > \rho^{*}_y$ whenever the disparate design is
supplied. In particular, the disparity-liability channel, which is
defined for $\rho_x > \rho_y$, is well defined at the firm's optimum.

\noindent\textbf{Step 2: Equal-accuracy design.}
Impose $\rho_x = \rho_y \equiv \rho$.
The relevant usage cutoffs are $\bar\rho_x$ defined in Step~1 and
\begin{equation*}
\bar\rho^{(y)} := \frac{(1-\underline{\alpha})b + c - \theta r}{b},
\end{equation*}
so that $d^{E}_x = \tilde d^{E}(\rho)$ if and only if
$\rho \le \bar\rho_x$ and $d^{E}_y = \tilde{\beta}\,\tilde d^{E}(\rho)$ if and
only if $\rho \le \bar\rho^{(y)}$. Note that
$\bar\rho^{(y)} < \bar\rho_x$ if and only if
$\overline{\alpha} + \underline{\alpha} > 1$.

\emph{Case A: $\overline{\alpha} + \underline{\alpha} > 1$
(so $\bar\rho^{(y)} < \bar\rho_x$).}
The firm's objective is
\begin{equation*}
\pi^{E}(\rho) =
\begin{cases}
\tilde{f}(1+\tilde{\beta})\,\tilde d^{E}(\rho) - (\kappa_x + \kappa_y)\big(\rho-\tfrac12\big)^2,
  & \text{if } \tfrac12 < \rho \le \bar\rho^{(y)},\\[1ex]
\tilde{f}\big[\tilde d^{E}(\rho) + \tilde{\beta}\,\bar d^{E}_y(\rho)\big]
  -(\kappa_x + \kappa_y)\big(\rho-\tfrac12\big)^2,
  & \text{if } \bar\rho^{(y)} < \rho \le \bar\rho_x,\\[1ex]
\tilde{f}\big[\bar d^{E}_x(\rho) + \tilde{\beta}\,\bar d^{E}_y(\rho)\big]
  - (\kappa_x + \kappa_y)\big(\rho-\tfrac12\big)^2,
  & \text{if } \bar\rho_x < \rho < 1.
\end{cases}
\end{equation*}
The marginal revenues on the three regions are $2\tilde{f}(1+\tilde{\beta})$,
$\tilde{f}(2+\tilde{\beta})$, and $\tilde{f}(1+\tilde{\beta})$, which are strictly decreasing across
regions; together with the strictly convex cost, this makes $\pi^{E}(\rho)$ globally
concave in $\rho$. The region-specific first-order candidates are
\begin{equation*}
\rho_1 := \tfrac12 + \tfrac{\tilde{f}(1+\tilde{\beta})}{\kappa_x + \kappa_y}, \qquad
\rho_2 := \tfrac12 + \tfrac{\tilde{f}(2+\tilde{\beta})}{2(\kappa_x + \kappa_y)}, \qquad
\rho_3 := \tfrac12 + \tfrac{\tilde{f}(1+\tilde{\beta})}{2(\kappa_x + \kappa_y)},
\end{equation*}
with $\rho_1 > \rho_2 > \rho_3$, and the unique maximizer is
\begin{equation}\label{eq:rhoEstarA}
\rho^{*} =
\begin{cases}
\rho_1, & \text{if } \rho_1 \le \bar\rho^{(y)},\\
\bar\rho^{(y)}, & \text{if } \rho_2 \le \bar\rho^{(y)} < \rho_1,\\
\rho_2, & \text{if } \bar\rho^{(y)} < \rho_2 \le \bar\rho_x,\\
\bar\rho_x, & \text{if } \rho_3 \le \bar\rho_x < \rho_2,\\
\rho_3, & \text{if } \bar\rho_x < \rho_3,
\end{cases}
\end{equation}
truncated, if necessary, to $(1/2,1)$; as in the baseline, we assume the optimum is interior.

\emph{Case B: $\overline{\alpha} + \underline{\alpha} \le 1$
(so $\bar\rho_x \le \bar\rho^{(y)}$).}
The middle region becomes
$(\bar\rho_x, \bar\rho^{(y)}]$, on which
$\pi^{E}(\rho) = \tilde{f}\big[\bar d^{E}_x(\rho) + \tilde{\beta}\,\tilde d^{E}(\rho)\big]
- (\kappa_x+\kappa_y)(\rho-\tfrac12)^2$ with marginal revenue $\tilde{f}(1+2\tilde{\beta})$ and candidate
$\rho_2' := \tfrac12 + \tfrac{\tilde{f}(1+2\tilde{\beta})}{2(\kappa_x + \kappa_y)}$. Because
$2\tilde{f}(1+\tilde{\beta}) > \tilde{f}(1+2\tilde{\beta}) > \tilde{f}(1+\tilde{\beta})$, global concavity again holds,
and $\rho^{*}$ is given by 
\begin{equation*}
\rho^{*} =
\begin{cases}
\rho_1, & \text{if } \rho_1 \le \bar\rho_x,\\
\bar\rho_x, & \text{if } \rho_2' \le \bar\rho_x < \rho_1,\\
\rho_2', & \text{if } \bar\rho_x< \rho_2' \le \bar\rho^{(y)} ,\\
\bar\rho^{(y)}, & \text{if } \rho_3 \le \bar\rho^{(y)}  < \rho_2',\\
\rho_3, & \text{if } \bar\rho^{(y)}  < \rho_3.
\end{cases}
\end{equation*}

Because no term in the equal-accuracy problem involves $\ell$, the maximized payoff $\pi^{E*}:=\pi^{E}(\rho^*)$ is independent of $\ell$.

\noindent\textbf{Step 3: Existence of the liability cutoff.}
Let $\pi^{D*}(\ell) := \pi_x(\rho^{*}_x) + \pi_y(\rho^{*}_y;\ell)$
denote the maximized payoff under the disparate design. The type-$x$
component does not depend on $\ell$. For the type-$y$ component, observe
that for every fixed $\rho_y \in (1/2,1)$,
\begin{equation*}
\frac{\partial \tilde d_y(\rho_y)}{\partial\ell}
= -\frac{2(1-\rho_y)\theta}{b} < 0,
\qquad
\frac{\partial \bar d_y(\rho_y)}{\partial\ell}
= -\frac{(1-\rho_y)\theta}{b} < 0,
\end{equation*}
so both branches of $\pi_y(\rho_y;\cdot)$, and hence their pointwise
minimum-based composition, are strictly decreasing in $\ell$ at every
$\rho_y$. (We restrict attention to parameters for which type-$y$ demand is strictly
positive at the disparate optimum for all $\ell\le\widetilde\ell$, as in our
numerical example, so that $\pi^{D*}(\ell)$ is strictly
decreasing on the relevant range.) Because the objective is strictly decreasing in $\ell$
pointwise in $\rho_y$, its maximum over $\rho_y$ is strictly decreasing
in $\ell$, on interior and boundary branches alike; continuity of
$\pi^{D*}(\ell)$ in $\ell$ follows from the maximum theorem. Because
$\pi^{E*}$ is constant in $\ell$, the difference
$\pi^{D*}(\ell) - \pi^{E*}$ is continuous and strictly decreasing, and
therefore crosses zero at most once. Define
$\widetilde\ell := \inf\{\ell \ge 0 : \pi^{D*}(\ell) \le \pi^{E*}\}$ (with
$\widetilde\ell = \infty$ if the set is empty). Then the firm strictly
prefers the disparate design for $\ell < \widetilde\ell$ and the
equal-accuracy design for $\ell > \widetilde\ell$; the parameter regions in
the statement are those for which $\widetilde\ell \in (0,\infty)$, which is
nonempty for the numerical configuration in the proof of \cref{prop:disparity_alphaxy}.
Finally, by Step~1, whenever the disparate design is chosen, the
optimal accuracies satisfy $\rho^{*}_x > \rho^{*}_y$ under
the assumption $b\kappa_y > 2\tilde{\beta}\kappa_x(b + \theta\ell)$\footnote{Absent the assumption $b\kappa_y > 2\tilde{\beta}\kappa_x(b + \theta\ell)$, the accuracy ranking under the optimal disparate algorithm can reverse, that is, $\rho_x^*<\rho_y^*$. The mechanism is that without the assumption, the type-$x$ cap $\bar\rho_x=(\overline{\alpha}b+c-\theta r)/b$ can bind, pushing $\rho_x^*$ to the half-return branch $\tfrac12+\tfrac{\tilde{f}}{2\kappa_x}$, whereas $\bar\rho_y$ cannot bind when its value exceeds one.  %
}.\hfill \emph{Q.E.D.}

\medskip

\begin{proposition}\label{prop:disparity_alphaxy}
There exist scenarios in which, as $\ell$ increases, the physician uses AI for fewer type-$y$ patients when $\ell$ is below a threshold, and for more type-$y$ patients when $\ell$ is above that threshold. Depending on which branch of \cref{equ:AI demand prior} is active, the threshold is either
 \[
 \frac{b(\kappa_y-2\tilde{\beta} \tilde{f})}{2\theta\tilde{\beta} \tilde{f}}
 \quad \text{or} \quad
 \frac{b(\kappa_y-4\tilde{\beta} \tilde{f})}{4\theta\tilde{\beta} \tilde{f}}.
 \]
\end{proposition}

\noindent \textsc{Proof of \cref{prop:disparity_alphaxy}}.
On the intermediate branch of \eqref{eq:rhoystar}, on which $\rho_y^*=\bar\rho_y$, type-$y$ demand equals $\tilde{\beta}(1-2\underline{\alpha})$ and does not vary with $\ell$. When the relevant threshold is $\frac{b(\kappa_y-4\tilde{\beta} \tilde{f})}{4\theta\tilde{\beta} \tilde{f}}$, the argument is identical to that in the proof of \cref{prop:disparity}. We therefore focus on the case in which the threshold is $\frac{b(\kappa_y-2\tilde{\beta} \tilde{f})}{2\theta\tilde{\beta} \tilde{f}}$. In this case, \cref{equ:AI demand prior} implies $d_y^{\mathrm D}=\tilde{\beta}\Big(\frac{\rho_y^* b-(1-\rho_y^*)\theta\ell-c+\theta r}{b}-\underline{\alpha}\Big)$, where $\rho_y^*=\tfrac12+\dfrac{\tilde{\beta} \tilde{f}(b+\theta\ell)}{2b\kappa_y}$ from \cref{prop:developer_rho*_alpha_t} (note the AI demand for type-$y$ patients is actually $d_y^D/\overline{\alpha}$, but it suffices to analyze $d_y^D$ because $\overline{\alpha}$ is a constant). Differentiating with respect to $\ell$ and simplifying yields $\frac{\partial d_y^{\mathrm D}}{\partial \ell}\ge 0$ if and only if $\ell \ge \frac{b(\kappa_y-2\tilde{\beta} \tilde{f})}{2\theta\tilde{\beta} \tilde{f}}$.
Thus, holding fixed the disparate-design regime, liability reduces type-$y$ AI use for $\ell$ below this cutoff and increases it above the cutoff. Non-monotonicity in equilibrium arises when the firm switches to an equal-accuracy design at a higher value of $\ell$; the numerical instance in the footnote illustrates this pattern.%
\footnote{For example, when  $\underline{\alpha}=0.4$, $\overline{\alpha}=0.88$, $b=0.9$, $\theta=1$, $c=0.36$, $r=0.3$, $\kappa_x=0.4$, $\kappa_y=0.55$, $\tilde{f}=0.28$, and $\tilde{\beta}=0.9$, the physician uses AI for fewer type-$y$ patients if $0<\ell<0.0821$ and for more type-$y$ patients if $0.0821<\ell<0.4224$; the equal-accuracy design is optimal if $\ell\ge 0.4224$. At these values, the sufficient condition in \cref{prop:developer_rho*_alpha_t} does not hold; the ranking $\rho_x^*>\rho_y^*$ holds by direct computation.} \hfill \emph{Q.E.D.}

The two thresholds correspond to the cases in which $d_y^{\mathrm D}$ is determined by the first and second expressions in \cref{equ:AI demand prior}, respectively. The baseline intuition still applies: liability exposure trades off against the accuracy gains that higher liability induces. When liability is small, exposure concerns dominate, and the physician uses AI for fewer patients; when liability is large, the accuracy effect dominates, and she uses AI for more patients.

For the welfare effect of mandating equal accuracy, we define the adoption cutoffs:
\begin{align*}
\underline a_x^{\mathrm D} &= \frac{(1-\rho_x^*)b+c-\theta r}{b}, &
\overline a_x^{\mathrm D} &= \min\left\{\overline{\alpha},\frac{\rho_x^*b-c+\theta r}{b}\right\}, \\
\underline a_y^{\mathrm D} &= \max\left\{\underline{\alpha},\frac{(1-\rho_y^*)(b+\theta\ell)+c-\theta r}{b}\right\}, &
\overline a_y^{\mathrm D} &= \frac{\rho_y^*b-(1-\rho_y^*)\theta\ell-c+\theta r}{b}, \\
\underline a_x^m &= \frac{(1-\rho^m)b+c-\theta r}{b}, &
\overline a_x^m &= \min\left\{\overline{\alpha},\frac{\rho^mb-c+\theta r}{b}\right\}, \\
\underline a_y^m &= \max\left\{\underline{\alpha},\frac{(1-\rho^m)b+c-\theta r}{b}\right\}, &
\overline a_y^m &= \frac{\rho^mb-c+\theta r}{b}.
\end{align*}
Then patient welfare under the disparate and equal-accuracy algorithms is
\begin{align*}
W_x(\rho_x^*)
&=\frac{1}{\overline{\alpha}}
\Bigg[
\int_{0}^{\underline a_x^{\mathrm D}}(1-\alpha)b\,d\alpha
+\int_{\underline a_x^{\mathrm D}}^{\overline a_x^{\mathrm D}}(\rho_x^*b-c)\,d\alpha
+\int_{\overline a_x^{\mathrm D}}^{\overline{\alpha}}\alpha b\,d\alpha
\Bigg], \notag\\
W_y(\rho_y^*)
&=\frac{\tilde{\beta}}{\overline{\alpha}}
\Bigg[
\int_{\underline{\alpha}}^{\underline a_y^{\mathrm D}}(1-\alpha)b\,d\alpha
+\int_{\underline a_y^{\mathrm D}}^{\overline a_y^{\mathrm D}}(\rho_y^*b-c)\,d\alpha
+\int_{\overline a_y^{\mathrm D}}^{1}\alpha b\,d\alpha
\Bigg], \notag\\
W_x(\rho^m)
&=\frac{1}{\overline{\alpha}}
\Bigg[
\int_{0}^{\underline a_x^m}(1-\alpha)b\,d\alpha
+\int_{\underline a_x^m}^{\overline a_x^m}(\rho^mb-c)\,d\alpha
+\int_{\overline a_x^m}^{\overline{\alpha}}\alpha b\,d\alpha
\Bigg], \notag\\
W_y(\rho^m)
&=\frac{\tilde{\beta}}{\overline{\alpha}}
\Bigg[
\int_{\underline{\alpha}}^{\underline a_y^m}(1-\alpha)b\,d\alpha
+\int_{\underline a_y^m}^{\overline a_y^m}(\rho^mb-c)\,d\alpha
+\int_{\overline a_y^m}^{1}\alpha b\,d\alpha
\Bigg].
\end{align*}

\cref{fig:welfare_alphat} reports the welfare effects at the parameter values of the preceding footnote. The central message of the baseline persists: over a substantial region of the parameter space, the mandate harms both patient groups. When the type-$y$ segment is small, the mandate lowers type-$x$ accuracy enough to curb type-$x$ overuse, so type-$x$ patients benefit. Type-$y$ patients benefit only when the type-$y$ segment is large and liability is low. The mandate then raises type-$y$ accuracy above $\frac{1}{2}+c/b$, the level at which AI use starts to create value for patients, whereas the disparate design leaves substantial type-$y$ use at lower accuracy. Elsewhere, the mandate harms type-$y$ patients, because the additional use it induces occurs where the clinical gain falls short of the use cost. In the hatched region, the type-$y$ segment is unprofitable under the disparate design, and the firm serves type-$x$ patients only.

The intuition parallels the baseline overuse logic. Because the physician's private return $\theta r$ tilts use beyond the altruistic benchmark, welfare effects are governed by utilization as much as by accuracy. The mandate moderates type-$x$ accuracy, which trims type-$x$ overuse, and it removes the liability wedge on type-$y$ use, which expands that use. The expansion helps type-$y$ patients only where the mandated accuracy makes the added use worth its cost.

\begin{figure}
    \centering
    \includegraphics[width=0.5\linewidth]{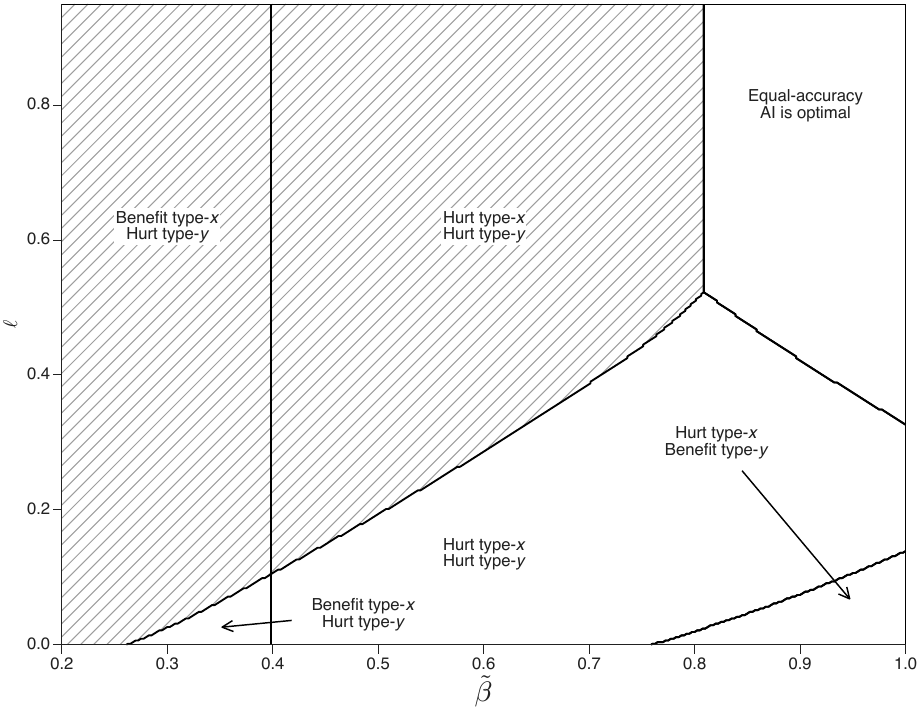}
    \caption{Impact of mandating equal-accuracy AI on expected welfare for each patient type ($\underline{\alpha}=0.4$, $\overline{\alpha}=0.88$, $b=0.9$, $\theta=1$, $c=0.36$, $r=0.3$, $\kappa_x=0.4$, $\kappa_y=0.55$, $\tilde f=0.28$). In the hatched region, the firm serves type-$x$ patients only under the disparate design.}
    \label{fig:welfare_alphat}
\end{figure}

\newpage

\end{document}